\documentclass{article}

 \PassOptionsToPackage{numbers, compress}{natbib}
 
 \usepackage[main, final]{neurips_2026}

\usepackage[utf8]{inputenc} 
\usepackage[T1]{fontenc}    
\usepackage{hyperref}       
\usepackage{url}            
\usepackage{booktabs}       
\usepackage{amsfonts}       
\usepackage{nicefrac}       
\usepackage{microtype}      
\usepackage{xcolor}         
\usepackage[table]{xcolor}         

\usepackage{graphicx}
\usepackage{amsmath,amssymb,amsthm,mathtools}
\usepackage{array}
\usepackage{multirow}
\usepackage{longtable}
\usepackage{enumitem}
\usepackage{algorithm}
\usepackage{algpseudocode}
\usepackage{titletoc}

\theoremstyle{plain}
\newtheorem{theorem}{Theorem}
\newtheorem{lemma}{Lemma}
\newtheorem{proposition}{Proposition}

\theoremstyle{definition}
\newtheorem{definition}{Definition}
\newtheorem{assumption}{Assumption}
\theoremstyle{remark}
\newtheorem{remark}{Remark}

\newcommand{\std}[1]{\textcolor{gray}{\footnotesize (#1)}}

\title{Co-optimization for Adaptive Conformal Prediction}

\author{%
  Xiaoyi Su\textsuperscript{1}\qquad Zhixin Zhou\textsuperscript{2}\qquad  Rui Luo\textsuperscript{1} \\
  \textsuperscript{1}Department of Systems Engineering, City University of Hong Kong \\
  \textsuperscript{2}Alpha Benito Research \\
  \texttt{xiaoyisu3-c@my.cityu.edu.hk}\qquad
  \texttt{zzhou@alphabenito.com}\qquad
  \texttt{ruiluo@cityu.edu.hk}
}

\begin{document}

\maketitle

\begin{abstract}
In regression, conformal prediction often suffers from inefficiency under heteroscedasticity and skewness due to fixed, non-adaptive interval centering. We propose CoCP, a new method that parameterizes prediction intervals by a center $m(x)$ and a radius $h(x)$, and alternates between learning $h(x)$ via smooth quantile regression on folded residuals and refining $m(x)$ using a smooth interval loss. This corrects mis-centering and drives the interval toward high-density regions. Beyond finite-sample marginal validity via split-conformal calibration, we prove that CoCP also attains asymptotic conditional coverage and optimal interval length if the base estimators are consistent. Experiments demonstrate that CoCP yields tighter intervals and achieves improved empirical conditional-coverage reliability. Code is available at \url{https://github.com/jjoeysu/cocp}.
\end{abstract}

\section{Introduction}

Even though conformal prediction (CP) \cite{vovk2005algorithmic, shafer2008atutorial, lei2018distribution, angelopoulos2023gentle} provides a distribution-free recipe for constructing prediction sets with finite-sample marginal coverage under exchangeability, valid prediction intervals should further be \emph{as short as the conditional distribution permits} while adapting to heteroscedasticity and shape changes. Shorter intervals carry direct downstream value, e.g., tightening the pessimism gap in robust predict-then-optimize \cite{patel2024conformal} and sharpening escalation decisions in selective regression \cite{johansson2024reliable}. In this context, Conformalized Quantile Regression (CQR) \cite{romano2019conformalized,sesia2020comparison} 
is a popular instantiation combining quantile estimation with CP to produce adaptive, marginally valid intervals.

The theoretical benchmark for interval efficiency is the highest density interval (HDI). By the Neyman--Pearson lemma, the shortest \((1-\alpha)\)-region corresponds to a density level set \(\{y: f(y) \ge \lambda\}\), which reduces to the contiguous HDI under unimodality \cite{casella2024statistical, hyndman1996computing}. For interior solutions, its endpoints satisfy an equal-density boundary condition. However, standard methods like CQR enforce equal-tailed error rates (e.g., \(\alpha/2\) on each side). Under skewed distributions, this rigid constraint systematically displaces the interval from high-density regions, yielding unnecessarily long intervals even when conditional coverage is maintained (Figure~\ref{fig:hdi_folding}a). Thus, an optimally efficient interval must not only scale to accommodate local noise but also translate its center toward concentrated probability mass.

The core insight of this paper is a geometric pathway from any interval to a HDI-like interval, conceptualized via the folded distribution visualization in Figure~\ref{fig:hdi_folding}. Consider an interval parameterized by a center \(m(x)\) and a radius \(h(x)\). For a fixed center \(m\), the radius required to capture \(1-\alpha\) conditional mass is determined by a single boundary quantity: the \((1-\alpha)\)-quantile of the folded residual \(\lvert Y-m(X)\rvert\mid X=x\). If the conditional densities at the two corresponding endpoints are imbalanced, translating \(m\) slightly toward the higher-density side shifts the data distribution relative to the center. Visually, this relative motion \emph{pushes} denser probability mass into the interval while \emph{pulling} sparser mass out (Figure~\ref{fig:hdi_folding}c). Because a net positive mass enters the interval, the boundary quantile---and thus the radius \(h\)---must contract to maintain the exact \(1-\alpha\) mass, thereby shortening the interval. This ``push--pull'' dynamic continues until the endpoint densities balance, which, under unimodality, perfectly recovers the HDI (Figure~\ref{fig:hdi_folding}d). Thus, the discrepancy between any interval and HDI-like interval can be elegantly framed as a \emph{boundary-balancing} problem in a folded geometry.

\begin{figure}[t!]
  \centering
  \includegraphics[width=0.95\textwidth]{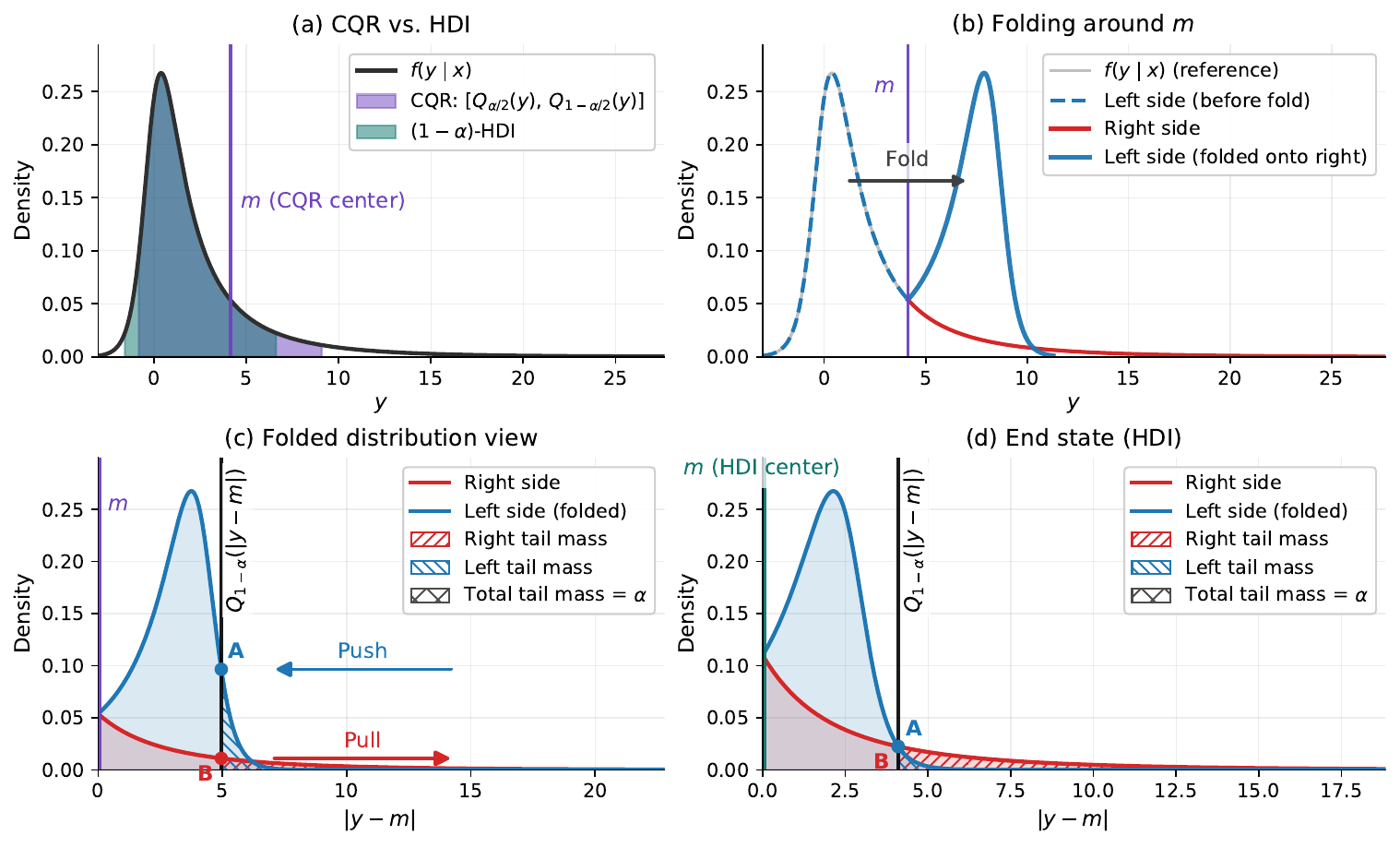}
  \caption{From any interval to HDI via folding and boundary balancing.
  (a) Under skewness, equal-tailed intervals can be displaced relative to the HDI and are typically longer.
  (b) Folding around a candidate center \(m\) maps \(y\) to \(\lvert y-m\rvert\), producing a two-layer view.
  (c) The folded \((1-\alpha)\)-boundary corresponds to two endpoints; translating \(m\) shortens the interval when the endpoint densities are imbalanced (push--pull).
  (d) At equilibrium the endpoint densities match, recovering the HDI.}
  \label{fig:hdi_folding}
\end{figure}

We operationalize this principle with \textbf{CoCP} (\textbf{Co}-optimization for adaptive \textbf{C}onformal \textbf{P}rediction). CoCP explicitly couples the learning of \(m(x)\) and \(h(x)\), enabling the interval to simultaneously translate and scale. Specifically, \(h\) tracks the folded \((1-\alpha)\)-boundary via a smooth pinball loss, while \(m\) is refined using a differentiable smooth interval loss. The surrogate concentrates the learning signal near the interval endpoints (Figure~\ref{fig:softcov}), directly capturing the boundary imbalance. Finally, we apply a standard conformal calibration step to guarantee finite-sample marginal coverage. Therefore, the co-optimization stage can be viewed as an interval-oriented conformalized training, learning adaptive interval geometry, while calibration ensures marginal validity.

Our main contributions are summarized as follows: (1) \textbf{Folded geometry perspective.} We introduce a boundary-balancing viewpoint that explains the inefficiency of center-invariant intervals under skewness, motivating the co-optimization of translation and scaling. (2) \textbf{CoCP framework.} We propose a differentiable alternating optimization method. Beyond distribution-free finite-sample marginal validity, we show that CoCP asymptotically recovers the optimal interval length and achieves asymptotic conditional coverage under the stated regularity and consistency conditions. (3) \textbf{Empirical performance.} Across diverse benchmarks, CoCP yields tighter intervals and consistently improves empirical conditional reliability over the evaluated baselines.

\begin{figure}[ht!]
  \centering
  \includegraphics[width=0.95\textwidth]{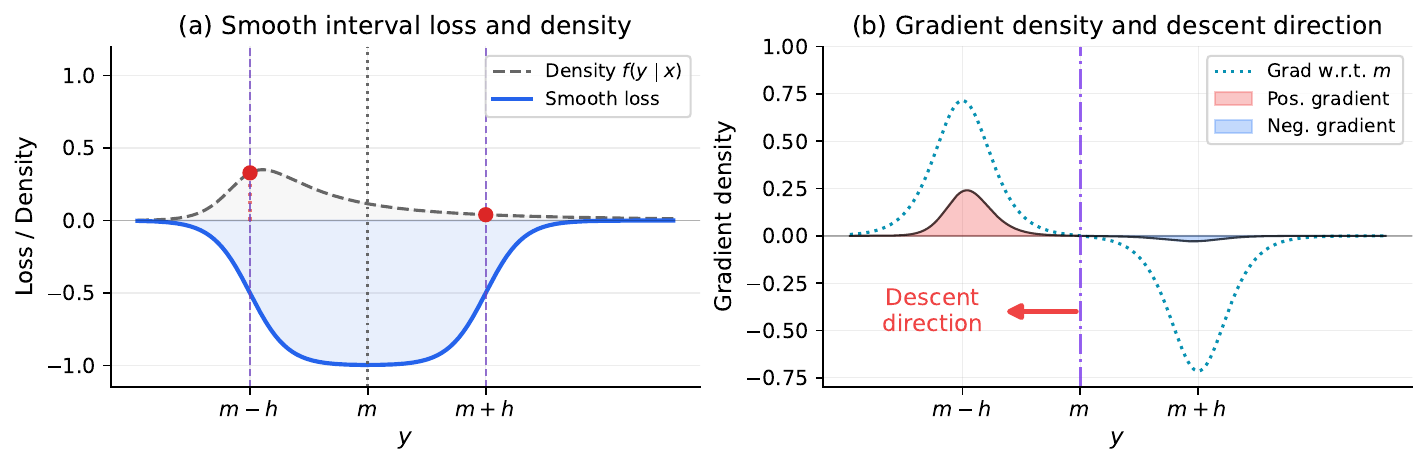}
  \caption{Gradient dynamics of the smooth interval loss.
  (a) The smooth interval loss serves as a differentiable surrogate for the negative interval indicator. The underlying data density (dashed line) is imbalanced at the two boundaries.
  (b) The gradient of the smooth loss with respect to the center \(m\) acts as a signed sampling kernel, being positive on the left boundary and negative on the right. The resulting gradient density (the product of the conditional density and the derivative kernel) is asymmetric under skewness; the higher-density boundary generates a dominant directional signal for gradient descent, shifting the center \(m\) toward the concentration of probability mass.}
  \label{fig:softcov}
\end{figure}

\section{Related Work}

Standard CQR \cite{romano2019conformalized,sesia2020comparison} yields heteroscedastic but equal-tailed intervals. To better adapt to skewed or multimodal data, distributional methods \cite{chernozhukov2021distributional,izbicki2022cd} extract high-density sets from learned conditional distributions. For instance, Conformal Histogram Regression (CHR) \cite{sesia2021conformal} and Conformal Interquantile Regression (CIR) \cite{guo2026fast} conformalize histograms or interquantile partitions, while generative models like C-HDR \cite{izbicki2022cd,dheur2025aunified} capture multimodality. Other approaches optimize restricted set classes \cite{gao2025volume} or Conformal Thresholded Intervals (CTI) \cite{luo2025conformal}. However, as noted in \cite{bao2025areview}, these methods often require estimating the full conditional density or CDF, rely on discretization, and may produce non-contiguous or unstable sets. Direct neural interval construction using soft-coverage heuristics \cite{pearce2018high} avoids full density estimation but relies on brittle Lagrangian penalties. In contrast, CoCP bypasses full density estimation by embedding soft-coverage within a principled folded geometry, using boundary-local gradients to autonomously center the interval and minimize the required radius.

Beyond marginal validity, achieving exact conditional coverage is fundamentally limited without strong assumptions \cite{lei2014distribution, barber2021thelimits}. Consequently, recent work focuses on rigorous diagnostics, such as evaluating the Excess Risk of Target coverage (ERT) \cite{braun2026conditional}, and principled relaxations via localization \cite{tibshirani2019conformal,guan2023localized,hore2025conformal} or group fairness \cite{bairaktari2025kandinsky}. Other approaches calibrate against specific function classes like reproducing kernel Hilbert spaces (RKHS) \cite{gibbs2025conformal}, traced efficiently by SpeedCP \cite{liu2026speedcp}, and extended by Conformal Prediction with Length-Optimization (CPL) \cite{kiyani2024length}. Alternatively, methods like Colorful Pinball Conformal Prediction (CPCP) \cite{chen2026colorful} and others \cite{kiyani2024conformal} directly minimize conditional coverage errors. CoCP provides a distinct geometric solution: it improves conditional reliability not through complex reweighting or functional constraints, but by explicitly co-optimizing the interval's translation and scaling to adapt to local density imbalances.

In addition, several methods improve local adaptivity by transforming conformity scores. Techniques like Rectified Conformal Prediction (RCP) \cite{plassier2025rectifying}, adaptive split-Jackknife+ \cite{deutschmann2024adaptive}, and locally trained predictors \cite{colombo2023training} refine the scale of uncertainty. Other works improve statistical efficiency via cross-conformal data reuse \cite{vovk2015cross,gasparin2025improving}. However, these score-adaptation approaches typically inherit a fixed interval center from a base predictor, adjusting only the scale and failing to correct mis-centering under skewness. CoCP fundamentally differs by abandoning the fixed-center paradigm, instead alternating between radius scaling and center translation to achieve boundary-balancing optimality.

\section{The Geometry of Efficient Intervals and Lyapunov Dynamics}\label{sec:geometry}

This section establishes the geometric foundation of CoCP. We demonstrate how actively translating an interval's center minimizes its required radius to achieve the HDI lower bound, and show that this dynamic forms a monotonically descending Lyapunov system.

Let \((X,Y)\sim P\) with \(X\sim P_X\), \(X\in\mathcal{X}\subseteq\mathbb{R}^d\) and \(Y\in\mathbb{R}\). Let \(f(y\mid x)\) denote the conditional density of \(Y\mid X=x\). For a unimodal and regular \(f(y\mid x)\), the shortest \((1-\alpha)\)-coverage interval is the unique HDI. In center-radius form, \(C^*(x)= \left[m^*(x)-h^*(x), m^*(x)+h^*(x)\right]\), which is uniquely characterized by two simultaneous conditions:
\begin{enumerate}[leftmargin=2em,topsep=0pt,itemsep=2pt]
    \item \textbf{Coverage condition:} \(\int_{m^*-h^*}^{m^*+h^*} f(y\mid x)\,dy = 1-\alpha\).
    \item \textbf{Density-equality condition:} \(f(m^*-h^*\mid x) = f(m^*+h^*\mid x)\).
\end{enumerate}

\textbf{Logistic smoothing.} Directly optimizing this geometric objective via gradient descent is intractable due to the zero gradients produced by the hard indicator \(\mathbf{1}\{\lvert Y-m\rvert \le h\}\). To overcome this, we introduce a logistic smoothing. Define the symmetric logistic kernel \(K_\beta(u)\coloneqq \frac{1}{\beta}\sigma'(u/\beta)\), where \(\sigma(z)\coloneqq (1+e^{-z})^{-1}\) and \(\beta>0\). The smoothed conditional density is defined as the convolution:
\begin{equation}
f_{\beta}(y\mid x)\coloneqq (f(\cdot \mid x)*K_\beta)(y) = \int_{\mathbb{R}} f(u\mid x)\,K_\beta(y-u)\,du.
\end{equation}
\begin{proposition}
\label{prop:unique_mode_bounded}
Fix \(x\in\mathcal X\). Suppose \(f(\cdot\mid x)\) is a bounded unimodal probability density on \(\mathbb R\). Then \(f_\beta(\cdot\mid x)\in C^1(\mathbb R)\), and \(f_\beta'(\cdot\mid x)\) has exactly one real zero \(\nu_\beta(x)\). Moreover, \(f_\beta'(y\mid x)>0\) for \(y<\nu_\beta(x)\), and \(f_\beta'(y\mid x)<0\) for \(y>\nu_\beta(x)\). In particular, \(f_\beta(\cdot\mid x)\) is strictly unimodal.
\end{proposition}
For a candidate \((m, h)\), define the smoothed coverage function \(\Gamma(m, h; x)\coloneqq\! \int_{m-h}^{m+h}\! f_{\beta}(y\mid x) \, dy\). Since \(f_\beta>0\), \(\Gamma\) is strictly monotonically increasing with respect to \(h\). Thus, the minimal feasible radius \(\phi_{\beta}(m; x) \coloneqq \inf \{ h > 0 : \Gamma(m, h; x) \ge 1-\alpha \}\) is achieved at: \(\Gamma(m, \phi_{\beta}; x) = 1-\alpha\).

\textbf{Translation dynamics and the Lyapunov property.} 
Applying the implicit function theorem to \(\Gamma(m, \phi_{\beta}; x) = 1-\alpha\), we derive the exact gradient of the required radius with respect to the center:
\begin{equation}\label{eq:lyapunov_gradient}
\frac{\partial \phi_{\beta}(m; x)}{\partial m} = - \frac{\frac{\partial \Gamma}{\partial m}}{\frac{\partial \Gamma}{\partial \phi_{\beta}}} = - \frac{f_{\beta}(m+\phi_{\beta}\mid x) - f_{\beta}(m-\phi_{\beta}\mid x)}{f_{\beta}(m+\phi_{\beta}\mid x) + f_{\beta}(m-\phi_{\beta}\mid x)}.
\end{equation}
Note that it depends only on the density at the interval endpoints. This reveals an elegant \emph{push--pull} dynamic: shifting the center \(m\) toward the denser endpoint strictly decreases the required radius.

\section{Co-optimizing Interval Geometry via Alternating Learning}\label{sec:method}

CoCP implements this principle through differentiable alternating optimization between the center \(m(x)\) and the radius \(h(x)\) in Algorithm~\ref{alg:cocp}. This creates a duality between HDI statistical properties and the functional optimization landscape.

\textbf{Radius update (Coverage).} Given a fixed center \(m\), we update \(h\) to estimate the smoothed \((1-\alpha)\)-quantile of the absolute residual by optimizing the \emph{smooth pinball loss}:
\begin{equation}
\label{eq:rad-loss-smooth}
\mathcal{L}_{\mathrm{H}}(h; m, \beta) \coloneqq \mathbb{E} \left[ \beta \sum_{s \in \{1, -1\}} \zeta\left(\frac{s\cdot h(X) - (Y-m(X))}{\beta}\right) - (1-\alpha)h(X) \right],
\end{equation}
where \(\zeta(z) = \log(1 + e^z)\). In the unrestricted population limit, the functional first-order condition (FOC) \(\partial \mathcal{L}_{\mathrm{H}}/\partial h = 0\) exactly recovers \(\Gamma(m, h; X) = 1-\alpha\), fulfilling the first defining property of the HDI. Furthermore, because \(\zeta(z)\) is globally strictly convex, \(\mathcal{L}_{\mathrm{H}}\) is globally strictly convex in \(h\).

\textbf{Center update (Density-equality).} Let \(\tilde{\mathbf{1}}_\beta(y;m,h)\coloneqq\sigma((h-(y-m))/\beta)-\sigma((-h-(y-m))/\beta)\) be the logistic-smoothed interval indicator. To correct mis-centering, we fix $h$ and optimize the \emph{smooth interval loss} (the negative expected smoothed coverage, unified with \(\mathcal{L}_{\mathrm{H}}\) via \(\zeta'=\sigma\)):
\begin{equation}
\label{eq:ctr-loss}
\mathcal{L}_{\mathrm{M}}(m;h, \beta) \coloneqq -\mathbb{E}\left[\tilde{\mathbf{1}}_\beta(Y;m(X),h(X))\right].
\end{equation}
The functional gradient evaluates exactly to the boundary density difference: \(\partial \mathcal{L}_{\mathrm{M}}/\partial m = -\bigl(f_\beta(m+h\mid X) - f_\beta(m-h\mid X)\bigr)\). Setting \(\partial \mathcal{L}_{\mathrm{M}}/\partial m = 0\) perfectly recovers the second defining property of the HDI. Furthermore, \(\partial^2 \mathcal{L}_{\mathrm{M}}/\partial m^2 = f'_\beta(m-h\mid X) - f'_\beta(m+h\mid X)\). By Proposition \ref{prop:unique_mode_bounded}, if the interval contains the mode in its interior, then \(f'_\beta(m-h\mid X) > 0\), \(f'_\beta(m+h\mid X) < 0\), making \(\partial^2 \mathcal{L}_{\mathrm{M}}/\partial m^2 > 0\) and \(\mathcal{L}_{\mathrm{M}}\) locally strictly convex in the population limit.

\textbf{Idealized Lyapunov dynamics.} Equations \eqref{eq:rad-loss-smooth} and \eqref{eq:ctr-loss} form a tightly coupled dynamical system. Comparing \(\partial \mathcal{L}_{\mathrm{M}}/\partial m\) with the implicit derivative in \eqref{eq:lyapunov_gradient}, we can formalize a population dynamic.
\begin{proposition}[Idealized Lyapunov descent]\label{prop:lyapunov}
Fix \(x \in \mathcal{X}\). Suppose the radius is instantaneously equilibrated to \(h_t = \phi_\beta(m_t; x)\), and the center follows the population gradient flow \(\frac{dm_t}{dt} = -\frac{\partial \mathcal{L}_{\mathrm{M}}}{\partial m}(m_t; h_t, \beta)\). Then the required radius is a Lyapunov function:
\begin{equation}
\frac{d}{dt}\phi_\beta(m_t; x) = -\frac{\bigl[f_\beta(m_t+h_t\mid x) - f_\beta(m_t-h_t\mid x)\bigr]^2}{f_\beta(m_t+h_t\mid x) + f_\beta(m_t-h_t\mid x)} \le 0,
\end{equation}
with equality if and only if the density-equality condition is met.
\end{proposition}
This property provides the geometric intuition for Algorithm~\ref{alg:cocp}: minimizing the center loss autonomously drives the interval toward density balance, intrinsically minimizing the required radius.

\textbf{Final radius fine-tuning and calibration.} 
After convergence, we perform a final fine-tuning of \(h\) to remove the logistic smoothing bias by minimizing the standard hard pinball loss \(\rho_\tau(u)\coloneqq u(\tau-\mathbf{1}\{u<0\})\) on the exact folded residual: 
\begin{equation}
\label{eq:rad-loss-final}
\mathcal{L}_{\mathrm{PB}}(h;m) \coloneqq \mathbb{E}\left[\rho_{1-\alpha}\left(\lvert Y-m(X)\rvert-h(X)\right)\right].
\end{equation}
Crucially, this step decouples the two objectives: the alternating optimization dictates the center \(m\) to minimize interval length, while this final fine-tuning accurately targets the true conditional quantile to ensure reliable conditional coverage. Finally, we aggregate the \(K\)-fold ensembled predictors \(\widehat{m}\) and \(\widehat{h}\), and calibrate them on an independent set \(\mathcal{D}_{\mathrm{cal}}\) to guarantee finite-sample marginal validity.

\begin{algorithm}[ht]
\caption{CoCP: Co-optimization for Adaptive Conformal Prediction}
\label{alg:cocp}
\begin{algorithmic}[1]
\Require Training data \(\mathcal{D}_{\mathrm{train}}\), calibration data \(\mathcal{D}_{\mathrm{cal}}\), miscoverage \(\alpha\in(0,1)\), folds \(K\), iterations \(T\), temperature \(\beta>0\)
\Ensure Calibrated interval function \(\widehat{C}(\cdot)\)
\State Partition \(\mathcal{D}_{\mathrm{train}}\) into \(K\) disjoint folds \(\{\mathcal{D}_1, \dots, \mathcal{D}_K\}\)
\For{\(k=1,\dots,K\)}
    \State Set \(\mathcal{D}_{h}^{(k)} = \mathcal{D}_k\) and \(\mathcal{D}_{m}^{(k)} = \mathcal{D}_{\mathrm{train}}\setminus \mathcal{D}_k\)
    \State Warm-start \(\widehat m^{(k)}\) on \(\mathcal{D}_{m}^{(k)}\) by minimizing Mean Squared Error (MSE)
    \For{\(t=1,\dots,T\)}
        \State \textbf{Radius update:} fix \(\widehat m^{(k)}\), train \(\widehat h^{(k)}\) on \(\mathcal{D}_{h}^{(k)}\) via smooth pinball loss \eqref{eq:rad-loss-smooth}
        \State \textbf{Center update:} fix \(\widehat h^{(k)}\), train \(\widehat m^{(k)}\) on \(\mathcal{D}_{m}^{(k)}\) via smooth interval loss \eqref{eq:ctr-loss}
    \EndFor
    \State Fine-tune \(\widehat h^{(k)}\) on \(\mathcal{D}_{h}^{(k)}\) by pinball loss \eqref{eq:rad-loss-final} with \(\widehat m^{(k)}\) fixed
\EndFor
\State Form ensemble predictors \(\widehat{m},\widehat{h}\) via averaging: \(\widehat{m} = \frac{1}{K}\sum_{k=1}^K \widehat m^{(k)}\), \(\widehat{h} = \frac{1}{K}\sum_{k=1}^K \widehat h^{(k)}\)
\State Compute nonconformity scores \(S_i=\lvert Y_i-\widehat{m}(X_i)\rvert / \widehat{h}(X_i)\) for all \((X_i,Y_i)\in\mathcal{D}_{\mathrm{cal}}\)
\State Set \(\widehat{q}\) as the \(\lceil (n_{\mathrm{cal}}+1)(1-\alpha)\rceil\)-th smallest score in \(\{S_i\}_{i=1}^{n_{\mathrm{cal}}}\cup\{+\infty\}\)
\State \Return \(\widehat{C}(x) = [\widehat{m}(x) - \widehat{q}\widehat{h}(x), \widehat{m}(x) + \widehat{q}\widehat{h}(x)]\)
\end{algorithmic}
\end{algorithm}
\section{Theoretical Properties}\label{sec:theory}

We establish three layers of guarantees for CoCP. First, we show that the algorithm inherits distribution-free \emph{finite-sample marginal coverage}. Second, we characterize the exact population fixed point of the alternating updates. Finally, we prove that under standard regularity and consistency conditions, CoCP achieves \emph{asymptotic conditional coverage} \citep{lei2018distribution, sesia2020comparison}, together with \emph{asymptotic optimal interval length} for unimodal distributions. Full proofs are deferred to Appendix~\ref{app:theory}.

\begin{theorem}[Finite-sample marginal coverage]\label{thm:cocp_marginal}
Let $\{(X_i,Y_i)\}_{i=1}^{n_{\mathrm{cal}}+1}$ be exchangeable, where the first $n_{\mathrm{cal}}$ pairs form $\mathcal{D}_{\mathrm{cal}}$ and $(X_{n_{\mathrm{cal}}+1},Y_{n_{\mathrm{cal}}+1})$ is a new test point. Suppose $\widehat{m}$ and $\widehat{h}$ are constructed via Algorithm~\ref{alg:cocp} with $\widehat h(x)>0$ a.s.. Then, for every $\alpha\in(0,1)$,
\begin{equation}
1-\alpha \;\le\; \mathbb P\bigl(Y_{n_{\mathrm{cal}}+1}\in\widehat C(X_{n_{\mathrm{cal}}+1})\bigr) \;\le\; 1-\alpha + \frac{1}{n_{\mathrm{cal}}+1},
\end{equation}
where the upper bound additionally requires almost-sure distinctness of the \(n_{\rm cal}+1\) scores.
\end{theorem}

\subsection{The Smoothed HDI as Population Target}\label{sec:smooth_hdi}

The efficiency and conditional coverage depend on the fixed point of the CoCP updates. The first-order conditions of CoCP's losses admit an exact interpretation: the radius FOC \(\nabla_h\mathcal L_H=0\) implies \(h=\phi_\beta(m;x)\), and the center FOC \(\nabla_m\mathcal L_M=0\) enforces \(f_\beta(m+h\mid x)=f_\beta(m-h\mid x)\). Their joint solution defines the \emph{smoothed HDI}. Crucially, as the temperature \(\beta\) decays, this smoothed target asymptotically recovers the true, non-differentiable HDI.

\begin{definition}[Smoothed HDI]\label{def:smooth_hdi}
Fix $\beta>0$ and $x\in\mathcal X$. The smoothed HDI of $Y\mid X=x$ at level $1-\alpha$ is any pair $(m_\beta^\dagger(x),h_\beta^\dagger(x))\in\mathbb R\times(0,\infty)$ satisfying
\begin{equation}\label{eq:smooth_hdi_def_main}
\Gamma\bigl(m_\beta^\dagger(x),h_\beta^\dagger(x);x\bigr)=1-\alpha,
\quad\text{and}\quad
f_\beta\bigl(m_\beta^\dagger(x)+h_\beta^\dagger(x)\bigm|x\bigr)
= f_\beta\bigl(m_\beta^\dagger(x)-h_\beta^\dagger(x)\bigm|x\bigr).
\end{equation}
\end{definition}

\begin{lemma}[Existence, uniqueness, and optimality]\label{lem:smooth_hdi}
Suppose \(f(\cdot\mid x)\) is bounded, integrable, and unimodal. Then for every \(\beta>0\), the smoothed HDI \((m_\beta^\dagger(x),h_\beta^\dagger(x))\) exists, is unique, and coincides with the shortest \((1-\alpha)\)-coverage interval of the smoothed density \(f_\beta(\cdot\mid x)\).
\end{lemma}

\begin{lemma}[Vanishing smoothing bias]\label{lem:beta_consistency}
Under Assumption~\ref{ass:density} (introduced below), for \(P_X\)-a.e.\ \(x\in\mathcal X\),
\begin{equation*}
m_\beta^\dagger(x)\;\rightarrow\;m^*(x),
\qquad
h_\beta^\dagger(x)\;\rightarrow\;h^*(x),
\qquad \text{as }\beta\to 0^+.
\end{equation*}
\end{lemma}

\subsection{Asymptotic Conditional Coverage and Optimal Length}
Throughout the asymptotic analysis, all observations are assumed i.i.d.\ from $P$.
Let \(n\) be the total sample size, and let \(n_{\mathrm{train}}\asymp n\) and \(n_{\mathrm{cal}}\asymp n\) denote the sizes of the training and calibration sets, respectively. To achieve asymptotic optimality as \(n \to \infty\), the temperature \(\beta_n\) must decay to zero to eliminate smoothing bias. Furthermore, the number of alternating iterations \(T\) in Algorithm~\ref{alg:cocp} should be sufficiently large to reach the optimization fixed point. We formalize this through standard regularity conditions on the data distribution (Assumption~\ref{ass:density}), the consistency of the base estimators (Assumption~\ref{ass:consistency}), and the stability of the calibration scores (Assumption~\ref{ass:cal}).

\begin{assumption}[Density regularity]\label{ass:density}
For \(P_X\)-a.e.\ \(x\in\mathcal X\), the conditional law \(Y\mid X=x\) admits a density \(f(\cdot\mid x)\) such that:
(D1) \(f(\cdot\mid x)\) is uniformly bounded by \(f_{\max}<\infty\), integrable, and unimodal, and \(\mathbb{E}[|Y|] < \infty\);
(D2) the shortest \((1-\alpha)\)-coverage interval is unique, denoted as \(C^*(x)=[m^*(x)-h^*(x), m^*(x)+h^*(x)]\);
(D3) \(f(\cdot\mid x)\) is continuous on a neighborhood of the endpoints, and \(f(m^*-h^*\mid x)=f(m^*+h^*\mid x)>0\);
(D4) the conditional CDF \(F(\cdot\mid x)\) is continuous and strictly increasing on the support.
\end{assumption}

\begin{assumption}[Estimation consistency]\label{ass:consistency}
Let \(\phi_0(m;X)\) denote the \((1-\alpha)\)-quantile of \(|Y-m|\mid X\). There exist a deterministic schedule \(\beta_n\to 0^+\) and deterministic sequences \(\eta_{m,n}, \rho_{m,n}, \eta_{h,n}, \rho_{h,n} \to 0\) such that, for a fresh test input \(X\sim P_X\) independent of \(\mathcal D_{\mathrm{train}}\),
\begin{align}
\mathbb P\Bigl(\,\mathbb E\left[(\widehat m(X)-m_{\beta_n}^\dagger(X))^2\,\big|\,\mathcal D_{\mathrm{train}}\right]\le \eta_{m,n}\Bigr)\,&\ge\, 1-\rho_{m,n},\label{eq:assumption_E_m}\\
\mathbb P\Bigl(\,\mathbb E\left[(\widehat h(X)-\phi_0(\widehat m(X);X))^2\,\big|\,\mathcal D_{\mathrm{train}}\right]\le \eta_{h,n}\Bigr)\,&\ge\, 1-\rho_{h,n}.\label{eq:assumption_E_h}
\end{align}
\end{assumption}

\begin{assumption}[Calibration regularity]\label{ass:cal}
With probability tending to 1 (over \(\mathcal D_{\mathrm{train}}\)), \(\widehat h(X)\in[h_-,h_+]\) for some deterministic constants \(0<h_-\le h_+<\infty\). Moreover, the score \(S\coloneqq|Y-\widehat m(X)|/\widehat h(X)\) has a Lebesgue density \(f_S\) bounded below by some \(s_0>0\) on a neighborhood \((1-\delta_0,1+\delta_0)\) of \(1\).
\end{assumption}

Under these conditions, CoCP not only guarantees marginal coverage but asymptotically recovers the oracle \((1-\alpha)\)-HDI. Note that conditional coverage comes primarily from radius consistency around the learned center, while optimal length is achieved via center consistency toward the true HDI center.

\begin{theorem}[Asymptotic conditional coverage and optimal length]\label{thm:cocp_asymp}
Under Assumptions~\ref{ass:density},~\ref{ass:consistency}, and~\ref{ass:cal}, the CoCP prediction interval \(\widehat C(\cdot)\) from Algorithm~\ref{alg:cocp} satisfies the following as \(n\to\infty\):
\begin{enumerate}[leftmargin=2em,topsep=0pt,itemsep=2pt]
\item[(i)] Calibration consistency: \(\widehat q\xrightarrow{\mathbb{P}}1\).
\item[(ii)] Asymptotic conditional coverage: there exist measurable sets \(\Lambda_n=\Lambda_n(\mathcal D_{\mathrm{train}}, \mathcal D_{\mathrm{cal}})\subseteq\mathcal X\) with \(\mathbb P(X\in\Lambda_n\mid \mathcal D_{\mathrm{train}}, \mathcal D_{\mathrm{cal}})\to 1\) such that
\begin{equation}\label{eq:asymp_cond_cov}
\sup_{x\in\Lambda_n}
\left|\mathbb P(Y\in\widehat C(X)\mid X=x, \mathcal D_{\mathrm{train}}, \mathcal D_{\mathrm{cal}})-(1-\alpha)\right|\;\xrightarrow{\mathbb{P}}\;0.
\end{equation}
\item[(iii)] Asymptotic optimal length: \(\mathrm{Len}(\widehat C(X))\big/\mathrm{Len}(C^*(X))\xrightarrow{\mathbb{P}}1\).
\end{enumerate}
\end{theorem}

\begin{remark}[Boundary HDIs]
Assumption~\ref{ass:density}(D3) excludes distributions whose true HDI touches a support boundary (e.g., the exponential law, with HDI $[0,q_{1-\alpha}]$). In such cases the equal-density condition fails for $f$ itself; however, by Proposition~\ref{prop:unique_mode_bounded}, $f_\beta$ remains strictly unimodal on $\mathbb R$ for every $\beta>0$, so the smoothed HDI is still interior and well-targeted by CoCP. Appendix~\ref{app:beta_rates_boundary} establishes an $O(\beta)$ boundary shift, compared with the $O(\beta^2)$ bias obtained for interior HDIs under stronger local smoothness conditions. The conclusions of Theorem~\ref{thm:cocp_asymp} extend to this boundary regime when $\beta_n\to 0$.
\end{remark}

\paragraph{Local sensitivity of conditional coverage.}
The density-equality condition also has implications for the sensitivity of conditional coverage to estimation errors. For an interior HDI, expanding the conditional coverage of $\widehat C(x)$ around $(m^*,h^*)$ at $X=x$ yields
\begin{equation}\label{eq:first_order_cov}
\begin{aligned}
&F(\widehat m+\widehat q\widehat h\mid x)
 -F(\widehat m-\widehat q\widehat h\mid x)-(1-\alpha)\\
&= F(\widehat m+\widehat q\widehat h\mid x)
 -F(\widehat m-\widehat q\widehat h\mid x)
-\left[F(m^*+h^*\mid x)-F(m^*-h^*\mid x)\right]\\
&= \left[f(m^*+h^*\mid x)-f(m^*-h^*\mid x)\right](\widehat m-m^*)\\
&\quad+\left[f(m^*+h^*\mid x)+f(m^*-h^*\mid x)\right](\widehat q\widehat h-h^*)\\
&\quad+o(|\widehat m-m^*|+|\widehat q\widehat h-h^*|).
\end{aligned}
\end{equation}
At the HDI, the endpoint densities are equal, so the first-order contribution of center-estimation error vanishes. Thus, conditional coverage is locally insensitive to small center errors when the radius is held fixed, with the leading error determined by the calibrated radius, which is estimated through a single conditional quantile of the folded residuals. This local robustness may contribute to improved conditional reliability in finite samples.

\begin{remark}[Scope of guarantees]
Note that for CQR, estimation errors in both conditional quantiles generally contribute to the first-order coverage error. In particular, under skewness, the endpoint densities need not agree, so even a common shift of both endpoints affects conditional coverage at first order. Nevertheless, consistently estimated CQR quantiles can also yield asymptotic conditional coverage under suitable regularity conditions \citep{sesia2020comparison}. The distinction therefore lies in the local sensitivity to estimation errors and, under unimodality, CoCP's additional recovery of the shortest conditional interval. Theorem~\ref{thm:cocp_marginal} provides the standard finite-sample marginal guarantee under exchangeability, while Theorem~\ref{thm:cocp_asymp}(ii) establishes asymptotic conditional coverage uniformly on data-dependent sets $\Lambda_n$ whose covariate probability tends to one, subject to Assumptions~\ref{ass:density}--\ref{ass:cal}.
\end{remark}

\paragraph{Consistency of the ensemble.}
Proposition~\ref{prop:lyapunov} characterizes the idealized population dynamics, whereas Algorithm~\ref{alg:cocp} averages fold-specific predictors before calibration. The consistency conditions in Theorem~\ref{thm:cocp_asymp} therefore concern the final ensemble. For fixed $K$, if each fold-specific center $\widehat m^{(k)}$ is $L^2$-consistent for the common target $m_{\beta_n}^\dagger$, Jensen's inequality gives
\begin{equation}
\mathbb{E}\Bigl[\bigl(\widehat m(X) - m_{\beta_n}^\dagger(X)\bigr)^2\Bigr]
\le \frac{1}{K}\sum_{k=1}^{K}
\mathbb{E}\Bigl[\bigl(\widehat m^{(k)}(X) - m_{\beta_n}^\dagger(X)\bigr)^2\Bigr],
\end{equation}
so averaging preserves center consistency. The same conclusion holds for the radii when each fold-specific radius consistently estimates its corresponding folded-residual quantile, using the Lipschitz continuity of $\phi_0$ (Lemma~\ref{lem:rad_lip_app}) and the consistency of the fold-specific centers.

\section{Experiments}\label{sec:experiments}

\paragraph{Protocol and implementation.}
We evaluate CoCP on (i) controlled synthetic benchmarks with known conditional distributions and (ii) real-world regression datasets. We target \(\alpha=0.1\) and report the mean (standard deviation) over 10 random splits of the data into \(50\%/10\%/20\%/20\%\) (train/val/cal/test). The validation and calibration set are used for early stopping and conformal calibration, respectively. All learning-based components share the same backbone (a 2-layer, 64-unit MLP) for fair comparison. More implementation details are provided in Appendix~\ref{app:exp-details}.

\paragraph{Baselines.}
We compare against standard residual-based split CP (Split \cite{papadopoulos2002inductive}), CQR \cite{romano2019conformalized}, distributional-set baselines (C-HDR \cite{izbicki2022cd,dheur2025aunified}, CHR \cite{sesia2021conformal}), and conditional-coverage-oriented methods (CPL \cite{kiyani2024length}, CPCP \cite{chen2026colorful}, RCP \cite{plassier2025rectifying}). Due to the prohibitive computational cost of RKHS-based conditional conformal methods \cite{gibbs2025conformal}, we include a simplified version without test imputation, denoted as Naive CC, exclusively on synthetic data.

\paragraph{Metrics.}
We always report marginal Coverage and mean interval Length. For conditional reliability, we report the exact mean absolute coverage error (ConMAE) on synthetic data where the true conditional CDF is known. On real data, we report diagnostic surrogates: Mean Squared Conditional Error (MSCE \cite{kiyani2024conformal}), Worst-Slice Coverage (WSC \cite{cauchois2021knowing, romano2020classification}), and Excess Risk of Target coverage (ERT \cite{braun2026conditional}), including \(\ell_1\)-ERT and \(\ell_2\)-ERT. Formal definitions are in Appendix~\ref{app:metrics}.

\subsection{Synthetic Data}\label{subsec:exp_synthetic}

\paragraph{Setup.}
We generate \(n=20{,}000\) samples with \(X \sim \mathrm{Unif}(-2, 2)\) and \(Y = \theta(X) + \eta(X, \epsilon)\), where location \(\theta(x) = \frac{1}{2} \sin(1.5x)\) and scale \(s(x) = 0.15 + 0.25x^2\). To cover symmetric, moderately skewed, and highly skewed conditional distributions, we instantiate three noise families for the noise term: \textbf{Normal}: \(\eta = s(X)\epsilon\) with \(\epsilon \sim \mathcal{N}(0, 1)\); \textbf{LogNormal}: \(\eta = s(X)(\epsilon - e^{-0.6^2})\) with \(\epsilon \sim \mathrm{LogNormal}(0, 0.6^2)\); and \textbf{Exponential}: \(\eta \sim \mathrm{Exp}(\lambda = 1/s(X))\), shifting the support to \([\theta(X), \infty)\). For reference, we include the \textbf{Oracle} interval, defined as the true \((1-\alpha)\)-HDI; note it is an absolute length lower bound only among intervals attaining exact conditional coverage.

\paragraph{Main results.}
As shown in Table~\ref{tab:synthetic_results}, CoCP consistently achieves the lowest conditional coverage error (ConMAE) among all practical baselines across the three synthetic benchmarks, demonstrating superior conditional reliability while maintaining nominal marginal coverage. 
The benefits of co-optimizing interval translation and scaling are particularly pronounced under skewed noise. On the LogNormal benchmark, CoCP substantially reduces average interval length (e.g., a \(\approx 14\%\) reduction compared to CQR) while nearly halving the ConMAE. On the highly skewed Exponential benchmark, CoCP yields the tightest non-oracle intervals, reducing length by \(\approx 21\%\) and ConMAE by \(\approx 53\%\) relative to CQR. 
As corroborated by Figure~\ref{fig:compare_interval} and the \(T=0\) ablation in Appendix~\ref{app:exp-sensitivity}, the \(K\)-fold ensemble inherently stabilizes intervals, but it is the alternating center updates that correct skew-induced mis-centering and recover the optimal HDI geometry.

\begin{table}[ht!]
\centering
\caption{Experimental results on synthetic datasets (mean \(\pm\) std over 10 splits, \(\alpha=0.1\)). Lower Length and ConMAE are better. The Oracle row reports the true \((1-\alpha)\)-HDI for reference.}
\label{tab:synthetic_results}
\resizebox{\textwidth}{!}{
\begin{tabular}{@{} l ccc ccc ccc @{}}
\toprule
\multirow{2}{*}{\textbf{Method}} & \multicolumn{3}{c}{\textbf{Normal}} & \multicolumn{3}{c}{\textbf{LogNormal}} & \multicolumn{3}{c}{\textbf{Exponential}} \\
\cmidrule(lr){2-4} \cmidrule(lr){5-7} \cmidrule(lr){8-10}
 & \textbf{Coverage} & \textbf{Length (\(\downarrow\))} & \textbf{ConMAE (\(\downarrow\))} 
 & \textbf{Coverage} & \textbf{Length (\(\downarrow\))} & \textbf{ConMAE (\(\downarrow\))} 
 & \textbf{Coverage} & \textbf{Length (\(\downarrow\))} & \textbf{ConMAE (\(\downarrow\))} \\
\midrule
Oracle   & 0.9000 \std{0.0000} & 1.5922 \std{0.0101} & 0.0000 \std{0.0000} & 0.9000 \std{0.0000} & 0.9591 \std{0.0061} & 0.0000 \std{0.0000} & 0.9000 \std{0.0000} & 1.1145 \std{0.0070} & 0.0000 \std{0.0000} \\
\midrule
Split    & 0.8995 \std{0.0050} & 1.8735 \std{0.0404} & 0.1119 \std{0.0030} & 0.9000 \std{0.0051} & 1.2377 \std{0.0301} & 0.0964 \std{0.0036} & 0.8989 \std{0.0031} & 1.5679 \std{0.0217} & 0.0937 \std{0.0036} \\
CPL      & 0.9010 \std{0.0091} & \textbf{1.5065} \std{0.0382} & 0.0764 \std{0.0061} & 0.9009 \std{0.0047} & 0.9741 \std{0.0229} & 0.0509 \std{0.0039} & 0.8968 \std{0.0076} & 1.1524 \std{0.0416} & 0.0346 \std{0.0061} \\
Naive CC & 0.9001 \std{0.0033} & 1.5902 \std{0.0159} & 0.0093 \std{0.0021} & 0.9001 \std{0.0028} & 0.9694 \std{0.0171} & 0.0078 \std{0.0025} & 0.8962 \std{0.0029} & 1.2270 \std{0.0431} & 0.0092 \std{0.0026} \\
CQR      & 0.9013 \std{0.0043} & 1.5996 \std{0.0167} & 0.0091 \std{0.0020} & 0.9006 \std{0.0050} & 1.1123 \std{0.0317} & 0.0127 \std{0.0043} & 0.8973 \std{0.0029} & 1.4251 \std{0.0186} & 0.0164 \std{0.0045} \\
CPCP     & 0.9027 \std{0.0130} & 1.6046 \std{0.0362} & 0.0189 \std{0.0052} & 0.9014 \std{0.0151} & 0.9813 \std{0.0471} & 0.0169 \std{0.0048} & 0.8933 \std{0.0090} & 1.2276 \std{0.0485} & 0.0149 \std{0.0071} \\
RCP      & 0.9004 \std{0.0045} & 1.5946 \std{0.0349} & 0.0136 \std{0.0037} & 0.9003 \std{0.0047} & 0.9759 \std{0.0293} & 0.0123 \std{0.0046} & 0.8995 \std{0.0043} & 1.2496 \std{0.0345} & 0.0101 \std{0.0029} \\
C-HDR    & 0.9005 \std{0.0038} & 1.5988 \std{0.0217} & 0.0086 \std{0.0021} & 0.9008 \std{0.0030} & 0.9797 \std{0.0220} & 0.0096 \std{0.0021} & 0.8971 \std{0.0024} & 1.2514 \std{0.0328} & 0.0080 \std{0.0024} \\
CHR      & 0.9025 \std{0.0043} & 1.6207 \std{0.0195} & 0.0097 \std{0.0019} & 0.9031 \std{0.0055} & 0.9894 \std{0.0142} & 0.0111 \std{0.0018} & 0.9019 \std{0.0025} & 1.1662 \std{0.0157} & 0.0113 \std{0.0023} \\
\rowcolor{gray!15} \textbf{CoCP} & 0.9007 \std{0.0037} & 1.5977 \std{0.0153} & \textbf{0.0066} \std{0.0013} & 0.9002 \std{0.0029} & \textbf{0.9599} \std{0.0132} & \textbf{0.0066} \std{0.0019} & 0.8975 \std{0.0026} & \textbf{1.1277} \std{0.0121} & \textbf{0.0077} \std{0.0022} \\
\bottomrule
\end{tabular}
}
\end{table}

\begin{figure}[ht!]
  \centering
  \includegraphics[width=\textwidth]{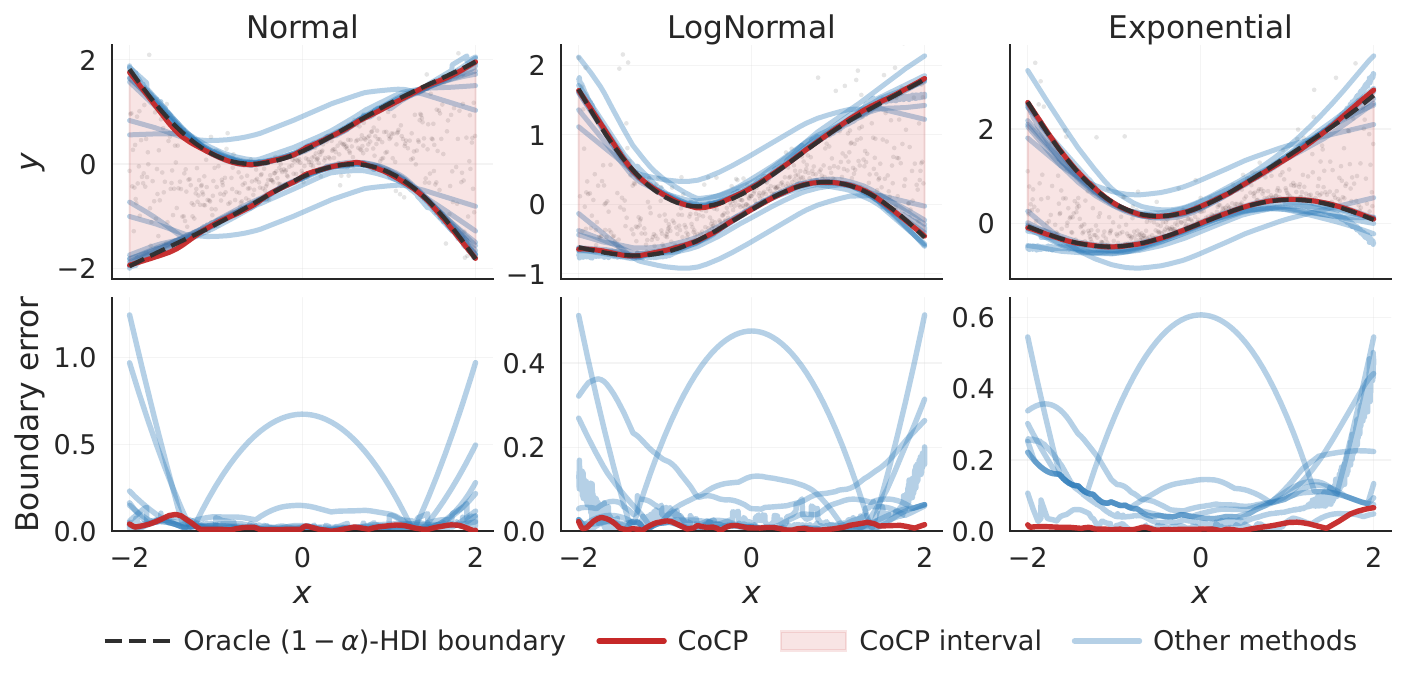}
  \caption{Interval geometry on synthetic benchmarks. \textbf{Top:} CoCP and baseline intervals vs. the oracle \((1-\alpha)\)-HDI. \textbf{Bottom:} Boundary error (mean absolute error between predicted and oracle endpoints). While most baselines perform comparably under symmetric noise, under skewness they suffer from either severe mis-centering or full-density estimation artifacts. CoCP avoids these issues, maintaining tight geometric alignment with the oracle HDI across all settings.}
  \label{fig:compare_interval}
\end{figure}

\subsection{Real Data}\label{sec:exp-real}

\paragraph{Setup.}
We evaluate on seven commonly used real-world regression datasets spanning different feature dimensions and noise characteristics: \textit{Bike} \cite{bike}, \textit{Bio} \cite{bio}, \textit{Blog} \cite{blog}, \textit{Facebook-1} \cite{facebook}, \textit{Facebook-2} \cite{facebook}, \textit{Homes} \cite{homes}, and \textit{Supercon.} \cite{superconductivity} (summarized in Appendix~\ref{app:exp-real-details}). Because true conditional coverage is unavailable, we report efficiency measures (Length) and conditional-reliability diagnostics (MSCE/WSC/ERT), all computed on the preprocessed response scale (Appendix~\ref{app:exp-preprocess}).

\begin{table}[ht!]
\centering
\caption{Experimental results on real-world datasets (mean \(\pm\) std over 10 splits, \(\alpha=0.1\)). Best values per dataset/metric (excluding Coverage) are highlighted in bold.}
\label{tab:real_world_results}
\resizebox{\textwidth}{!}{
\begin{tabular}{@{} l l c c c c c c @{}}
\toprule
\textbf{Dataset} & \textbf{Method} & \textbf{Coverage} & \textbf{Length (\(\downarrow\))} & \textbf{MSCE (\(\downarrow\))} & \textbf{WSC (\(\uparrow\))} & \textbf{\(\ell_1\)-ERT (\(\downarrow\))} & \textbf{\(\ell_2\)-ERT (\(\downarrow\))} \\
\midrule

\multirow{8}{*}{\textit{Bike}}
 & Split & 0.8980 \std{0.0111} & 0.7700 \std{0.0360} & 0.0041 \std{0.0016} & 0.6974 \std{0.0259} & 0.0663 \std{0.0063} & 0.0080 \std{0.0022} \\
 & CPL   & 0.9146 \std{0.0100} & 0.7938 \std{0.0373} & 0.0021 \std{0.0015} & 0.7800 \std{0.0515} & 0.0369 \std{0.0087} & 0.0022 \std{0.0023} \\
 & CQR   & 0.8993 \std{0.0040} & 0.7856 \std{0.1255} & 0.0018 \std{0.0006} & 0.7573 \std{0.0180} & 0.0350 \std{0.0127} & 0.0020 \std{0.0015} \\
 & CPCP  & 0.9030 \std{0.0162} & 0.7475 \std{0.0515} & 0.0010 \std{0.0003} & 0.7933 \std{0.0251} & 0.0206 \std{0.0090} & 0.0003 \std{0.0007} \\
 & RCP   & 0.9014 \std{0.0096} & 0.7497 \std{0.0474} & 0.0013 \std{0.0006} & 0.7787 \std{0.0158} & 0.0250 \std{0.0125} & 0.0008 \std{0.0009} \\
 & C-HDR & 0.8991 \std{0.0089} & 0.7617 \std{0.0327} & 0.0015 \std{0.0005} & 0.7740 \std{0.0123} & 0.0270 \std{0.0133} & 0.0013 \std{0.0014} \\
 & CHR   & 0.9067 \std{0.0122} & 0.7018 \std{0.0349} & 0.0013 \std{0.0004} & 0.7919 \std{0.0218} & 0.0279 \std{0.0095} & 0.0011 \std{0.0007} \\
 \rowcolor{gray!15} \cellcolor{white} & \textbf{CoCP}  & 0.8988 \std{0.0059} & \textbf{0.6053} \std{0.0190} & \textbf{0.0006} \std{0.0004} & \textbf{0.7982} \std{0.0096} & \textbf{0.0131} \std{0.0078} & \textbf{-0.0001} \std{0.0004} \\
\midrule

\multirow{8}{*}{\textit{Bio}}
 & Split & 0.9032 \std{0.0049} & 1.7584 \std{0.0301} & 0.0006 \std{0.0002} & 0.8258 \std{0.0070} & 0.0241 \std{0.0044} & 0.0009 \std{0.0003} \\
 & CPL   & 0.9057 \std{0.0081} & 1.6358 \std{0.0449} & 0.0005 \std{0.0002} & 0.8438 \std{0.0130} & 0.0161 \std{0.0042} & 0.0004 \std{0.0003} \\
 & CQR   & 0.8986 \std{0.0052} & 1.5251 \std{0.0164} & 0.0003 \std{0.0001} & 0.8455 \std{0.0079} & 0.0088 \std{0.0028} & 0.0001 \std{0.0001} \\
 & CPCP  & 0.9048 \std{0.0072} & 1.6452 \std{0.0413} & 0.0004 \std{0.0002} & 0.8463 \std{0.0119} & 0.0134 \std{0.0049} & 0.0002 \std{0.0002} \\
 & RCP   & 0.9013 \std{0.0042} & 1.6355 \std{0.0265} & 0.0003 \std{0.0003} & 0.8403 \std{0.0090} & 0.0074 \std{0.0063} & 0.0003 \std{0.0004} \\
 & C-HDR & 0.9019 \std{0.0046} & 1.5601 \std{0.0292} & 0.0006 \std{0.0004} & 0.8328 \std{0.0153} & 0.0179 \std{0.0067} & 0.0006 \std{0.0004} \\
 & CHR   & 0.9017 \std{0.0072} & 1.5145 \std{0.0350} & 0.0003 \std{0.0001} & 0.8460 \std{0.0097} & 0.0106 \std{0.0044} & 0.0001 \std{0.0001} \\
 \rowcolor{gray!15} \cellcolor{white} & \textbf{CoCP}  & 0.9013 \std{0.0056} & \textbf{1.4231} \std{0.0142} & \textbf{0.0002} \std{0.0001} & \textbf{0.8478} \std{0.0091} & \textbf{0.0059} \std{0.0038} & \textbf{-0.0000} \std{0.0001} \\
\midrule

\multirow{8}{*}{\textit{Blog}}
 & Split & 0.8998 \std{0.0040} & 3.7416 \std{0.0604} & 0.0048 \std{0.0007} & 0.7349 \std{0.0140} & 0.0726 \std{0.0043} & 0.0068 \std{0.0015} \\
 & CPL   & 0.9179 \std{0.0069} & 3.7549 \std{0.1810} & 0.0012 \std{0.0004} & \textbf{0.8503} \std{0.0093} & 0.0377 \std{0.0060} & -0.0000 \std{0.0008} \\
 & CQR   & 0.8973 \std{0.0027} & 2.7947 \std{0.1018} & 0.0014 \std{0.0004} & 0.8087 \std{0.0101} & 0.0300 \std{0.0048} & -0.0007 \std{0.0007} \\
 & CPCP  & 0.8992 \std{0.0061} & 3.3285 \std{0.0708} & 0.0005 \std{0.0003} & 0.8328 \std{0.0137} & 0.0353 \std{0.0075} & 0.0003 \std{0.0013} \\
 & RCP   & 0.8980 \std{0.0043} & 3.3837 \std{0.0917} & 0.0008 \std{0.0004} & 0.8154 \std{0.0130} & 0.0359 \std{0.0037} & 0.0009 \std{0.0009} \\
 & C-HDR & 0.8978 \std{0.0023} & 3.1153 \std{0.2141} & 0.0011 \std{0.0008} & 0.8230 \std{0.0216} & 0.0268 \std{0.0063} & -0.0005 \std{0.0012} \\
 & CHR   & 0.9093 \std{0.0049} & \textbf{2.5531} \std{0.0744} & 0.0006 \std{0.0003} & 0.8479 \std{0.0117} & 0.0208 \std{0.0047} & \textbf{-0.0018} \std{0.0006} \\
 \rowcolor{gray!15} \cellcolor{white} & \textbf{CoCP}  & 0.8994 \std{0.0046} & 2.7025 \std{0.0897} & \textbf{0.0004} \std{0.0001} & 0.8416 \std{0.0069} & \textbf{0.0177} \std{0.0036} & -0.0016 \std{0.0007} \\
\midrule

\multirow{8}{*}{\textit{Facebook-1}}
 & Split & 0.9012 \std{0.0053} & 2.1838 \std{0.0341} & 0.0021 \std{0.0004} & 0.7814 \std{0.0154} & 0.0547 \std{0.0051} & 0.0037 \std{0.0009} \\
 & CPL   & 0.8993 \std{0.0069} & 1.8935 \std{0.0756} & 0.0004 \std{0.0002} & 0.8425 \std{0.0090} & 0.0172 \std{0.0061} & 0.0002 \std{0.0005} \\
 & CQR   & 0.8984 \std{0.0060} & 1.7500 \std{0.0462} & 0.0008 \std{0.0003} & 0.8144 \std{0.0099} & 0.0249 \std{0.0074} & 0.0006 \std{0.0007} \\
 & CPCP  & 0.8991 \std{0.0053} & 1.9626 \std{0.0567} & 0.0005 \std{0.0003} & 0.8252 \std{0.0197} & 0.0235 \std{0.0108} & 0.0005 \std{0.0011} \\
 & RCP   & 0.8998 \std{0.0049} & 1.9151 \std{0.0394} & 0.0003 \std{0.0001} & 0.8395 \std{0.0117} & 0.0146 \std{0.0049} & 0.0000 \std{0.0004} \\
 & C-HDR & 0.9006 \std{0.0027} & 1.8299 \std{0.0381} & 0.0007 \std{0.0005} & 0.8287 \std{0.0148} & 0.0158 \std{0.0074} & 0.0002 \std{0.0008} \\
 & CHR   & 0.9059 \std{0.0052} & 1.6920 \std{0.0344} & 0.0006 \std{0.0002} & 0.8297 \std{0.0074} & 0.0188 \std{0.0080} & -0.0001 \std{0.0003} \\
 \rowcolor{gray!15} \cellcolor{white} & \textbf{CoCP}  & 0.8998 \std{0.0047} & \textbf{1.6380} \std{0.0232} & \textbf{0.0002} \std{0.0001} & \textbf{0.8470} \std{0.0074} & \textbf{0.0056} \std{0.0051} & \textbf{-0.0004} \std{0.0003} \\
\midrule

\multirow{8}{*}{\textit{Facebook-2}}
 & Split & 0.8991 \std{0.0026} & 2.0821 \std{0.0192} & 0.0015 \std{0.0004} & 0.7964 \std{0.0080} & 0.0560 \std{0.0026} & 0.0035 \std{0.0006} \\
 & CPL   & 0.8973 \std{0.0034} & 1.7975 \std{0.0381} & 0.0002 \std{0.0001} & 0.8482 \std{0.0093} & 0.0155 \std{0.0037} & 0.0003 \std{0.0003} \\
 & CQR   & 0.8977 \std{0.0036} & 1.6684 \std{0.0315} & 0.0006 \std{0.0004} & 0.8269 \std{0.0124} & 0.0246 \std{0.0057} & 0.0007 \std{0.0005} \\
 & CPCP  & 0.8994 \std{0.0026} & 1.8508 \std{0.0430} & 0.0003 \std{0.0002} & 0.8433 \std{0.0148} & 0.0208 \std{0.0098} & 0.0006 \std{0.0008} \\
 & RCP   & 0.8996 \std{0.0031} & 1.8083 \std{0.0380} & 0.0003 \std{0.0001} & 0.8477 \std{0.0089} & 0.0160 \std{0.0028} & 0.0004 \std{0.0002} \\
 & C-HDR & 0.8986 \std{0.0017} & 1.7364 \std{0.0254} & 0.0004 \std{0.0002} & 0.8402 \std{0.0096} & 0.0206 \std{0.0045} & 0.0005 \std{0.0003} \\
 & CHR   & 0.9020 \std{0.0032} & \textbf{1.4582} \std{0.0400} & 0.0005 \std{0.0002} & 0.8361 \std{0.0089} & 0.0230 \std{0.0033} & 0.0004 \std{0.0002} \\
 \rowcolor{gray!15} \cellcolor{white} & \textbf{CoCP}  & 0.8988 \std{0.0024} & 1.5376 \std{0.0141} & \textbf{0.0001} \std{0.0001} & \textbf{0.8547} \std{0.0072} & \textbf{0.0103} \std{0.0023} & \textbf{-0.0001} \std{0.0001} \\
\midrule

\multirow{8}{*}{\textit{Homes}}
 & Split & 0.9000 \std{0.0087} & 0.5964 \std{0.0287} & 0.0112 \std{0.0016} & 0.5788 \std{0.0302} & 0.0931 \std{0.0044} & 0.0191 \std{0.0022} \\
 & CPL   & 0.9087 \std{0.0119} & 0.6331 \std{0.0377} & 0.0009 \std{0.0004} & 0.7909 \std{0.0335} & 0.0357 \std{0.0093} & 0.0021 \std{0.0015} \\
 & CQR   & 0.9034 \std{0.0059} & 0.5664 \std{0.0189} & 0.0007 \std{0.0005} & 0.8192 \std{0.0158} & 0.0118 \std{0.0078} & 0.0003 \std{0.0006} \\
 & CPCP  & 0.8978 \std{0.0094} & 0.5888 \std{0.0212} & 0.0009 \std{0.0007} & 0.7984 \std{0.0413} & 0.0201 \std{0.0117} & 0.0007 \std{0.0012} \\
 & RCP   & 0.8992 \std{0.0047} & 0.5845 \std{0.0204} & 0.0006 \std{0.0004} & 0.8020 \std{0.0207} & 0.0169 \std{0.0082} & 0.0008 \std{0.0009} \\
 & C-HDR & 0.9003 \std{0.0046} & 0.5373 \std{0.0154} & 0.0007 \std{0.0003} & 0.8202 \std{0.0088} & 0.0080 \std{0.0067} & -0.0001 \std{0.0003} \\
 & CHR   & 0.9013 \std{0.0091} & 0.5428 \std{0.0165} & 0.0007 \std{0.0006} & 0.8199 \std{0.0152} & 0.0149 \std{0.0085} & 0.0002 \std{0.0007} \\
 \rowcolor{gray!15} \cellcolor{white} & \textbf{CoCP}  & 0.9017 \std{0.0065} & \textbf{0.5242} \std{0.0087} & \textbf{0.0003} \std{0.0002} & \textbf{0.8283} \std{0.0145} & \textbf{0.0052} \std{0.0070} & \textbf{-0.0003} \std{0.0003} \\
\midrule

\multirow{8}{*}{\textit{Supercon.}}
 & Split & 0.9036 \std{0.0065} & 1.0186 \std{0.0442} & 0.0066 \std{0.0010} & 0.7061 \std{0.0246} & 0.0875 \std{0.0035} & 0.0088 \std{0.0009} \\
 & CPL   & 0.9068 \std{0.0134} & 0.9371 \std{0.0325} & 0.0012 \std{0.0005} & 0.8074 \std{0.0298} & 0.0333 \std{0.0093} & 0.0015 \std{0.0014} \\
 & CQR   & 0.9013 \std{0.0072} & 0.9059 \std{0.0345} & 0.0006 \std{0.0002} & 0.8094 \std{0.0104} & 0.0211 \std{0.0069} & 0.0000 \std{0.0007} \\
 & CPCP  & 0.8971 \std{0.0124} & 0.9186 \std{0.0426} & 0.0007 \std{0.0007} & 0.7998 \std{0.0371} & 0.0269 \std{0.0084} & 0.0007 \std{0.0015} \\
 & RCP   & 0.9049 \std{0.0069} & 0.9695 \std{0.0393} & 0.0007 \std{0.0004} & 0.8125 \std{0.0187} & 0.0236 \std{0.0113} & 0.0004 \std{0.0014} \\
 & C-HDR & 0.9050 \std{0.0068} & 1.0642 \std{0.0455} & \textbf{0.0003} \std{0.0002} & 0.8221 \std{0.0128} & 0.0207 \std{0.0062} & -0.0001 \std{0.0005} \\
 & CHR   & 0.9131 \std{0.0063} & 0.9875 \std{0.0396} & 0.0013 \std{0.0004} & 0.8111 \std{0.0140} & 0.0287 \std{0.0053} & 0.0003 \std{0.0007} \\
 \rowcolor{gray!15} \cellcolor{white} & \textbf{CoCP}  & 0.9041 \std{0.0044} & \textbf{0.8167} \std{0.0156} & \textbf{0.0003} \std{0.0002} & \textbf{0.8245} \std{0.0069} & \textbf{0.0127} \std{0.0091} & \textbf{-0.0006} \std{0.0006} \\
\bottomrule
\end{tabular}%
}
\end{table}

\paragraph{Main results.}
Table~\ref{tab:real_world_results} summarizes the performance across seven real-world datasets. CoCP achieves the shortest average interval lengths on five datasets while satisfying the nominal marginal coverage. On the remaining two (\textit{Facebook-2} and \textit{Blog}), CoCP remains highly competitive as the runner-up to CHR. This slight trade-off in length is offset by CoCP’s superior conditional reliability: on \textit{Facebook-2}, it outperforms CHR across all four diagnostics (MSCE, WSC, \(\ell_1\)-ERT, and \(\ell_2\)-ERT), while on \textit{Blog}, it maintains lower MSCE and \(\ell_1\)-ERT. Overall, CoCP consistently demonstrates robust empirical conditional reliability. It achieves the lowest MSCE and \(\ell_1\)-ERT across all datasets, and leads in WSC and \(\ell_2\)-ERT on six. These results confirm that CoCP effectively enhances local reliability while providing improved efficiency in prediction intervals.

\paragraph{Efficiency and Robustness.} CoCP maintains the inference speed of standard CQR (typically $<0.05$ seconds) despite its alternating optimization (Appendix~\ref{app:runtime}). Its offline cost is substantially higher: for example, on \textit{Facebook-2}, the base implementation takes about 1844 seconds versus 66 seconds for CQR. The Fast variant (Appendix~\ref{app:cocp_fast}) reduces the CoCP time to about 463 seconds via fold parallelism and strict early stopping, but CoCP remains best suited to settings where retraining is infrequent and inference is repeated. Appendix~\ref{sec:adaptive_beta} gives a practical adaptive-\(\beta\) strategy based on an initial local-radius estimate; the asymptotic theory
applies to vanishing smoothing schedules.

\paragraph{Further analysis.}
Appendix~\ref{app:additional-res} reports several supplementary investigations.
A $T\!=\!0$ ablation (Appendix~\ref{app:center-tradeoff}) isolates the contribution of center refinement: raising $T$ from 0 to 5 shortens intervals on nine of ten datasets, while the unrefined center remains competitive with the Split predictor in point-prediction error. A separate predictor can thus coexist with CoCP for point decisions.
Fold-disagreement analysis (Appendix~\ref{app:fold-unimodality}) suggests that variance reduction contributes to the stronger benefit of ensembling on real data, while local dip tests provide little evidence of widespread residual multimodality.
A validation-tuned asymmetric CQR baseline (Appendix~\ref{app:acqr}) substantially improves over equal-tailed CQR under skewness, yet CoCP remains shorter on nine of ten datasets and achieves lower conditional error on all ten, confirming that its gains extend beyond relaxing the equal-tailed constraint.
The ranking is stable across coverage levels $\alpha\in\{0.05,0.1,0.2\}$ even without per-level tuning (Appendix~\ref{app:coverage-levels}), and persists for training sets as small as $n\!=\!2{,}000$, calibration sets as small as $n_\mathrm{cal}\!=\!500$, and three additional real datasets with $n\le 2{,}200$ (Appendix~\ref{app:data-size-sensitivity}).

\section{Conclusions}

This work presented CoCP, a novel conformal regression framework that overcomes the inefficiency of traditional center-invariant intervals. By framing interval construction as a boundary-balancing problem within a folded geometry, CoCP co-optimizes interval translation and scaling. The resulting dynamic push-pull mechanism allows the interval to autonomously seek high-density regions, recovering the optimal HDI under unimodality. CoCP guarantees distribution-free finite-sample marginal coverage from split conformal prediction, while achieving asymptotic conditional coverage and optimal interval length under standard regularity and consistency conditions. Empirically, it produces tighter prediction intervals with stronger conditional reliability across the evaluated benchmarks.

While highly effective, CoCP indeed involves certain trade-offs. Theoretically, its optimality assumes unimodal target distributions, leaving multimodal regions as an open challenge. It also relies on the consistency of base estimators, which is not strictly guaranteed for non-convex neural networks. Practically, the alternating \(K\)-fold optimization increases offline training costs (though the inference time is negligible), and performance depends on carefully tuning the smoothing temperature \(\beta\). Furthermore, extending this geometric push-pull mechanism to multivariate or structured outputs with generative or representation-based construction pipelines \cite{dheur2025aunified} remains nontrivial. We provide a more comprehensive discussion of these limitations and exciting avenues for future work in Appendix~\ref{app:limitations}.

\begin{ack}
This work was supported by the National Natural Science Foundation of China (Grant No. 62506315), Hong Kong Research Grant Council (CityU 21213626), Guangdong Provincial Natural Science Foundation (Grant No. 2026A1515011937), and City University of Hong Kong (Grant No. 9610639). We thank the reviewers and the Area Chair for their valuable comments.

The authors declare no competing interests.
\end{ack}

\bibliographystyle{plainnat}
\bibliography{references}

@article{hyndman1996computing,
 ISSN = {00031305},
 author = {Rob J. Hyndman},
 journal = {The American Statistician},
 number = {2},
 pages = {120--126},
 publisher = {[American Statistical Association, Taylor & Francis, Ltd.]},
 title = {Computing and Graphing Highest Density Regions},
 urldate = {2026-02-15},
 volume = {50},
 year = {1996}
}

@book{vovk2005algorithmic,
  title={Algorithmic learning in a random world},
  author={Vovk, Vladimir and Gammerman, Alexander and Shafer, Glenn},
  year={2005},
  publisher={Springer}
}

@article{shafer2008atutorial,
author = {Shafer, Glenn and Vovk, Vladimir},
title = {A Tutorial on Conformal Prediction},
year = {2008},
issue_date = {6/1/2008},
publisher = {JMLR.org},
volume = {9},
issn = {1532-4435},
journal = {J. Mach. Learn. Res.},
month = jun,
pages = {371–421},
numpages = {51}
}

@article{lei2014distribution,
  title={Distribution-free prediction bands for non-parametric regression},
  author={Lei, Jing and Wasserman, Larry},
  journal={Journal of the Royal Statistical Society Series B: Statistical Methodology},
  volume={76},
  number={1},
  pages={71--96},
  year={2014},
  publisher={Oxford University Press}
}

@article{lei2018distribution,
  title={Distribution-free predictive inference for regression},
  author={Lei, Jing and G’Sell, Max and Rinaldo, Alessandro and Tibshirani, Ryan J and Wasserman, Larry},
  journal={Journal of the American Statistical Association},
  volume={113},
  number={523},
  pages={1094--1111},
  year={2018},
  publisher={Taylor \& Francis}
}

@article{angelopoulos2023gentle,
    author = {Angelopoulos, Anastasios N. and Bates, Stephen},
    title = {Conformal Prediction: A Gentle Introduction},
    journal = {Foundations and Trends in Machine Learning},
    volume = {16},
    number = {4},
    pages = {494-591},
    year = {2023},
    month = {03},
    issn = {1935-8237},
    doi = {10.1561/2200000101}
}

@article{barber2021thelimits,
    author = {Foygel Barber, Rina and Candès, Emmanuel J and Ramdas, Aaditya and Tibshirani, Ryan J},
    title = {The limits of distribution-free conditional predictive inference},
    journal = {Information and Inference: A Journal of the IMA},
    volume = {10},
    number = {2},
    pages = {455-482},
    year = {2021},
    month = {06},
    issn = {2049-8772},
    doi = {10.1093/imaiai/iaaa017}
}

@article{vovk2015cross,
  title={Cross-conformal predictors},
  author={Vovk, Vladimir},
  journal={Annals of Mathematics and Artificial Intelligence},
  volume={74},
  number={1},
  pages={9--28},
  year={2015},
  publisher={Springer}
}

@InProceedings{gasparin2025improving,
  title = 	 {Improving the Statistical Efficiency of Cross-Conformal Prediction},
  author =       {Gasparin, Matteo and Ramdas, Aaditya},
  booktitle = 	 {Proceedings of the 42nd International Conference on Machine Learning},
  pages = 	 {18848--18867},
  year = 	 {2025},
  editor = 	 {Singh, Aarti and Fazel, Maryam and Hsu, Daniel and Lacoste-Julien, Simon and Berkenkamp, Felix and Maharaj, Tegan and Wagstaff, Kiri and Zhu, Jerry},
  volume = 	 {267},
  series = 	 {Proceedings of Machine Learning Research},
  month = 	 {13--19 Jul},
  publisher =    {PMLR}
}

@inproceedings{romano2019conformalized,
 author = {Romano, Yaniv and Patterson, Evan and Candes, Emmanuel},
 booktitle = {Advances in Neural Information Processing Systems},
 editor = {H. Wallach and H. Larochelle and A. Beygelzimer and F. d\textquotesingle Alch\'{e}-Buc and E. Fox and R. Garnett},
 pages = {},
 publisher = {Curran Associates, Inc.},
 title = {Conformalized Quantile Regression},
 volume = {32},
 year = {2019}
}

@article{sesia2020comparison,
  title={A comparison of some conformal quantile regression methods},
  author={Sesia, Matteo and Cand{\`e}s, Emmanuel J},
  journal={Stat},
  volume={9},
  number={1},
  pages={e261},
  year={2020},
  publisher={Wiley Online Library}
}

@inproceedings{sesia2021conformal,
 author = {Sesia, Matteo and Romano, Yaniv},
 booktitle = {Advances in Neural Information Processing Systems},
 editor = {M. Ranzato and A. Beygelzimer and Y. Dauphin and P.S. Liang and J. Wortman Vaughan},
 pages = {6304--6315},
 publisher = {Curran Associates, Inc.},
 title = {Conformal Prediction using Conditional Histograms},
 volume = {34},
 year = {2021}
}

@article{guan2023localized,
    author = {Guan, Leying},
    title = {Localized conformal prediction: a generalized inference framework for conformal prediction},
    journal = {Biometrika},
    volume = {110},
    number = {1},
    pages = {33-50},
    year = {2023},
    month = {03},
    issn = {1464-3510},
    doi = {10.1093/biomet/asac040}
}

@article{hore2025conformal,
    author = {Hore, Rohan and Barber, Rina Foygel},
    title = {Conformal prediction with local weights: randomization enables robust guarantees},
    journal = {Journal of the Royal Statistical Society Series B: Statistical Methodology},
    volume = {87},
    number = {2},
    pages = {549-578},
    year = {2025},
    month = {04},
    issn = {1369-7412},
    doi = {10.1093/jrsssb/qkae103}
}

@article{gibbs2025conformal,
    author = {Gibbs, Isaac and Cherian, John J and Candès, Emmanuel J},
    title = {Conformal prediction with conditional guarantees},
    journal = {Journal of the Royal Statistical Society Series B: Statistical Methodology},
    volume = {87},
    number = {4},
    pages = {1100-1126},
    year = {2025},
    month = {03},
    issn = {1369-7412},
    doi = {10.1093/jrsssb/qkaf008},
}

@InProceedings{bairaktari2025kandinsky,
  title = 	 {Kandinsky Conformal Prediction: Beyond Class- and Covariate-Conditional Coverage},
  author =       {Bairaktari, Konstantina and Wu, Jiayun and Wu, Steven},
  booktitle = 	 {Proceedings of the 42nd International Conference on Machine Learning},
  pages = 	 {2581--2602},
  year = 	 {2025},
  editor = 	 {Singh, Aarti and Fazel, Maryam and Hsu, Daniel and Lacoste-Julien, Simon and Berkenkamp, Felix and Maharaj, Tegan and Wagstaff, Kiri and Zhu, Jerry},
  volume = 	 {267},
  series = 	 {Proceedings of Machine Learning Research},
  month = 	 {13--19 Jul},
  publisher =    {PMLR}
}

@article{
deutschmann2024adaptive,
title={Adaptive Conformal Regression with Split-Jackknife+ Scores},
author={Nicolas Deutschmann and Mattia Rigotti and Maria Rodriguez Martinez},
journal={Transactions on Machine Learning Research},
issn={2835-8856},
year={2024},
url={https://openreview.net/forum?id=1fbTGC3BUD},
note={}
}

@InProceedings{colombo2023training,
  title = 	 {On training locally adaptive CP},
  author =       {Colombo, Nicolo},
  booktitle = 	 {Proceedings of the Twelfth Symposium on Conformal
 and Probabilistic Prediction with Applications},
  pages = 	 {384--398},
  year = 	 {2023},
  editor = 	 {Papadopoulos, Harris and Nguyen, Khuong An and Boström, Henrik and Carlsson, Lars},
  volume = 	 {204},
  series = 	 {Proceedings of Machine Learning Research},
  month = 	 {13--15 Sep},
  publisher =    {PMLR},
}

@InProceedings{plassier2025rectifying,
  title = 	 {Rectifying Conformity Scores for Better Conditional Coverage},
  author =       {Plassier, Vincent and Fishkov, Alexander and Dheur, Victor and Guizani, Mohsen and Ben Taieb, Souhaib and Panov, Maxim and Moulines, Eric},
  booktitle = 	 {Proceedings of the 42nd International Conference on Machine Learning},
  pages = 	 {49459--49492},
  year = 	 {2025},
  editor = 	 {Singh, Aarti and Fazel, Maryam and Hsu, Daniel and Lacoste-Julien, Simon and Berkenkamp, Felix and Maharaj, Tegan and Wagstaff, Kiri and Zhu, Jerry},
  volume = 	 {267},
  series = 	 {Proceedings of Machine Learning Research},
  month = 	 {13--19 Jul},
  publisher =    {PMLR}
}

@inproceedings{tibshirani2019conformal,
 author = {Tibshirani, Ryan J and Foygel Barber, Rina and Candes, Emmanuel and Ramdas, Aaditya},
 booktitle = {Advances in Neural Information Processing Systems},
 editor = {H. Wallach and H. Larochelle and A. Beygelzimer and F. d\textquotesingle Alch\'{e}-Buc and E. Fox and R. Garnett},
 pages = {},
 publisher = {Curran Associates, Inc.},
 title = {Conformal Prediction Under Covariate Shift},
 volume = {32},
 year = {2019}
}

@article{chernozhukov2021distributional,
  title={Distributional conformal prediction},
  author={Chernozhukov, Victor and W{\"u}thrich, Kaspar and Zhu, Yinchu},
  journal={Proceedings of the National Academy of Sciences},
  volume={118},
  number={48},
  pages={e2107794118},
  year={2021},
  publisher={National Academy of Sciences}
}

@article{izbicki2022cd,
  title={Cd-split and hpd-split: Efficient conformal regions in high dimensions},
  author={Izbicki, Rafael and Shimizu, Gilson and Stern, Rafael B},
  journal={Journal of Machine Learning Research},
  volume={23},
  number={87},
  pages={1--32},
  year={2022}
}

@InProceedings{liu2026speedcp,
  title = 	 {{S}peed{CP}: Fast Kernel-based Conditional Conformal Prediction},
  author =       {Liu, Yating and Jung, Yeo Jin and Wu, Zixuan and Jeong, Sowon and Donnat, Claire},
  booktitle = 	 {Proceedings of the 43rd International Conference on Machine Learning},
  pages = 	 {76980--77027},
  year = 	 {2026},
  editor = 	 {Zhang, Tong and Dudik, Miroslav and Jaggi, Martin and Agarwal, Alekh and Li, Sharon and Schuurmans, Dale and Zhu, Jerry and Berkenkamp, Felix and Dong, Hanze and Bietti, Alberto},
  volume = 	 {306},
  series = 	 {Proceedings of Machine Learning Research},
  month = 	 {06--11 Jul},
  publisher =    {PMLR}
}

@inproceedings{kiyani2024length,
 author = {Kiyani, Shayan and Pappas, George and Hassani, Hamed},
 booktitle = {Advances in Neural Information Processing Systems},
 doi = {10.52202/079017-3158},
 editor = {A. Globerson and L. Mackey and D. Belgrave and A. Fan and U. Paquet and J. Tomczak and C. Zhang},
 pages = {99519--99563},
 publisher = {Curran Associates, Inc.},
 title = {Length Optimization in Conformal Prediction},
 volume = {37},
 year = {2024}
}

@InProceedings{dheur2025aunified,
  title = 	 {A Unified Comparative Study with Generalized Conformity Scores for Multi-Output Conformal Regression},
  author =       {Dheur, Victor and Fontana, Matteo and Estievenart, Yorick and Desobry, Naomi and Ben Taieb, Souhaib},
  booktitle = 	 {Proceedings of the 42nd International Conference on Machine Learning},
  pages = 	 {13444--13485},
  year = 	 {2025},
  editor = 	 {Singh, Aarti and Fazel, Maryam and Hsu, Daniel and Lacoste-Julien, Simon and Berkenkamp, Felix and Maharaj, Tegan and Wagstaff, Kiri and Zhu, Jerry},
  volume = 	 {267},
  series = 	 {Proceedings of Machine Learning Research},
  month = 	 {13--19 Jul},
  publisher =    {PMLR}
}

@inproceedings{guo2026fast,
  title={Fast conformal prediction using conditional interquantile intervals},
  author={Guo, Naixin and Luo, Rui and Zhou, Zhixin},
  booktitle={Proceedings of the AAAI Conference on Artificial Intelligence},
  volume={40},
  pages={21468--21476},
  year={2026}
}

@InProceedings{gao2025volume,
  title = 	 {Volume Optimality in Conformal Prediction with Structured Prediction Sets},
  author =       {Gao, Chao and Shan, Liren and Srinivas, Vaidehi and Vijayaraghavan, Aravindan},
  booktitle = 	 {Proceedings of the 42nd International Conference on Machine Learning},
  pages = 	 {18495--18527},
  year = 	 {2025},
  editor = 	 {Singh, Aarti and Fazel, Maryam and Hsu, Daniel and Lacoste-Julien, Simon and Berkenkamp, Felix and Maharaj, Tegan and Wagstaff, Kiri and Zhu, Jerry},
  volume = 	 {267},
  series = 	 {Proceedings of Machine Learning Research},
  month = 	 {13--19 Jul},
  publisher =    {PMLR}
}

@inproceedings{luo2025conformal,
  title={Conformal thresholded intervals for efficient regression},
  author={Luo, Rui and Zhou, Zhixin},
  booktitle={Proceedings of the AAAI Conference on Artificial Intelligence},
  volume={39},
  pages={19216--19223},
  year={2025}
}

@book{casella2024statistical,
  title={Statistical inference},
  author={Casella, George and Berger, Roger},
  year={2024},
  publisher={Chapman and Hall/CRC}
}

@InProceedings{pearce2018high,
  title = 	 {High-Quality Prediction Intervals for Deep Learning: A Distribution-Free, Ensembled Approach},
  author =       {Pearce, Tim and Brintrup, Alexandra and Zaki, Mohamed and Neely, Andy},
  booktitle = 	 {Proceedings of the 35th International Conference on Machine Learning},
  pages = 	 {4075--4084},
  year = 	 {2018},
  editor = 	 {Dy, Jennifer and Krause, Andreas},
  volume = 	 {80},
  series = 	 {Proceedings of Machine Learning Research},
  month = 	 {10--15 Jul},
  publisher =    {PMLR}
}

@InProceedings{chen2026colorful,
  title = 	 {Colorful Pinball: Density-Weighted Quantile Regression for Conditional Guarantee of Conformal Prediction},
  author =       {Chen, Qianyi and Li, Bo},
  booktitle = 	 {Proceedings of the 43rd International Conference on Machine Learning},
  pages = 	 {15118--15145},
  year = 	 {2026},
  editor = 	 {Zhang, Tong and Dudik, Miroslav and Jaggi, Martin and Agarwal, Alekh and Li, Sharon and Schuurmans, Dale and Zhu, Jerry and Berkenkamp, Felix and Dong, Hanze and Bietti, Alberto},
  volume = 	 {306},
  series = 	 {Proceedings of Machine Learning Research},
  month = 	 {06--11 Jul},
  publisher =    {PMLR}
}

@InProceedings{bao2025areview,
  title = 	 {A Review and Comparative Analysis of Univariate Conformal Regression Methods},
  author =       {Bao, Jie and Colombo, Nicolo and Manokhin, Valery and Cao, Suqun and Luo, Rui},
  booktitle = 	 {Proceedings of the Fourteenth Symposium on Conformal and Probabilistic Prediction with Applications},
  pages = 	 {282--304},
  year = 	 {2025},
  editor = 	 {Nguyen, Khuong An and Luo, Zhiyuan and Papadopoulos, Harris and Löfström, Tuwe and Carlsson, Lars and Boström, Henrik},
  volume = 	 {266},
  series = 	 {Proceedings of Machine Learning Research},
  month = 	 {10--12 Sep},
  publisher =    {PMLR}
}

@InProceedings{braun2026conditional,
  title = 	 {Conditional Coverage Diagnostics for Conformal Prediction},
  author =       {Braun, Sacha and Holzm\"{u}ller, David and Jordan, Michael I. and Bach, Francis},
  booktitle = 	 {Proceedings of the 43rd International Conference on Machine Learning},
  pages = 	 {9757--9794},
  year = 	 {2026},
  editor = 	 {Zhang, Tong and Dudik, Miroslav and Jaggi, Martin and Agarwal, Alekh and Li, Sharon and Schuurmans, Dale and Zhu, Jerry and Berkenkamp, Felix and Dong, Hanze and Bietti, Alberto},
  volume = 	 {306},
  series = 	 {Proceedings of Machine Learning Research},
  month = 	 {06--11 Jul},
  publisher =    {PMLR}
}

@InProceedings{kiyani2024conformal,
  title = 	 {Conformal Prediction with Learned Features},
  author =       {Kiyani, Shayan and Pappas, George J. and Hassani, Hamed},
  booktitle = 	 {Proceedings of the 41st International Conference on Machine Learning},
  pages = 	 {24749--24769},
  year = 	 {2024},
  editor = 	 {Salakhutdinov, Ruslan and Kolter, Zico and Heller, Katherine and Weller, Adrian and Oliver, Nuria and Scarlett, Jonathan and Berkenkamp, Felix},
  volume = 	 {235},
  series = 	 {Proceedings of Machine Learning Research},
  month = 	 {21--27 Jul},
  publisher =    {PMLR}
}

@article{cauchois2021knowing,
author = {Cauchois, Maxime and Gupta, Suyash and Duchi, John C.},
title = {Knowing what you know: valid and validated confidence sets in multiclass and multilabel prediction},
year = {2021},
issue_date = {January 2021},
publisher = {JMLR.org},
volume = {22},
number = {1},
issn = {1532-4435},
journal = {J. Mach. Learn. Res.},
month = jan,
articleno = {81},
numpages = {42}
}

@article{romano2020classification,
  title={Classification with valid and adaptive coverage},
  author={Romano, Yaniv and Sesia, Matteo and Candes, Emmanuel},
  journal={Advances in neural information processing systems},
  volume={33},
  pages={3581--3591},
  year={2020}
}

@inproceedings{papadopoulos2002inductive,
  title={Inductive confidence machines for regression},
  author={Papadopoulos, Harris and Proedrou, Kostas and Vovk, Volodya and Gammerman, Alex},
  booktitle={European conference on machine learning},
  pages={345--356},
  year={2002},
  organization={Springer}
}

@misc{bike,
  author       = {Fanaee-T, Hadi},
  title        = {{Bike Sharing}},
  year         = {2013},
  howpublished = {UCI Machine Learning Repository},
  note         = {{DOI}: https://doi.org/10.24432/C5W894}
}

@misc{bio,
  author       = {Rana, Prashant},
  title        = {{Physicochemical Properties of Protein Tertiary Structure}},
  year         = {2013},
  howpublished = {UCI Machine Learning Repository},
  note         = {{DOI}: https://doi.org/10.24432/C5QW3H}
}

@misc{blog,
  author       = {Buza, Krisztian},
  title        = {{BlogFeedback}},
  year         = {2014},
  howpublished = {UCI Machine Learning Repository},
  note         = {{DOI}: https://doi.org/10.24432/C58S3F}
}

@misc{facebook,
  author       = {Singh, Kamaljot},
  title        = {{Facebook Comment Volume}},
  year         = {2015},
  howpublished = {UCI Machine Learning Repository},
  note         = {{DOI}: https://doi.org/10.24432/C5Q886}
}

@misc{homes,
  author       = {harlfoxem},
  title        = {House Sales in King County, USA},
  year         = {2016},
  publisher    = {Kaggle},
  url          = {https://www.kaggle.com/datasets/harlfoxem/housesalesprediction},
  note         = {Accessed: 2026-03-01} 
}

@misc{superconductivity,
  author       = {Hamidieh, Kam},
  title        = {{Superconductivity Data}},
  year         = {2018},
  howpublished = {UCI Machine Learning Repository},
  note         = {{DOI}: https://doi.org/10.24432/C53P47}
}

@misc{concrete,
  author       = {Yeh, I-Cheng},
  title        = {{Concrete Compressive Strength}},
  year         = {1998},
  howpublished = {UCI Machine Learning Repository},
  note         = {{DOI}: https://doi.org/10.24432/C5PK67}
}

@misc{community,
  author       = {Redmond, Michael},
  title        = {{Communities and Crime}},
  year         = {2002},
  howpublished = {UCI Machine Learning Repository},
  note         = {{DOI}: https://doi.org/10.24432/C53W3X}
}

@misc{star,
author = {C.M. Achilles and Helen Pate Bain and Fred Bellott and Jayne Boyd-Zaharias and Jeremy Finn and John Folger and John Johnston and Elizabeth Word},
publisher = {Harvard Dataverse},
title = {{Tennessee's Student Teacher Achievement Ratio (STAR) project}},
UNF = {UNF:3:Ji2Q+9HCCZAbw3csOdMNdA==},
year = {2008},
version = {V1},
doi = {10.7910/DVN/SIWH9F},
url = {https://doi.org/10.7910/DVN/SIWH9F}
}

@InProceedings{patel2024conformal,
  title = 	 {Conformal Contextual Robust Optimization},
  author =       {Patel, Yash P. and Rayan, Sahana and Tewari, Ambuj},
  booktitle = 	 {Proceedings of The 27th International Conference on Artificial Intelligence and Statistics},
  pages = 	 {2485--2493},
  year = 	 {2024},
  editor = 	 {Dasgupta, Sanjoy and Mandt, Stephan and Li, Yingzhen},
  volume = 	 {238},
  series = 	 {Proceedings of Machine Learning Research},
  month = 	 {02--04 May},
  publisher =    {PMLR}
}

@InProceedings{johansson2024reliable,
  title = 	 {Conformal Regression with Reject Option},
  author =       {Johansson, Ulf and S\"{o}nstr\"{o}d, Cecilia and Bostr\"{o}m, Henrik},
  booktitle = 	 {Proceedings of the Thirteenth Symposium on Conformal and Probabilistic Prediction with Applications},
  pages = 	 {277--294},
  year = 	 {2024},
  editor = 	 {Vantini, Simone and Fontana, Matteo and Solari, Aldo and Boström, Henrik and Carlsson, Lars},
  volume = 	 {230},
  series = 	 {Proceedings of Machine Learning Research},
  month = 	 {09--11 Sep},
  publisher =    {PMLR}
}


\clearpage
\appendix
\numberwithin{equation}{section}

\startcontents[appendix]

\begin{center}
    {\large\bfseries Appendix Contents}
\end{center}
\vspace{-0.5em}

\begingroup
\small
\setlength{\parskip}{0pt}
\printcontents[appendix]{}{1}{\setcounter{tocdepth}{2}}
\endgroup

\vspace{1em}
\section{Theoretical Proofs of Main Results}\label{app:theory}

\subsection{Standing Notation}\label{app:notation}
We briefly define the shorthand notations used throughout the appendix to distinguish between \emph{hard} and \emph{smoothed} objects. The conditional density and CDF of \(Y \mid X=x\) are denoted by \(f(\cdot\mid x)\) and \(F(\cdot\mid x)\). The true HDI endpoints are \(\ell^*(x)\) and \(u^*(x)\), with center \(m^*(x)\) and radius \(h^*(x)\).

For any candidate center \(m \in \mathbb{R}\), we define the hard radius map \(\phi_0(m;x)\) as the conditional \((1-\alpha)\)-quantile of \(|Y-m|\mid X=x\), with corresponding endpoints 
\begin{equation*}
    \ell_0(m;x) \coloneqq m - \phi_0(m;x)\qquad
    \text{and}\qquad
    u_0(m;x) \coloneqq m + \phi_0(m;x).
\end{equation*}
The smoothed radius map is \(\phi_\beta(m; x) \coloneqq \inf\{h > 0 : \Gamma(m, h; x) \ge 1-\alpha\}\). The smoothed HDI endpoints are \(\ell_\beta^\dagger(x) \coloneqq m_\beta^\dagger(x) - h_\beta^\dagger(x)\) and \(u_\beta^\dagger(x) \coloneqq m_\beta^\dagger(x) + h_\beta^\dagger(x)\). We emphasize the crucial distinction that \(h_\beta^\dagger(x) = \phi_\beta(m_\beta^\dagger(x); x) \neq \phi_0(m_\beta^\dagger(x);x)\) in general.

To analyze the density equality condition, we define the endpoint-imbalance functions:
\begin{align*}
G_0(m;x) &\coloneqq f(u_0(m;x)\mid x) - f(\ell_0(m;x)\mid x), \\
G_\beta(m;x) &\coloneqq f_\beta(m+\phi_\beta(m;x)\mid x) - f_\beta(m-\phi_\beta(m;x)\mid x).
\end{align*}
By Definition~\ref{def:smooth_hdi}, \(G_\beta(m_\beta^\dagger(x);x) = 0\), and under Assumption~\ref{ass:density}(D3), \(G_0(m^*(x);x) = 0\).

\subsection{Auxiliary Lemmas}\label{app:aux_lemmas}

\begin{lemma}[Lipschitz dependence on the center]\label{lem:rad_lip_app}
For any \(x\) and any \(m,m'\in\mathbb R\),
\(|\phi_0(m;x)-\phi_0(m';x)|\le|m-m'|\) and \(|\phi_\beta(m;x)-\phi_\beta(m';x)|\le|m-m'|\).
\end{lemma}
\begin{proof}
Let \(U=|Y-m|\) and \(V=|Y-m'|\) conditional on \(X=x\). By the reverse triangle inequality, \(|U-V| = \big| |Y-m| - |Y-m'| \big| \le |m-m'|\) almost surely. Hence \(U\le V+|m-m'|\) and \(V\le U+|m-m'|\) almost surely. Taking the \((1-\alpha)\)-quantile on both sides preserves the inequality, yielding \(Q_{1-\alpha}(U)\le Q_{1-\alpha}(V)+|m-m'|\) and vice versa, which proves the bound for \(\phi_0\).

For the smoothed radius \(\phi_\beta\), observe that the logistic convolution is equivalent to adding independent noise \(\varepsilon_\beta \sim K_\beta\). Let \(Y_\beta = Y + \varepsilon_\beta\). Then \(\Gamma(m, h; x) = \mathbb{P}(|Y_\beta - m| \le h \mid X=x)\), which means \(\phi_\beta(m; x) = Q_{1-\alpha}(|Y_\beta - m| \mid X=x)\). Applying the same reverse triangle inequality to \(Y_\beta\) yields \(|\phi_\beta(m;x)-\phi_\beta(m';x)|\le|m-m'|\).
\end{proof}

\begin{lemma}[Coverage is Lipschitz in the radius]\label{lem:cov_sens_app}
Under Assumption~\ref{ass:density}, for any \(x\), \(m\in\mathbb R\) and \(h_1,h_2\ge 0\),
\begin{equation*}
\bigl|\mathbb P(|Y-m|\le h_1\mid X=x)-\mathbb P(|Y-m|\le h_2\mid X=x)\bigr|\le 2f_{\max}|h_1-h_2|.
\end{equation*}
\end{lemma}
\begin{proof}
Observe that \(\mathbb P(|Y-m|\le h\mid X=x)=F(m+h\mid x)-F(m-h\mid x)\). Since the density is bounded by \(f_{\max}\), the CDF \(F(\cdot\mid x)\) is \(f_{\max}\)-Lipschitz. Thus, each endpoint contributes at most \(f_{\max}|h_1-h_2|\) to the difference.
\end{proof}

\begin{lemma}[Lower bound on hard and smoothed radii]\label{lem:rad_lower_bound_beta}
For all \(x\) and \(m\in\mathbb R\), 
\begin{equation*}
    \phi_0(m;x)\ge h_{\min}\coloneqq(1-\alpha)/(2f_{\max}).
\end{equation*}
The same lower bound holds for \(\phi_\beta(m;x)\). In particular, \(h^*(x), h_\beta^\dagger(x)\ge h_{\min}\).
\end{lemma}
\begin{proof}
Since \(f(\cdot\mid x)\le f_{\max}\), the maximum probability mass any interval of length \(2h\) can capture is \(2hf_{\max}\). To achieve \(1-\alpha\) coverage, we must have \(2hf_{\max} \ge 1-\alpha\), meaning any \(h<h_{\min}\) is infeasible. Convolution does not increase the \(L^\infty\) norm, so \(\|f_\beta(\cdot\mid x)\|_\infty\le f_{\max}\), and the same lower bound holds for \(\phi_\beta\).
\end{proof}

\begin{lemma}[Uniform \(L^\infty\) control of \(F_\beta-F\)]\label{lem:Linfty_F}
For every \(x\) and \(\beta>0\), \(\sup_y|F_\beta(y\mid x)-F(y\mid x)|\le\|f_\beta(\cdot\mid x)-f(\cdot\mid x)\|_{L^1}\), and the right-hand side tends to \(0\) as \(\beta\to 0^+\).
\end{lemma}
\begin{proof}
By definition, \(F_\beta(y\mid x)-F(y\mid x)=\int_{-\infty}^y(f_\beta-f)(z\mid x)\,dz\). Taking the absolute value yields \(|F_\beta(y\mid x)-F(y\mid x)|\le \int_{-\infty}^\infty |f_\beta(z\mid x)-f(z\mid x)|\,dz = \|f_\beta-f\|_{L^1}\) uniformly in \(y\). The vanishing of \(\|K_\beta*f-f\|_{L^1}\) for any integrable \(f\) is a standard property of mollifiers.
\end{proof}

\begin{lemma}[Local uniform convergence of \(f_\beta\) to \(f\)]\label{lem:conv_uniform_beta}
Fix \(x\) and a compact set \(\mathcal K\subset\mathbb R\). Suppose there exists \(\delta_0>0\) such that \(f(\cdot\mid x)\) is continuous on \(\mathcal K^{\delta_0} \coloneqq \{y\in\mathbb R : \inf_{z\in\mathcal K}|y-z|\le\delta_0\}\). Then \(\sup_{y\in\mathcal K}|f_\beta(y\mid x)-f(y\mid x)|\to 0\) as \(\beta\to 0^+\).
\end{lemma}
\begin{proof}
By continuity on the compact set \(\mathcal K^{\delta_0}\), \(f(\cdot\mid x)\) is uniformly continuous there. Fix \(\varepsilon>0\) and choose \(0<\delta\le\delta_0\) such that \(|f(y-u\mid x)-f(y\mid x)|\le\varepsilon\) for \(y\in\mathcal K\) and \(|u|\le\delta\). Then for \(y\in\mathcal K\),
\begin{align*}
|f_\beta(y\mid x)-f(y\mid x)| &\le \int_{\mathbb{R}} |f(y-u\mid x)-f(y\mid x)| K_\beta(u)\,du \\
&\le \int_{|u|\le\delta} \varepsilon K_\beta(u)\,du + \int_{|u|>\delta} 2f_{\max} K_\beta(u)\,du \\
&\le \varepsilon + 2f_{\max}\int_{|t|>\delta/\beta}\sigma'(t)\,dt.
\end{align*}
The second term vanishes as \(\beta\to 0^+\) due to the exponential tails of \(\sigma'\).
\end{proof}

\begin{lemma}\label{lem:l2_to_prob}
Let \(Z_n\) be real-valued and \(\mathcal G_n\) a sigma-field. If there exist deterministic sequences \(\eta_n\to0\) and \(\rho_n\to0\) such that \(\mathbb P\!\left(\mathbb E[Z_n^2\mid \mathcal G_n]\le \eta_n\right)\ge 1-\rho_n\), then \(Z_n\xrightarrow{\mathbb P}0\). In particular, for every \(\varepsilon>0\),
\begin{equation*}
\mathbb P(|Z_n|>\varepsilon)\le \frac{\eta_n}{\varepsilon^2}+\rho_n.
\end{equation*}
\end{lemma}
\begin{proof}
Let \(A_n=\{\mathbb E[Z_n^2\mid \mathcal G_n]\le \eta_n\}\). By the law of total probability and Markov's inequality conditional on \(\mathcal G_n\):
\begin{align*}
\mathbb P(|Z_n|>\varepsilon) &\le \mathbb P(|Z_n|>\varepsilon \mid A_n)\mathbb P(A_n) + \mathbb P(A_n^c) \\
&\le \mathbb E\!\left[\frac{\mathbb E[Z_n^2\mid \mathcal G_n]}{\varepsilon^2}\mathbf 1_{A_n}\right] + \rho_n \\
&\le \frac{\eta_n}{\varepsilon^2}+\rho_n.
\end{align*}
Since \(\eta_n \to 0\) and \(\rho_n \to 0\), it follows that \(Z_n \xrightarrow{\mathbb P} 0\).
\end{proof}

\subsection{Proof of Proposition~\ref{prop:unique_mode_bounded}}\label{proof_unique_mode}
\begin{proof}
Fix \(x\) and suppress it from the notation. By the properties \(K_\beta\in C^1(\mathbb R)\), \(K_\beta'\in L^\infty(\mathbb R)\), and \(f\in L^1(\mathbb R)\), we obtain \(f_\beta\in C^1(\mathbb R)\) and \(f_\beta'=K_\beta' * f\).

By the unimodalility, there exists an (\(x\)-dependent) interval \([a, b]\) with \(a \le b\) such that \(f\) is nondecreasing on \((-\infty, a]\), constant on \([a, b]\), and nonincreasing on \([b, \infty)\). Because \(f\) is bounded, integrable, and follows this monotonic structure, it is of bounded variation on \(\mathbb R\). Hence its distributional derivative \(Df\) is a finite signed measure. Write the Jordan decomposition as \(Df=\mu-\nu\), where \(\mu\) and \(\nu\) are positive measures. By the structural property of \(f\), the positive variation \(\mu\) is supported on \((-\infty, a]\) and the negative variation \(\nu\) is supported on \([b, \infty)\). Additionally, the vanishing tails \(\lim_{|y|\to\infty} f(y)=0\) imply that \(\mu(\mathbb R)=\nu(\mathbb R)>0\). 

In the sense of distributions, and hence pointwise by continuity, \(f_\beta'=K_\beta*Df\). Therefore, we can write \(f_\beta'(y)=A(y)-B(y)\), where
\begin{equation*}
A(y)=\int K_\beta(y-u)\,d\mu(u), \qquad B(y)=\int K_\beta(y-v)\,d\nu(v).
\end{equation*}
Given that \(K_\beta\) is strictly positive and both \(\mu\) and \(\nu\) are non-trivial finite positive measures, it follows that \(A(y)>0\) and \(B(y)>0\) for all \(y\). Recall that \(K_\beta(t)=\frac{1}{4\beta}\operatorname{sech}^2\left(\frac{t}{2\beta}\right)\), which implies
\begin{equation*}
(\log K_\beta)''(t) = -\frac{1}{2\beta^2}\operatorname{sech}^2\left(\frac{t}{2\beta}\right) < 0.
\end{equation*}
Thus \(K_\beta\) is strictly log-concave. Consequently, for every \(u<v\), the function \(y\mapsto K_\beta(y-u)/K_\beta(y-v)\) is strictly decreasing. Consider the ratio \(R(y)=A(y)/B(y)\). For any \(y_1<y_2\), we examine the cross-multiplication difference:
\begin{align*}
&A(y_2)B(y_1)-A(y_1)B(y_2)\\&= \iint \left[ K_\beta(y_2-u)K_\beta(y_1-v)
- K_\beta(y_1-u)K_\beta(y_2-v) \right] d\mu(u)d\nu(v).
\end{align*}
Since \(\operatorname{supp}(\mu) \subseteq (-\infty, a]\) and \(\operatorname{supp}(\nu) \subseteq [b, \infty)\) with \(a \le b\), we have \(u \le v\) for \(\mu \times \nu\)-almost all \((u,v)\). Furthermore, by the fundamental properties of the Jordan decomposition, the measures \(\mu\) and \(\nu\) are mutually singular (\(\mu \perp \nu\)), meaning they share no point masses. Thus, the strict inequality \(u < v\) holds \(\mu \times \nu\)-almost everywhere. By the strict log-concavity of \(K_\beta\), the term in the bracket is strictly negative almost everywhere with respect to \(\mu \times \nu\). This implies \(A(y_2)B(y_1)-A(y_1)B(y_2)<0\), meaning the ratio \(R(y)\) is strictly decreasing on \(\mathbb R\).

The condition \(f_\beta'(y)=0\) is equivalent to \(R(y)=1\), which guarantees that \(f_\beta'\) has at most one real zero. To establish existence, observe that \(K_\beta\in C_0(\mathbb R)\) and \(f\in L^1(\mathbb R)\) together yield \(\lim_{|y|\to\infty} f_\beta(y)=0\). Furthermore, \(f_\beta(y)>0\) everywhere because \(K_\beta>0\) and \(f\) is a probability density. Therefore, \(f_\beta\) attains a positive maximum at some finite point \(\nu_\beta\), and this maximum satisfies \(f_\beta'(\nu_\beta)=0\), proving that \(f_\beta'\) has exactly one real zero.

Finally, the strict monotonicity of \(R\) combined with \(R(\nu_\beta)=1\) ensures that \(R(y)>1\) for \(y<\nu_\beta\) and \(R(y)<1\) for \(y>\nu_\beta\). As \(f_\beta'(y)=B(y)[R(y)-1]\) with \(B(y)\) strictly positive, this yields
\begin{equation*}
f_\beta'(y)>0 \quad \text{for } y<\nu_\beta, \qquad f_\beta'(y)<0 \quad \text{for } y>\nu_\beta.
\end{equation*}
The proof is complete.
\end{proof}

\subsection{Proof of Proposition~\ref{prop:lyapunov}}\label{app:proof_lyapunov}
\begin{proof}
Fix \(x \in \mathcal{X}\) and suppress it from the notation where clear. By definition, the radius is instantaneously equilibrated to \(h_t = \phi_\beta(m_t; x)\). The center follows the negative population gradient flow of the smooth interval loss:
\begin{equation*}
\frac{dm_t}{dt} = -\frac{\partial \mathcal{L}_{\mathrm{M}}}{\partial m}(m_t; h_t, \beta) = f_\beta(m_t+h_t\mid x) - f_\beta(m_t-h_t\mid x).
\end{equation*}
By the chain rule, the time derivative of the required radius along this trajectory is:
\begin{equation*}
\frac{d}{dt}\phi_\beta(m_t; x) = \frac{\partial \phi_\beta(m_t; x)}{\partial m} \cdot \frac{dm_t}{dt}.
\end{equation*}
Recall from Equation~\eqref{eq:lyapunov_gradient} that the implicit derivative of the smoothed radius map with respect to the center is:
\begin{equation*}
\frac{\partial \phi_\beta(m_t; x)}{\partial m} = - \frac{f_\beta(m_t+h_t\mid x) - f_\beta(m_t-h_t\mid x)}{f_\beta(m_t+h_t\mid x) + f_\beta(m_t-h_t\mid x)}.
\end{equation*}
Substituting both expressions into the chain rule yields:
\begin{align*}
\frac{d}{dt}\phi_\beta(m_t; x) &= \left( - \frac{f_\beta(m_t+h_t\mid x) - f_\beta(m_t-h_t\mid x)}{f_\beta(m_t+h_t\mid x) + f_\beta(m_t-h_t\mid x)} \right) \cdot \bigl( f_\beta(m_t+h_t\mid x) - f_\beta(m_t-h_t\mid x) \bigr) \\
&= - \frac{\bigl[f_\beta(m_t+h_t\mid x) - f_\beta(m_t-h_t\mid x)\bigr]^2}{f_\beta(m_t+h_t\mid x) + f_\beta(m_t-h_t\mid x)}.
\end{align*}
Since \(f_\beta(\cdot \mid x)\) is a strictly positive smoothed density (as shown in Proposition~\ref{prop:unique_mode_bounded}), the denominator is strictly positive. The numerator is a squared real number, hence non-negative. Therefore, we conclude that:
\begin{equation*}
\frac{d}{dt}\phi_\beta(m_t; x) \le 0.
\end{equation*}
Equality holds if and only if the numerator is zero, which means \(f_\beta(m_t+h_t\mid x) = f_\beta(m_t-h_t\mid x)\). This is precisely the smoothed density-equality condition (the second defining property of the smoothed HDI). Thus, the continuous-time dynamic strictly and monotonically decreases the required radius until the interval is perfectly balanced.
\end{proof}

\subsection{Proof of Theorem~\ref{thm:cocp_marginal}}\label{app:proof_marginal}

\begin{proof}
By construction, the predictors \((\widehat m,\widehat h)\) are trained on \(\mathcal D_{\mathrm{train}}\) and are thus independent of \(\mathcal D_{\mathrm{cal}}\cup\{(X_{n_{\mathrm{cal}}+1},Y_{n_{\mathrm{cal}}+1})\}\). Define the nonconformity scores \(S_i\coloneqq|Y_i-\widehat m(X_i)|/\widehat h(X_i)\) for \(i=1,\dots,n_{\mathrm{cal}}+1\). Conditional on \((\widehat m,\widehat h)\), the exchangeability of the input pairs implies the exchangeability of the scores \(S_i\). 

The covering event for the test point is:
\begin{equation*}
\{Y_{n_{\mathrm{cal}}+1}\in\widehat C(X_{n_{\mathrm{cal}}+1})\} = \{S_{n_{\mathrm{cal}}+1}\le\widehat q\},
\end{equation*}
where \(\widehat q\) is the \(\lceil(n_{\mathrm{cal}}+1)(1-\alpha)\rceil\)-th order statistic of \(\{S_i\}_{i\le n_{\mathrm{cal}}}\) (and \(+\infty\) in the degenerate case). By the fundamental property of exchangeable random variables, the probability that \(S_{n_{\mathrm{cal}}+1}\) is among the smallest \(\lceil(n_{\mathrm{cal}}+1)(1-\alpha)\rceil\) values of the \(n_{\mathrm{cal}}+1\) scores is at least the proportional fraction, regardless of ties. Therefore,
\begin{equation*}
\mathbb P(S_{n_{\mathrm{cal}}+1}\le\widehat q) \ge \frac{\lceil(n_{\mathrm{cal}}+1)(1-\alpha)\rceil}{n_{\mathrm{cal}}+1} \ge 1-\alpha.
\end{equation*}
If all \(n_{\mathrm{cal}}+1\) scores are almost-surely distinct (i.e., there are no ties, which is guaranteed if the conditional score distribution is continuous), the rank of \(S_{n_{\mathrm{cal}}+1}\) is exactly uniformly distributed on \(\{1,\dots,n_{\mathrm{cal}}+1\}\), yielding the exact upper bound:
\begin{equation*}
\mathbb P(S_{n_{\mathrm{cal}}+1}\le\widehat q) \le \frac{\lceil(n_{\mathrm{cal}}+1)(1-\alpha)\rceil}{n_{\mathrm{cal}}+1} \le 1-\alpha+\frac{1}{n_{\mathrm{cal}}+1}.
\end{equation*}
\end{proof}

\subsection{Proof of Lemma~\ref{lem:smooth_hdi}}\label{app:proof_smooth_hdi}
\begin{proof}
Fix \(x\in\mathcal{X}\) and abbreviate \(f_\beta\coloneqq f_\beta(\cdot\mid x)\). By Proposition~\ref{prop:unique_mode_bounded}, \(f_\beta\in C^1(\mathbb R)\), \(f_\beta>0\), \(f_\beta\) is strictly unimodal with unique mode \(\nu_\beta\), and \(f_\beta(y)\to 0\) as \(|y|\to\infty\). 

\emph{Existence and Uniqueness.} The level sets \(L_c\coloneqq\{f_\beta\ge c\}\) are closed bounded intervals \([\ell(c),u(c)]\ni\nu_\beta\) for every \(c\in(0,\sup f_\beta)\), with \(\ell\) strictly increasing in \(c\), \(u\) strictly decreasing, and \(f_\beta(\ell(c))=f_\beta(u(c))=c\).
Define the mass function 
\begin{equation*}
    \Phi(c)\coloneqq\int_{\ell(c)}^{u(c)}f_\beta(y)\,dy.
\end{equation*}
By dominated convergence and the strict monotonicity of \(\ell\) and \(u\), \(\Phi\) is continuous and strictly decreasing on \((0,\sup f_\beta)\), with \(\Phi(0^+)=1\) and \(\Phi((\sup f_\beta)^-)=0\). Hence there exists a unique \(c^\dagger\in(0,\sup f_\beta)\) with \(\Phi(c^\dagger)=1-\alpha\). Set
\begin{equation*}
    m_\beta^\dagger\coloneqq \frac{\ell(c^\dagger)+u(c^\dagger)}{2}\quad\text{and}
    \quad
    h_\beta^\dagger\coloneqq\frac{u(c^\dagger)-\ell(c^\dagger)}{2}>0
\end{equation*}Then \(\Gamma(m_\beta^\dagger,h_\beta^\dagger;x)=\Phi(c^\dagger)=1-\alpha\) and \(f_\beta(m_\beta^\dagger\pm h_\beta^\dagger)=c^\dagger\), satisfying the definition of the smoothed HDI. If any other \((m,h)\) satisfies the conditions with \(h>0\), then \(f_\beta(m+h)=f_\beta(m-h)\), and strict unimodality forces \(\{m-h,m+h\}=\{\ell(c),u(c)\}\) for some \(c\in(0,\sup f_\beta)\). The coverage condition \(\Phi(c)=1-\alpha\) then forces \(c=c^\dagger\) by uniqueness, hence \((m,h)=(m_\beta^\dagger,h_\beta^\dagger)\).

\emph{Optimality.} Take any interval \([a,b]\) with length \(b-a \le u(c^\dagger)-\ell(c^\dagger) = 2h_\beta^\dagger\), assuming \([a,b] \neq L_{c^\dagger}\). We show \([a,b]\) cannot satisfy \(\int_a^b f_\beta\ge 1-\alpha\). Compute the difference in probability mass using \(L_{c^\dagger}=[\ell(c^\dagger),u(c^\dagger)]\):
\begin{align*}
\int_{L_{c^\dagger}}f_\beta - \int_{[a,b]}f_\beta &= \int_{L_{c^\dagger}\setminus[a,b]}f_\beta - \int_{[a,b]\setminus L_{c^\dagger}}f_\beta \\
&> c^\dagger\bigl(|L_{c^\dagger}\setminus[a,b]| - |[a,b]\setminus L_{c^\dagger}|\bigr) \\
&= c^\dagger\bigl(|L_{c^\dagger}| - |[a,b]|\bigr) \ge 0,
\end{align*}
where the strict inequality uses \(f_\beta > c^\dagger\) on the interior of \(L_{c^\dagger}\) and \(f_\beta < c^\dagger\) on \((L_{c^\dagger})^c\), and the fact that the symmetric difference is non-empty since \([a,b] \neq L_{c^\dagger}\). The equality uses the fact that \(|L_{c^\dagger}\cap[a,b]|\) is identical in both expressions. Since \(\int_{L_{c^\dagger}}f_\beta=1-\alpha\), it follows that \(\int_{[a,b]}f_\beta<1-\alpha\), contradicting feasibility. Hence \(L_{c^\dagger}\) is the unique shortest feasible interval.
\end{proof}

\subsection{Proof of Lemma~\ref{lem:beta_consistency}}\label{app:proof_beta_consistency}
\begin{proof}
Fix any \(x\) at which Assumption~\ref{ass:density} holds, and abbreviate \(m_\beta^\dagger\coloneqq m_\beta^\dagger(x)\), \(h_\beta^\dagger\coloneqq h_\beta^\dagger(x)\), \(\ell_\beta^\dagger\coloneqq m_\beta^\dagger-h_\beta^\dagger\), \(u_\beta^\dagger\coloneqq m_\beta^\dagger+h_\beta^\dagger\), and similarly \(m^*,h^*,\ell^*,u^*\) for the true HDI. By Lemma~\ref{lem:Linfty_F}, 
\begin{equation*}
    \|F_\beta(\cdot\mid x)-F(\cdot\mid x)\|_\infty\to 0, \quad
    \|f_\beta(\cdot\mid x)-f(\cdot\mid x)\|_{L^1}\to 0.
\end{equation*}
We argue by combining tightness, optimality, and HDI uniqueness.

By Lemma~\ref{lem:rad_lower_bound_beta}, \(h_\beta^\dagger\) is lower bounded by 
\begin{equation*}
    h_\beta^\dagger\ge h_{\min}=(1-\alpha)/(2f_{\max})>0.
\end{equation*}
Fix any \(\varepsilon>0\). By Assumption~\ref{ass:density}(D4), 
\begin{equation*}
    F(u^*+\varepsilon\mid x)-F(\ell^*-\varepsilon\mid x) > F(u^*\mid x)-F(\ell^*\mid x) = 1-\alpha.
\end{equation*}
Thus, by Lemma~\ref{lem:Linfty_F}, \(F_\beta(u^*+\varepsilon\mid x)-F_\beta(\ell^*-\varepsilon\mid x) > 1-\alpha\) for \(\beta\) sufficiently small. Hence \([\ell^*-\varepsilon,u^*+\varepsilon]\) is feasible for the smoothed HDI, and Lemma~\ref{lem:smooth_hdi} (optimality) gives \(2h_\beta^\dagger\le 2(h^*+\varepsilon)\) for \(\beta\) small. Letting \(\varepsilon\to 0\), we obtain:
\begin{equation}\label{eq:limsup_h_beta}
\limsup_{\beta\to 0^+}h_\beta^\dagger\le h^*.
\end{equation}
From \(F_\beta(u_\beta^\dagger\mid x)-F_\beta(\ell_\beta^\dagger\mid x)=1-\alpha\) and \(\|F_\beta-F\|_\infty\to 0\), we have
\begin{equation*}
    F(u_\beta^\dagger\mid x)-F(\ell_\beta^\dagger\mid x)=1-\alpha+o(1).
\end{equation*}
If \(\ell_\beta\to-\infty\) along some subsequence, then \(u_\beta^\dagger=\ell_\beta^\dagger+2h_\beta\le\ell_\beta^\dagger+2(h^*+1)\to-\infty\) for \(\beta\) small (using \eqref{eq:limsup_h_beta}), so \(F(u_\beta^\dagger\mid x)\to 0\) and \(F(u_\beta^\dagger\mid x)-F(\ell_\beta^\dagger\mid x)\to 0\), contradicting \(1-\alpha>0\). Symmetrically \(\ell_\beta^\dagger\to+\infty\) is excluded. Hence \(\ell_\beta^\dagger\), \(u_\beta^\dagger\), and \(m_\beta^\dagger=(\ell_\beta^\dagger+u_\beta^\dagger)/2\) are bounded.

Take any subsequence \(\beta_k\to 0^+\) along which \((m_{\beta_k}^\dagger,h_{\beta_k}^\dagger)\to(\bar m,\bar h)\) with \(\bar h\ge h_{\min}>0\). By \(\|F_\beta-F\|_\infty\to 0\) and continuity of \(F(\cdot\mid x)\),
\begin{equation*}
F(\bar m+\bar h\mid x)-F(\bar m-\bar h\mid x)=\lim_{k\to\infty}\left[F_{\beta_k}(m_{\beta_k}^\dagger+h_{\beta_k}^\dagger\mid x)-F_{\beta_k}(m_{\beta_k}^\dagger-h_{\beta_k}^\dagger\mid x)\right]=1-\alpha.
\end{equation*}
So \([\bar m-\bar h,\bar m+\bar h]\) is a valid \((1-\alpha)\)-coverage interval.

From \eqref{eq:limsup_h_beta}, \(\bar h\le h^*\). By Assumption~\ref{ass:density}, \([\ell^*,u^*]\) is the unique shortest \((1-\alpha)\)-coverage interval, of length \(2h^*\). Hence the \((1-\alpha)\)-coverage interval \([\bar m-\bar h,\bar m+\bar h]\) of length \(2\bar h\le 2h^*\) must satisfy \(2\bar h\ge 2h^*\) (HDI optimality) and equal \([\ell^*,u^*]\) (HDI uniqueness). Therefore \(\bar h=h^*\) and \(\bar m=m^*\).

By tightness, every subsequence of \(\{(m_\beta^\dagger,h_\beta^\dagger)\}_{\beta\to 0^+}\) has a further subsequence converging to some limit, which must be \((m^*,h^*)\). Hence the full sequence converges to \((m^*,h^*)\).
\end{proof}

\subsection{Calibration Consistency}\label{app:calibration}

\begin{proposition}[Calibration consistency]\label{prop:cal_quantile}
Under Assumptions~\ref{ass:density},~\ref{ass:consistency},~\ref{ass:cal} and \(n_{\mathrm{cal}}\to\infty\), \(\widehat q\xrightarrow{\mathbb{P}}1\) with rate
\(|\widehat q-1|=O_{\mathbb{P}}\left(\sqrt{\eta_{h,n}}+n_{\mathrm{cal}}^{-1/2}\right)\).
\end{proposition}

\begin{proof}

Condition on \(\mathcal D_{\mathrm{train}}\). Let \(S=|Y-\widehat m(X)|/\widehat h(X)\) with conditional CDF \(F_S\), and let \(q\) denote the population conditional \((1-\alpha)\)-quantile of \(S\). We bound 
\begin{equation*}
    |\widehat q-1|\le|\widehat q-q|+|q-1|.
\end{equation*}
For each \(x\in \mathcal{X}\), by definition of \(\phi_0\), 
\begin{equation*}
    \mathbb P(|Y-\widehat m(X)|\le\phi_0(\widehat m(X);X)\mid X=x,\mathcal D_{\mathrm{train}})=1-\alpha.
\end{equation*}
Hence by Lemma~\ref{lem:cov_sens_app},
\begin{align*}
&|F_S(1)-(1-\alpha)| \\
&\le\mathbb E_X\left[\bigl|\mathbb P\bigl(|Y-\widehat m(X)|\le\widehat h(X)\mid X\bigr)-\mathbb P\bigl(|Y-\widehat m(X)|\le\phi_0(\widehat m(X);X)\mid X\bigr)\bigr|\right] \\
&\le 2f_{\max}\,\mathbb E_X\left[\bigl|\widehat h(X)-\phi_0(\widehat m(X);X)\bigr|\right],
\end{align*}
where all expectations are conditional on \(\mathcal D_{\mathrm{train}}\). By Cauchy--Schwarz and Assumption~\ref{ass:consistency}, 
\begin{equation*}
    \mathbb E_X\left[\bigl|\widehat h(X)-\phi_0(\widehat m(X);X)\bigr|\right]\le\sqrt{\eta_{h,n}}
\end{equation*}
on the high-probability event of \eqref{eq:assumption_E_h}.

On the high-probability event \(B_n=\{f_S\ge s_0\text{ on }(1-\delta_0,1+\delta_0)\}\) from Assumption~\ref{ass:cal}, \(F_S\) has slope \(\ge s_0\) on this neighborhood. Thus, the quantile map of \(S\) is locally \(1/s_0\)-Lipschitz. When \(|F_S(1)-(1-\alpha)|\le s_0\delta_0/2\) (which happens w.p.\ tending to one), we have \(|q-1|\le|F_S(1)-(1-\alpha)|/s_0 \le \delta_0/2\), ensuring \(q \in (1-\delta_0/2, 1+\delta_0/2)\).

Conditional on \(\mathcal D_{\mathrm{train}}\), the calibration scores are i.i.d.\ with CDF \(F_S\). Note that \(\widehat q\) is the \(\lceil(n_{\mathrm{cal}}+1)(1-\alpha)\rceil\)-th order statistic. The target level is \(1-\alpha + O(1/n_{\mathrm{cal}})\). 
By the Dvoretzky--Kiefer--Wolfowitz (DKW) inequality, the empirical CDF satisfies 
\begin{equation*}
    \sup_t|F_{S,n_{\mathrm{cal}}}(t)-F_S(t)|=O_{\mathbb{P}}\left(n_{\mathrm{cal}}^{-1/2}\right).
\end{equation*}
On \(B_n\) and for \(n_{\mathrm{cal}}\) large, the \(O(1/n_{\mathrm{cal}})\) shift is absorbed, and \(\widehat q\) will also lie in \((1-\delta_0,1+\delta_0)\) with high probability. Within this neighborhood, the quantile map is \(1/s_0\)-Lipschitz, yielding the same rate for \(|\widehat q-q|\).

Combining these yields \(|\widehat q-1|=O_{\mathbb{P}}\left(\sqrt{\eta_{h,n}}+n_{\mathrm{cal}}^{-1/2}\right)\xrightarrow{\mathbb{P}}0\).
\end{proof}

\subsection{Proof of Theorem~\ref{thm:cocp_asymp}}\label{app:proof_asymp}
\begin{proof}
\textbf{(i) Calibration consistency} is established in Proposition~\ref{prop:cal_quantile}.

\textbf{(ii) Asymptotic conditional coverage.} Choose any deterministic \(\delta_n\to 0^+\) such that \(\eta_{h,n}/\delta_n^2\to 0\) and \(n_{\mathrm{cal}}^{-1/2}/\delta_n\to 0\); e.g.\ \(\delta_n=\eta_{h,n}^{1/4}+n_{\mathrm{cal}}^{-1/4}\). Define the random sets (depending on \(\mathcal D_{\mathrm{train}}\)):
\begin{equation*}
\Lambda_n\coloneqq\bigl\{x\in\mathcal X:\bigl|\widehat h(x)-\phi_0(\widehat m(x);x)\bigr|\le\delta_n\text{ and }\widehat h(x)\le h_+\bigr\}.
\end{equation*}
Let 
\begin{equation*}
    E_{h,n} \coloneqq \bigl\{\mathbb E[(\widehat h(X)-\phi_0(\widehat m(X);X))^2\mid \mathcal D_{\mathrm{train}}]\le \eta_{h,n}\bigr\}.
\end{equation*}
On \(E_{h,n}\), conditional Markov's inequality gives
\begin{equation*}
    \mathbb{P}\bigl(|\widehat h(X)-\phi_0(\widehat m(X);X)|>\delta_n \mid \mathcal D_{\mathrm{train}}\bigr)\le\eta_{h,n}/\delta_n^2.
\end{equation*}
Since \(\mathbb{P}(E_{h,n}^c)\le \rho_{h,n}\to 0\), and \(\mathbb{P}(\widehat h(X)\le h_+ \mid \mathcal D_{\mathrm{train}})\xrightarrow{\mathbb{P}} 1\) from Assumption~\ref{ass:cal}, we obtain \(\mathbb{P}(X\in\Lambda_n \mid \mathcal D_{\mathrm{train}})\xrightarrow{\mathbb{P}} 1\). Define the event \(A_n\coloneqq\bigl\{|\widehat q-1|\le\delta_n\bigr\}\). Proposition~\ref{prop:cal_quantile} gives \(\mathbb P(A_n)\to 1\). For any \(x\in\Lambda_n\) on the event \(A_n\), since
\begin{equation*}
    \mathbb P(|Y-\widehat m(X)|\le\phi_0(\widehat m(X);X)\mid X=x)=1-\alpha,
\end{equation*}
Lemma~\ref{lem:cov_sens_app} gives
\begin{align*}
&\quad\left|\mathbb P(Y\in\widehat C(X)\mid X=x, \mathcal D_{\mathrm{train}}, \mathcal D_{\mathrm{cal}})-(1-\alpha)\right|\\
&=\Bigl|\mathbb P(|Y-\widehat m(X)|\le\widehat q\,\widehat h(X)\mid X=x, \mathcal D_{\mathrm{train}}, \mathcal D_{\mathrm{cal}})\\
&\qquad\qquad-\mathbb P(|Y-\widehat m(X)|\le\phi_0(\widehat m(X);X)\mid X=x, \mathcal D_{\mathrm{train}})\Bigr|\\
&\le 2f_{\max}\bigl|\widehat q\,\widehat h(x)-\phi_0(\widehat m(x);x)\bigr|\\
&\le 2f_{\max}\left[\bigl|\widehat q-1\bigr|\,\widehat h(x)+\bigl|\widehat h(x)-\phi_0(\widehat m(x);x)\bigr|\right]\\
&\le 2f_{\max}(h_+\delta_n+\delta_n)=2f_{\max}(h_++1)\delta_n\xrightarrow{}0.
\end{align*}
Since this bound is uniform in \(x\in\Lambda_n\) on \(A_n\) and \(\mathbb P(A_n)\to 1\), \eqref{eq:asymp_cond_cov} follows.

\textbf{(iii) Asymptotic optimal length.} We use the decomposition:
\begin{equation}\label{eq:length_gap_decomp_main}
\begin{aligned}
&\quad\tfrac{1}{2}\bigl|\,\mathrm{Len}(\widehat C(X))-\mathrm{Len}(C^*(X))\,\bigr|\\
&=\bigl|\widehat q\,\widehat h(X)-\phi_0(m^*(X);X)\bigr|\\
&\le \underbrace{|\widehat q-1|\widehat h(X)}_{\text{(a) calibration error}}
+\underbrace{|\widehat h(X)-\phi_0(\widehat m(X);X)|}_{\text{(b) radius-fitting error}}
+\underbrace{|\widehat m(X)-m^*(X)|}_{\text{(c) center error}},
\end{aligned}
\end{equation}
where the bound in (c) uses the Lipschitz property from Lemma~\ref{lem:rad_lip_app}. We further decompose (c) into a smoothing-bias and an estimation-error part:
\begin{equation}\label{eq:center_split}
|\widehat m(X)-m^*(X)|\le|\widehat m(X)-m_{\beta_n}^\dagger(X)|+|m_{\beta_n}^\dagger(X)-m^*(X)|.
\end{equation}
All four resulting terms vanish in probability:
\begin{itemize}[leftmargin=2em,topsep=0pt,itemsep=2pt]
\item (a): \(|\widehat q-1|\xrightarrow{\mathbb{P}}0\) (Proposition~\ref{prop:cal_quantile}); \(\widehat h(X)\le h_+\) w.h.p.\ by Assumption~\ref{ass:cal}; product is \(o_{\mathbb{P}}(1)\).
\item (b): \(|\widehat h(X)-\phi_0(\widehat m(X);X)|\xrightarrow{\mathbb{P}}0\) by~\eqref{eq:assumption_E_h} and Lemma~\ref{lem:l2_to_prob}.
\item (c.1): \(|\widehat m(X)-m_{\beta_n}^\dagger(X)|\xrightarrow{\mathbb{P}}0\) by~\eqref{eq:assumption_E_m} and Lemma~\ref{lem:l2_to_prob}.
\item (c.2): \(|m_{\beta_n}^\dagger(X)-m^*(X)|\to 0\) \(P_X\)-a.s.\ by Lemma~\ref{lem:beta_consistency}.
\end{itemize}
Hence \(\mathrm{Len}(\widehat C(X))-\mathrm{Len}(C^*(X))\xrightarrow{\mathbb{P}}0\). Since \(\mathrm{Len}(C^*(X))=2h^*(X)\ge 2h_{\min}>0\) a.s.\ by Lemma~\ref{lem:rad_lower_bound_beta}, dividing yields \(\mathrm{Len}(\widehat C(X))/\mathrm{Len}(C^*(X))\xrightarrow{\mathbb{P}}1\).
\end{proof}

\section{Sharper Bias Rates and Boundary HDIs}\label{app:beta_rates_boundary}

We sharpen Lemma~\ref{lem:beta_consistency} in two regimes: (\S\ref{app:beta_bias_beta2}) a second-order \(O(\beta^2)\) bias when the true HDI is interior with \(C^2\) regularity; (\S\ref{app:beta_boundary_shift}) a first-order \(O(\beta)\) shift when the HDI is pinned at an active support boundary. All statements are pointwise in \(X=x\). Throughout, we abbreviate \(m_\beta^\dagger\coloneqq m_\beta^\dagger(x)\), \(h_\beta^\dagger\coloneqq h_\beta^\dagger(x)\), \(\ell_\beta^\dagger\coloneqq m_\beta^\dagger-h_\beta^\dagger\), \(u_\beta^\dagger\coloneqq m_\beta^\dagger+h_\beta^\dagger\). Set \(\sigma^2_K\coloneqq\int_\mathbb{R}t^2\sigma'(t)\,dt\) (which is finite, since \(\sigma'\) has exponential tails).

\subsection{Second-Order \texorpdfstring{$\beta$}{Beta}-Bias in the Interior HDI Regime}\label{app:beta_bias_beta2}

\begin{assumption}[Local \(C^2\) smoothness with non-flat endpoints]\label{as:local_smooth}
Fix \(x\in\mathcal{X}\). There exists \(\delta, M>0\) such that \(f(\cdot\mid x)\) is \(C^2\) on \(U_*(\delta)\coloneqq [\ell^*-\delta,\ell^*+\delta]\cup[u^*-\delta,u^*+\delta]\) with 
\begin{equation*}
    \sup_{y\in U_*(\delta)}|f''(y\mid x)|\le M<\infty.
\end{equation*}
Moreover, \(f'(\ell^*\mid x)>0\) and \(f'(u^*\mid x)<0\).
\end{assumption}

\begin{lemma}[Second-order local error of \(f_\beta\)]\label{lem:local_smooth_O_beta2}
Fix \(x\in\mathcal{X}\). Under Assumption~\ref{as:local_smooth}, there exists \(C<\infty\) such that for all sufficiently small \(\beta\),
\begin{equation*}
\sup_{|y-\ell^*|\le\delta/2}|f_\beta(y\mid x)-f(y\mid x)|+\sup_{|y-u^*|\le\delta/2}|f_\beta(y\mid x)-f(y\mid x)|\le C\beta^2.
\end{equation*}
\end{lemma}
\begin{proof}
Fix \(z\) with \(|z-\ell^*|\le\delta/2\) (the case near \(u^*\) is identical). Decompose
\begin{equation*}
    f_\beta(z\mid x)-f(z\mid x)=\int(f(z-u\mid x)-f(z\mid x))K_\beta(u)\,du.
\end{equation*}
On \(|u|\le\delta/2\), \(z-u\in U_*(\delta)\), so we can apply Taylor expansion:
\begin{equation*}
f(z-u\mid x)=f(z\mid x)-uf'(z\mid x)+\tfrac{u^2}{2}f''(\xi_{z,u}\mid x)
\end{equation*}
for some \(\xi_{z,u}\) between \(z\) and \(z-u\). By the symmetry of the kernel \(\int uK_\beta(u) du=0\),
\begin{align*}
\left|
\int_{|u|\le \delta/2}
\Bigl(f(z-u\mid x)-f(z\mid x)\Bigr)\,K_\beta(u)\,du
\right|
&\le
\frac{M}{2}\int u^2K_\beta(u)\,du \\
&=
\frac{M}{2}\beta^2\sigma^2_K \\
&=
O(\beta^2).
\end{align*}
For \(|u|>\delta/2\), use boundedness of \(f(y\mid x)\):
\begin{align*}
\left|
\int_{|u|>\delta/2}
\Bigl(f(z-u\mid x)-f(z\mid x)\Bigr)\,K_\beta(u)\,du
\right|
&\le
2f_{\max}\int_{|u|>\delta/2}K_\beta(u)\,du \\
&=
2f_{\max}\int_{|t|>\delta/(2\beta)}\sigma'(t)\,dt \\
&=
o(\beta^2),
\end{align*}
since \(\sigma'\) has exponentially decaying tails. Combining the two parts proves the first bound.
\end{proof}

\begin{lemma}[Second-order local error of \(F_\beta\)]\label{lem:local_F_O_beta2}
Fix \(x\in\mathcal{X}\). Under Assumption~\ref{as:local_smooth}, for all sufficiently small \(\beta\) and uniformly for \(|y-\ell^*|\le\delta/2\) or \(|y-u^*|\le\delta/2\),
\begin{equation*}
F_\beta(y\mid x)=F(y\mid x)+\tfrac{\beta^2\sigma^2_K}{2}f'(y\mid x)+O(\beta^3).
\end{equation*}
\end{lemma}
\begin{proof}
We apply the same Taylor expansion as in Lemma~\ref{lem:local_smooth_O_beta2}, but to \(F\) instead of \(f\), using \(F'=f\) and \(F''=f'\) on \(U_*(\delta)\). The cubic remainder is \(O(\beta^3)\) due to the exponential tails of \(\sigma'\).
\end{proof}

\begin{proposition}[Second-order interior bias]\label{prop:beta_bias_beta2}
Fix \(x\in\mathcal{X}\). Under Assumptions~\ref{ass:density} and~\ref{as:local_smooth}, \(|m_\beta^\dagger-m^*|=O(\beta^2)\) and \(|h_\beta^\dagger-h^*|=O(\beta^2)\) as \(\beta\to 0^+\).
\end{proposition}
\begin{proof}
By Lemma~\ref{lem:beta_consistency}, \(\ell_\beta^\dagger\to\ell^*\) and \(u_\beta^\dagger\to u^*\), so for \(\beta\) small both \(\ell_\beta^\dagger\) and \(u_\beta^\dagger\) lie in the \(C^2\) neighborhoods \(U_*(\delta/2)\) of Assumption~\ref{as:local_smooth}.

\textbf{Step 1.} The smoothed coverage equation reads \(F_\beta(u_\beta^\dagger\mid x)-F_\beta(\ell_\beta^\dagger\mid x)=1-\alpha\), and by definition of the hard center-to-radius map, \(F(m_\beta^\dagger+\phi_0(m_\beta^\dagger;x)\mid x)-F(m_\beta^\dagger-\phi_0(m_\beta^\dagger;x)\mid x)=1-\alpha\). Subtracting the two and applying Lemma~\ref{lem:local_F_O_beta2} to the smoothed terms yields:
\begin{align*}
&\bigl[F(u_\beta^\dagger\mid x)-F(m_\beta^\dagger+\phi_0(m_\beta^\dagger;x)\mid x)\bigr]-\bigl[F(\ell_\beta^\dagger\mid x)-F(m_\beta^\dagger-\phi_0(m_\beta^\dagger;x)\mid x)\bigr] \\
&\quad = -\tfrac{\beta^2\sigma^2_K}{2}\bigl[f'(u_\beta^\dagger\mid x)-f'(\ell_\beta^\dagger\mid x)\bigr]+O(\beta^3).
\end{align*}
Using \(u_\beta^\dagger-(m_\beta^\dagger+\phi_0(m_\beta^\dagger;x))=h_\beta^\dagger-\phi_0(m_\beta^\dagger;x)\) and similarly for the lower endpoint, Taylor expansion on the LHS gives:
\begin{align*}
\text{LHS} &= (h_\beta^\dagger-\phi_0(m_\beta^\dagger;x))\,\bigl[f(m_\beta^\dagger+\phi_0(m_\beta^\dagger;x)\mid x)+f(m_\beta^\dagger-\phi_0(m_\beta^\dagger;x)\mid x)\bigr]\\
&\qquad+O\left((h_\beta^\dagger-\phi_0(m_\beta^\dagger;x))^2\right).
\end{align*}
Define
\begin{equation*}
    A_\beta\coloneqq f(m_\beta^\dagger+\phi_0(m_\beta^\dagger;x)\mid x)+f(m_\beta^\dagger-\phi_0(m_\beta^\dagger;x)\mid x).
\end{equation*}
Since \(m_\beta^\dagger\to m^*\) and \(\phi_0\) is continuous,
\begin{equation*}
    A_\beta\to 2f(u^*\mid x)>0,
\end{equation*}
so \(A_\beta\ge c_0>0\) for some \(c_0\) and \(\beta\) small. The RHS is \(O(\beta^2)\). Solving this (and using \(h_\beta^\dagger-\phi_0(m_\beta^\dagger;x)\to 0\) to absorb the quadratic correction) yields \(|h_\beta^\dagger-\phi_0(m_\beta^\dagger;x)|=O(\beta^2)\).

\textbf{Step 2.} By definition,
\begin{equation*}
    G_0(m_\beta^\dagger;x) = f(m_\beta^\dagger+\phi_0(m_\beta^\dagger;x)\mid x)-f(m_\beta^\dagger-\phi_0(m_\beta^\dagger;x)\mid x).
\end{equation*}
The smoothed HDI density-equality gives \(f_\beta(u_\beta^\dagger\mid x)=f_\beta(\ell_\beta^\dagger\mid x)\). Using the local Lipschitzness of \(f\) and
\begin{equation*}
    |u_\beta^\dagger-(m_\beta^\dagger+\phi_0(m_\beta^\dagger;x))|=|h_\beta^\dagger-\phi_0(m_\beta^\dagger;x)|=O(\beta^2)
\end{equation*}
from Step 1, we decompose:
\begin{align*}
G_0(m_\beta^\dagger;x) &= \underbrace{\bigl[f(m_\beta^\dagger+\phi_0(m_\beta^\dagger;x)\mid x)-f(u_\beta^\dagger\mid x)\bigr]}_{O(\beta^2)\text{ (Lipschitz + Step 1)}}
+\underbrace{\bigl[f(u_\beta^\dagger\mid x)-f_\beta(u_\beta^\dagger\mid x)\bigr]}_{O(\beta^2)\text{ (Lemma~\ref{lem:local_smooth_O_beta2})}} \\
&\quad +\underbrace{\bigl[f_\beta(u_\beta^\dagger\mid x)-f_\beta(\ell_\beta^\dagger\mid x)\bigr]}_{=0}
+\underbrace{\bigl[f_\beta(\ell_\beta^\dagger\mid x)-f(\ell_\beta^\dagger\mid x)\bigr]}_{O(\beta^2)} \\
&\quad +\underbrace{\bigl[f(\ell_\beta^\dagger\mid x)-f(m_\beta^\dagger-\phi_0(m_\beta^\dagger;x)\mid x)\bigr]}_{O(\beta^2)} = O(\beta^2).
\end{align*}

\textbf{Step 3.} For \(m\) near \(m^*\), the implicit function theorem applied to 
\begin{equation*}
    H_0(m,h;x)\coloneqq F(m+h\mid x)-F(m-h\mid x)-(1-\alpha)
\end{equation*}
at \((m^*,h^*)\) gives that \(\phi_0\) is \(C^1\) near \(m^*\) with derivative:
\begin{equation*}
\phi_0'(m;x) = -\frac{G_0(m;x)}{f(u_0(m;x)\mid x)+f(\ell_0(m;x)\mid x)}.
\end{equation*}
At \(m=m^*\), \(G_0(m^*;x)=0\), so \(\phi_0'(m^*;x)=0\). Differentiating \(G_0\) via the chain rule (using \(u_0'(m)=1+\phi_0'(m)\), \(\ell_0'(m)=1-\phi_0'(m)\)) at \(m=m^*\) yields:
\begin{equation*}
G_0'(m^*;x) = f'(u^*\mid x)\cdot 1-f'(\ell^*\mid x)\cdot 1 = f'(u^*\mid x)-f'(\ell^*\mid x)<0
\end{equation*}
by Assumption~\ref{as:local_smooth}. Hence there exist \(c_G>0\) and \(\eta>0\) such that
\begin{equation*}
    |G_0(m;x)|\ge c_G|m-m^*|\quad
    \text{for}\quad
    |m-m^*|\le\eta.
\end{equation*}

Since \(m_\beta^\dagger\to m^*\), eventually \(|m_\beta^\dagger-m^*|\le\eta\). Combining Steps 2 and 3 gives
\begin{equation*}
    |m_\beta^\dagger-m^*|\le c_G^{-1}|G_0(m_\beta^\dagger;x)|=O(\beta^2).
\end{equation*}
Then by Step 1 and Lemma~\ref{lem:rad_lip_app}, 
\begin{equation*}
    |h_\beta^\dagger-h^*|\le|h_\beta^\dagger-\phi_0(m_\beta^\dagger;x)|+|\phi_0(m_\beta^\dagger;x)-\phi_0(m^*;x)|\le O(\beta^2)+|m_\beta^\dagger-m^*|=O(\beta^2).
\end{equation*}
\end{proof}

\subsection{Active Support Boundaries: First-Order \texorpdfstring{$\beta$}{Beta}-Shift}\label{app:beta_boundary_shift}

\begin{assumption}[Active lower-boundary truncation]\label{as:boundary_trunc}
Fix \(x\in\mathcal{X}\). \(f(\cdot\mid x)\) has support \([a,b]\), with \(b\in(a,\infty]\). On \((a,b)\), \(f(\cdot\mid x)\) is continuous, bounded by \(f_{\max}\), and unimodal; extended by zero outside. The right-limit \(f(a^+\mid x)\coloneqq\lim_{y\to a^+}f(y\mid x)\) exists and is finite-positive. The unique \((1-\alpha)\)-HDI is boundary-pinned: \(C^*(x)=[a,u^*(x)]\) with \(u^*(x)\in(a,b)\) and \(f(a^+\mid x)>f(u^*(x)\mid x)>0\).
\end{assumption}

\begin{proposition}[First-order boundary shift]\label{prop:beta_boundary}
Fix \(x\in\mathcal{X}\). Under Assumption~\ref{as:boundary_trunc}, define
\begin{equation*}
    \kappa\coloneqq
    \log \left(\frac{f(u^*(x)\mid x)}{f(a^+\mid x)-f(u^*(x)\mid x)}\right) 
    \in\mathbb R.
\end{equation*}
Then as \(\beta\to 0^+\),
\begin{equation*}
\ell_\beta^\dagger=a+\kappa\beta+o(\beta),\qquad u_\beta^\dagger=u^*(x)+O(\beta),\qquad \mathrm{Len}(C_\beta^\dagger(x))-\mathrm{Len}(C^*(x))=O(\beta),
\end{equation*}
and consequently \(|m_\beta^\dagger-m^*|=O(\beta)\), \(|h_\beta^\dagger-h^*|=O(\beta)\).
\end{proposition}
\begin{proof}
Drop \(x\) from the notation. Mass-binding \(\Gamma(m_\beta^\dagger,h_\beta^\dagger)=1-\alpha\) together with \(F_\beta-F=O(\|f_\beta-f\|_{L^1})\) (Lemma~\ref{lem:Linfty_F}) gives
\begin{equation}\label{eq:bd_mass_bind_app}
F(u_\beta^\dagger)-F(\ell_\beta^\dagger)=1-\alpha+O(\|f_\beta-f\|_{L^1}),
\end{equation}
and \(\|f_\beta-f\|_{L^1}=O(\beta)\) for any bounded unimodal \(f\), since \(f\) has bounded total variation \(\le 2f_{\max}\),
\begin{equation*}
    \|f_\beta-f\|_{L^1}\le 2f_{\max}\int|u|K_\beta(u)\,du=O(\beta).
\end{equation*}
The smoothed endpoint balance is
\begin{equation}\label{eq:bd_imbalance_app}
f_\beta(u_\beta^\dagger\mid x)=f_\beta(\ell_\beta^\dagger\mid x).
\end{equation}

We first prove that \(\ell_\beta^\dagger\to a\) and \(u_\beta^\dagger\to u^*\). By Lemma~\ref{lem:rad_lower_bound_beta}, \(h_\beta^\dagger \ge h_{\min} > 0\). Using a similar tightness argument as in Lemma~\ref{lem:beta_consistency}, one can show \(\ell_\beta^\dagger\) and \(u_\beta^\dagger\) remain bounded. Suppose along a subsequence \(\ell_\beta^\dagger\to\ell_0\ge a+\delta\) for some \(\delta>0\). Then \(F(\ell_\beta^\dagger)\to F(\ell_0)>0\), and~\eqref{eq:bd_mass_bind_app} forces \(u_\beta^\dagger\to u_0\) with \(\int_{\ell_0}^{u_0}f=1-\alpha\). Lemma~\ref{lem:conv_uniform_beta} on a compact subset of \((a,b)\) combined with~\eqref{eq:bd_imbalance_app} gives \(f(\ell_0)=f(u_0)\), so \([\ell_0,u_0]\) is an interior \((1-\alpha)\)-HDI for \(f\). This contradicts uniqueness in Assumption~\ref{as:boundary_trunc}. The case \(\ell_\beta^\dagger\le a-\delta\) is also excluded: there \(f_\beta(\ell_\beta^\dagger)\to 0\) but \(f_\beta(u_\beta^\dagger)\to f(u^*)>0\), contradicting~\eqref{eq:bd_imbalance_app}. Hence \(\ell_\beta^\dagger\to a\), and~\eqref{eq:bd_mass_bind_app} together with strict monotonicity of \(F\) on \((a,b)\) yields \(u_\beta^\dagger\to u^*\).

For \(t\in\mathbb R\), substituting \(z=a+\beta w\) gives:
\begin{equation}\label{eq:boundary_rightlimit}
\begin{aligned}
f_\beta(a+\beta t\mid x)&=\int_0^{(b-a)/\beta}f(a+\beta w\mid x)\sigma'(t-w)\,dw\\&\xrightarrow{\beta\to 0^+}f(a^+\mid x)\!\int_0^\infty\sigma'(t-w)\,dw=f(a^+\mid x)\sigma(t),
\end{aligned}
\end{equation}
by dominated convergence.

Define \(t_\beta\coloneqq(\ell_\beta^\dagger-a)/\beta\). Then we have \(\ell_\beta^\dagger-a\to 0\). We argue \(\{t_\beta\}\) is bounded. 

If \(t_\beta\to-\infty\), we directly bound the integral:
\begin{equation*}
f_\beta(\ell_\beta^\dagger) = \int_0^{(b-a)/\beta} f(a+\beta w)\sigma'(t_\beta-w)\,dw \le f_{\max} \int_0^\infty \sigma'(t_\beta-w)\,dw = f_{\max}\sigma(t_\beta).
\end{equation*}
As \(t_\beta \to -\infty\), \(\sigma(t_\beta) \to 0\), so \(f_\beta(\ell_\beta^\dagger)\to 0\). But \(f_\beta(u_\beta^\dagger)\to f(u^*)>0\), contradicting~\eqref{eq:bd_imbalance_app}. 

If \(t_\beta\to+\infty\) (so \(\ell_\beta^\dagger-a\gg\beta\)), we apply a boundary-layer argument. We have
\begin{equation*}
    f_\beta(\ell_\beta^\dagger) = \int f(\ell_\beta^\dagger - \beta w)\sigma'(w)dw.
\end{equation*}
Since \(\ell_\beta^\dagger - a \gg \beta\), for any fixed \(w\), \(\ell_\beta^\dagger - \beta w > a\) for \(\beta\) small enough, meaning \(f(\ell_\beta^\dagger - \beta w) \to f(a^+)\). By dominated convergence, \(f_\beta(\ell_\beta^\dagger) \to f(a^+) \int \sigma'(w)dw = f(a^+) > f(u^*)\), which again contradicts~\eqref{eq:bd_imbalance_app}. Therefore \(\{t_\beta\}\) is bounded.

Along any convergent subsequence \(t_\beta\to t\), from \eqref{eq:bd_imbalance_app} and \eqref{eq:boundary_rightlimit} we have
\begin{equation*}
    f_\beta(\ell_\beta^\dagger)\to f(a^+)\sigma(t)=f(u^*).
\end{equation*}
Solving \(\sigma(t)=f(u^*)/f(a^+)\in(0,1)\) yields the unique \(t=\kappa\), where
\begin{equation*}
    \kappa
    = \sigma^{-1}\left(\frac{f(u^*)}{f(a^+)}\right)
    =\log \left(\frac{f(u^*)}{f(a^+)-f(u^*)}\right).
\end{equation*}
Uniqueness of subsequential limits gives \(t_\beta\to\kappa\), i.e., \(\ell_\beta^\dagger=a+\kappa\beta+o(\beta)\).

From~\eqref{eq:bd_mass_bind_app}, \(F(u_\beta^\dagger)-F(\ell_\beta^\dagger)=1-\alpha+O(\beta)\). Since \(F\) is \(f_{\max}\)-Lipschitz and \(F(y)=0\) for \(y \le a\), 
\begin{equation*}
    |F(\ell_\beta^\dagger)-F(a)| \le f_{\max}(\ell_\beta^\dagger-a)_+ = O(\beta).
\end{equation*}
With \(F(a)=0\) (so \(F(u^*)=1-\alpha\)) we get \(F(u_\beta^\dagger)=F(u^*)+O(\beta)\). Strict monotonicity with \(F'(u^*)=f(u^*)>0\) gives local invertibility, so \(u_\beta^\dagger=u^*+O(\beta)\). Therefore, 
\begin{align*}
    \mathrm{Len}(C_\beta^\dagger)-\mathrm{Len}(C^*)
    &=(u_\beta^\dagger-\ell_\beta^\dagger)-(u^*-a)\\
    &=(u_\beta^\dagger-u^*)-(\ell_\beta^\dagger-a)\\
    &=O(\beta)-(\kappa\beta+o(\beta))=O(\beta).
\end{align*}
Then \(m_\beta^\dagger=(\ell_\beta^\dagger+u_\beta^\dagger)/2=m^*+O(\beta)\) and \(h_\beta^\dagger=(u_\beta^\dagger-\ell_\beta^\dagger)/2=h^*+O(\beta)\).
\end{proof}

\begin{remark}[Compatibility with Theorem~\ref{thm:cocp_asymp}]
For \(P_X\)-a.e. \(x\), assume either the interior assumptions (Assumption~\ref{ass:density}) or the boundary assumptions (Assumption~\ref{as:boundary_trunc}) hold. Proposition~\ref{prop:beta_boundary} gives \(|m_{\beta_n}^\dagger(x)-m^*(x)|=O(\beta_n)\) and \(|h_{\beta_n}^\dagger(x)-h^*(x)|=O(\beta_n)\) at active support boundaries. Since \(\beta_n\to 0\), the smoothing-bias term (c.2) in~\eqref{eq:length_gap_decomp_main} still vanishes \(P_X\)-a.s.; the remaining three terms (a), (b), (c.1) are unchanged. Conditional coverage at level \(1-\alpha\) holds verbatim, since its proof in~\S\ref{app:proof_asymp}(ii) only uses the calibration error \(|\widehat q-1|\) and the radius-fitting error \(|\widehat h-\phi_0(\widehat m;X)|\), both agnostic to whether the HDI is interior or boundary-pinned. Hence Theorem~\ref{thm:cocp_asymp}(i)--(iii) hold in the boundary regime as well, with the rate of (iii) governed by \(\beta_n\) rather than \(\beta_n^2\).
\end{remark}

\section{Experimental Details}\label{app:exp-details}

\textbf{Computational Resources.} All experiments were conducted on a server equipped with 2x NVIDIA A100 80GB PCIe GPUs, running CUDA 12.6 with driver version 560.35.03. The peak GPU memory usage is approximately 600 MB for the base (sequential) implementation, and increases to approximately 3 GB when using the Fast variant (Appendix~\ref{app:cocp_fast}) with fold parallelism ($K=4$).

\subsection{Real-world Datasets Details}\label{app:exp-real-details}

Table~\ref{tab:dataset_summary} provides a summary of the seven real-world regression datasets used in Section~\ref{sec:exp-real}. These datasets span a variety of domains, sample sizes (\(n\)), and feature dimensions (\(d\)), offering a comprehensive testbed for evaluating the efficiency and conditional reliability of the methods. All datasets are publicly available under standard open-source licenses, as detailed in the table below.

\begin{table}[ht!]
\centering
\caption{Overview of the real-world datasets used in our main study. \(n\) and \(d\) represent the number of samples and features, respectively.}
\label{tab:dataset_summary}
\small
\begin{tabular}{l p{4.8cm} r r c l}
\toprule
\textbf{Dataset} & \textbf{Description} & \textbf{Samples (\(n\))} & \textbf{Dim (\(d\))} & \textbf{Ref.} & \textbf{License} \\
\midrule
\textit{Bike} & Hourly bike sharing demand & 10,886 & 18 & \cite{bike} & CC BY 4.0 \\
\textit{Bio} & Protein tertiary structure properties & 45,730 & 9 & \cite{bio} & CC BY 4.0 \\
\textit{Blog} & Blog feedback volume prediction & 52,397 & 280 & \cite{blog} & CC BY 4.0 \\
\textit{Facebook-1} & Facebook comment volume (Variant 1) & 40,948 & 53 & \cite{facebook} & CC BY 4.0 \\
\textit{Facebook-2} & Facebook comment volume (Variant 2) & 81,311 & 53 & \cite{facebook} & CC BY 4.0 \\
\textit{Homes} & King County house sales prices & 21,613 & 19 & \cite{homes} & CC0 \\
\textit{Supercon.} & Superconductivity critical temperature & 21,263 & 81 & \cite{superconductivity} & CC BY 4.0 \\
\bottomrule
\end{tabular}
\end{table}

\subsection{Preprocessing}\label{app:exp-preprocess}

\paragraph{Synthetic data.}
Covariates are generated as \(X \sim \mathrm{Unif}(-2,2)\) and are used without further normalization. Targets are generated directly from the specified conditional distributions (Section~\ref{subsec:exp_synthetic}).

\paragraph{Real data.}
On each run, we apply the following preprocessing steps fit on the training split only:
\begin{enumerate}[leftmargin=*,itemsep=1pt]
    \item \textbf{Feature standardization:} we standardize each feature to zero mean and unit variance using the training split statistics, and apply the same transform to validation/calibration/test.
    \item \textbf{Target stabilization:} for heavy-tailed targets, specifically in \textit{Blog}, \textit{Facebook-1}, and \textit{Facebook-2}, we apply dataset-specific standard transformations commonly used in prior work, i.e., \(\log(1+Y)\).
    \item \textbf{Response rescaling:} following the procedure in \cite{sesia2020comparison}, we further rescale \(Y\) by a constant \(c=\frac{1}{|\mathcal{D}_{\text{train}}|}\sum_{(X,Y)\in\mathcal{D}_{\text{train}}} |Y|\) computed on the training split, i.e., \(Y \leftarrow Y/c\). This improves numerical conditioning and does not affect coverage; it only rescales interval lengths by a constant shared across methods within each repetition.
\end{enumerate}

\subsection{Implementation Details}\label{app:implementation-details}

\paragraph{General training setup and backbone MLP.}
Unless otherwise specified, all learning-based methods are trained with a batch size of 512 for a maximum of 1000 epochs, employing an early stopping patience of 100 based on validation performance. All neural networks are optimized with Adam. They share the same backbone architecture: a two-hidden-layer MLP with 64 ReLU units per layer and a linear output layer. Method-specific heads enforce required constraints, e.g.:
\begin{itemize}[leftmargin=*,itemsep=1pt]
    \item Center of CoCP outputs \(m(x)\in\mathbb{R}\).
    \item Radius of CoCP outputs \(h(x)>0\) by applying a Softplus nonlinearity.
    \item CQR outputs \((\widehat q_{\mathrm{lo}}(x),\widehat q_{\mathrm{hi}}(x))\) and ensures monotonicity (\(\widehat q_{\mathrm{lo}}\le \widehat q_{\mathrm{hi}}\)).
\end{itemize}

\paragraph{CoCP.}
We use $K=4$ fold cross-fitting. For each fold, the radius $h$ is trained on the current fold while the center $m$ is trained on the remaining $K-1$ folds. The procedure involves an initial MSE warm-up for $m$, followed by $T=5$ alternating iterations between $h$ and $m$, concluding with a final adjustment of $h$. We set $\beta=0.005$ and employ exponential learning rate decay for $h$. For synthetic data, $m$ uses a constant learning rate of $10^{-3}$; $h$ decays from $5 \cdot 10^{-4}$ to $0$. For real data, $m$ decays from $2 \cdot 10^{-3}$ to $10^{-3}$; $h$ decays from $5 \cdot 10^{-4}$ to $0$. Note that the MSE warm-up is used only as a practical initialization: in the idealized
population dynamics, the alternating updates are well-defined for any finite center, whereas in finite-sample training it places the boundary kernels in data-dense regions and reduces noisy or nearly zero gradients when \(\beta\) is small.

\paragraph{Split.}
Standard split conformal prediction using absolute residual scores (learning rate $10^{-3}$).

\paragraph{CQR and CHR.}
CQR uses the shared backbone (learning rate $10^{-3}$). For CHR, we employ an alternative faster implementation via CIR \cite{guo2026fast}, using 98 interval components and the shared backbone (learning rate $10^{-3}$).

\paragraph{C-HDR.}
We fit a conditional normalizing flow (2 layers, width 64, learning rate $2 \cdot 10^{-3}$). At test time, we approximate HDRs via Monte Carlo sampling using 5000 samples per test covariate. To ensure comparability with interval-based methods, we compute the convex hull of these samples to enforce a single connected prediction interval.

\paragraph{CPCP.}
We use the split fractions 0.4 and 0.4, bandwidth \(\delta=0.02\), and learning rate $5 \cdot 10^{-4}$, enabling clipping (\texttt{use\_clip}=\texttt{true}, multiplier \(5.0\)) and mixture augmentation (\texttt{use\_mix}=\texttt{true}, \(\lambda_{\mathrm{mix}}=0.5\)).

\paragraph{RCP.}
For RCP, we train the conditional quantile function of the nonconformity score on a fraction \(0.5\) of the calibration set (learning rate \(10^{-3}\)), and we employ the same calibration step as CoCP via the normalized score.

\paragraph{CPL.}
CPL relies on a user-specified finite-dimensional function class through a feature map \(\Phi(x)\), used to define conditional moment constraints. We consider \(\Phi(x) = [1,\ x^\top]^\top\) in synthetic data and \(\Phi(x) = [1,\ \mathrm{emb}(x)]^\top\) in real data, where \(\mathrm{emb}(x)\) is the embedding of \(x\) from the shared MLP model. We run 100 outer iterations. In each outer iteration, we train the threshold network for 200 steps on real data (resp.\ 500 steps on 1D synthetic data), with learning rate \(10^{-3}\), and we update the nonnegative multipliers with learning rate \(0.2\). We use smoothing parameter 0.1, length penalty weight 1.0, and split ratio 0.7 for the internal CPL split.

\paragraph{Naive CC.}
The RKHS-based conditional conformal method of \cite{gibbs2025conformal} can be extremely computationally expensive in our setting. We therefore include it only on synthetic data, and use a simplified variant that omits test imputation. Under i.i.d.\ sampling and large sample sizes, this ``naive'' variant is empirically close to the full method while being significantly more tractable. We use an RBF kernel
\begin{equation*}
k_\gamma(x,x') = \exp(-\gamma\|x-x'\|_2^2),
\end{equation*}
with \((\gamma=0.5, \lambda=0.00025)\) and a finite-dimensional feature map \(\Phi(x)\) as in CPL. In our implementation, the nonnegative interval radius is parameterized as
\begin{equation*}
t(x) = \max\{0,\ K(x,X_{\mathrm{cal}})w + \Phi(x)^\top \nu\},
\end{equation*}
where \(K(x,X_{\mathrm{cal}})\) is the vector of kernel evaluations against the calibration covariates.

\subsection{Metrics and Estimation}\label{app:metrics}
Let \(\mathcal D_{\mathrm{test}}=\{(x_i,y_i)\}_{i=1}^{n_{\mathrm{test}}}\) be the test set, let
\(\widehat C(x)=[\widehat \ell(x),\widehat u(x)]\) be the predicted interval, and define
\(z_i=\mathbf{1}\{y_i\in \widehat C(x_i)\}\). We evaluate the performance using the following metrics:

\paragraph{Coverage.}
For synthetic datasets, the conditional distribution is known, so we compute the exact conditional coverage
\(c_i=\mathbb P(Y\in \widehat C(x_i)\mid X=x_i)\). For continuous outcomes,
\(c_i=F(\widehat u(x_i)\mid x_i)-F(\widehat \ell(x_i)\mid x_i)\), where \(F(\cdot\mid x_i)\) is the conditional CDF of \(Y\mid X=x_i\).
We report
\begin{equation*}
\mathrm{Coverage}=
\begin{cases}
\displaystyle \frac{1}{n_{\mathrm{test}}}\sum_{i=1}^{n_{\mathrm{test}}} c_i,
& \text{for synthetic datasets,}\\[8pt]
\displaystyle \frac{1}{n_{\mathrm{test}}}\sum_{i=1}^{n_{\mathrm{test}}} z_i,
& \text{for real datasets.}
\end{cases}
\end{equation*}

\paragraph{Length.}
We report the average interval width:
\begin{equation*}
\mathrm{Length}
=
\frac{1}{n_{\mathrm{test}}}
\sum_{i=1}^{n_{\mathrm{test}}}
\bigl(\widehat u(x_i)-\widehat \ell(x_i)\bigr).
\end{equation*}

\paragraph{ConMAE.}
For synthetic datasets, we measure the mean absolute deviation of conditional coverage from the target level \(1-\alpha\):
\begin{equation*}
\mathrm{ConMAE}
=
\frac{1}{n_{\mathrm{test}}}
\sum_{i=1}^{n_{\mathrm{test}}}
\bigl|c_i-(1-\alpha)\bigr|.
\end{equation*}

\paragraph{MSCE.}
Following \citet{braun2026conditional}, we estimate the MSCE using a partition-based estimator. We run \(K\)-means on the standardized test covariates to obtain \(G\) clusters, with \(G=10\) throughout. Let \(\mathcal I_g\) be the index set of cluster \(g\), and define \(\widehat c_g
=
\frac{1}{|\mathcal I_g|}
\sum_{i\in \mathcal I_g} z_i\).
We then compute
\begin{equation*}
\mathrm{MSCE}
=
\sum_{g=1}^{G}
\frac{|\mathcal I_g|}{n_{\mathrm{test}}}
\bigl(\widehat c_g-(1-\alpha)\bigr)^2.
\end{equation*}

\paragraph{WSC.}
We evaluate WSC following \citet{cauchois2021knowing}, using the implementation in \texttt{covmetrics} \citep{braun2026conditional}. For \(v\in\mathbb R^d\) and \(a<b\), define the slab
\begin{equation*}
S_{v,a,b}=\{x\in\mathbb R^d: a\le v^\top x\le b\}.
\end{equation*}
Let \(\hat{\mathbb{P}}_n(S) = \frac{1}{n_{\mathrm{test}}}\sum_{i=1}^{n_{\mathrm{test}}}\mathbf{1}\{x_i\in S\}\) denote the empirical mass of a slab \(S\). Given a mass threshold \(\delta\in(0,1]\) (we use \(\delta=0.1\)), WSC is the minimum empirical coverage over all scanned slabs with empirical mass at least \(\delta\):
\begin{equation*}
\mathrm{WSC}
=
\min_{v\in\mathcal V}\;
\min_{\substack{a<b:\\ \hat{\mathbb{P}}_n(S_{v,a,b})\ge \delta}}
\frac{
\sum_{i=1}^{n_{\mathrm{test}}} z_i\,\mathbf{1}\{x_i\in S_{v,a,b}\}
}{
\sum_{i=1}^{n_{\mathrm{test}}}\mathbf{1}\{x_i\in S_{v,a,b}\}
}.
\end{equation*}

\paragraph{ERT.}
We compute ERT using \texttt{covmetrics} \citep{braun2026conditional}. ERT treats \(\{(x_i,z_i)\}_{i=1}^{n_{\mathrm{test}}}\) as a probabilistic classification problem and compares the risk of the constant predictor \(1-\alpha\) with that of an auditor \(a:\mathcal X\to[0,1]\) under a proper scoring rule \(\ell\). To avoid evaluating the auditor on the same data used for training, ERT is estimated using \(k\)-fold (\(k=5\)) cross-validation:
\begin{equation*}
\ell\text{-ERT}
=
\frac{1}{n_{\mathrm{test}}}
\sum_{j=1}^{k}
\sum_{i\in \mathcal I_j}
\Bigl[
\ell(1-\alpha,z_i)-\ell\bigl(a^{(-j)}(x_i),z_i\bigr)
\Bigr],
\end{equation*}
where \(\mathcal I_1,\dots,\mathcal I_k\) is a partition of the test set into \(k\) folds, and \(a^{(-j)}\) is trained on all samples outside fold \(\mathcal I_j\). Larger ERT indicates larger deviations from conditional coverage.
We report both \(\ell_1\)-ERT and \(\ell_2\)-ERT, where \(\ell_1\) is the score in \citet{braun2026conditional} and \(\ell_2\) is the Brier score,
\begin{equation*}
\ell_1(p,z)
=
\operatorname{sgn}\bigl(p-(1-\alpha)\bigr)\bigl((1-\alpha)-z\bigr),
\qquad
\ell_2(p,z)=(p-z)^2,
\end{equation*}
using Logistic Regression on standardized covariates as the auditor; hence ERT is a diagnostic over this linear auditor class rather than a certificate against all nonlinear forms of conditional miscoverage. Due to finite-sample variability, ERT can occasionally be slightly negative; such values should be interpreted as effectively zero.


\section{Additional Results}\label{app:additional-res}

\subsection{Runtime Comparison}\label{app:runtime}

In this section, we compare the computational efficiency of CoCP against the baselines. Table~\ref{tab:training_time} and Table~\ref{tab:inference_time} report the average training and inference times (in seconds) across all datasets. Note that the Naive CC method is evaluated only on synthetic benchmarks due to the difficulty of identifying robust hyperparameters for larger real-world datasets; these entries are therefore marked with ``-''.

As shown in Table~\ref{tab:training_time}, CoCP generally requires more training time than standard baselines like Split CP and CQR. For example, on the \textit{Facebook-2} dataset, CoCP takes approximately 1844 seconds to train, compared to 66 seconds for CQR. This increased offline computational cost is an expected trade-off: CoCP relies on an alternating optimization process and \(K\)-fold cross-fitting to achieve its geometric interval alignment and superior conditional reliability.

Crucially, this computational overhead is strictly confined to the \textbf{offline training phase}. In practical deployments, real-time efficiency is paramount. As demonstrated in Table~\ref{tab:inference_time}, CoCP's \textbf{inference time} is practically negligible (typically $< 0.05$ seconds). It matches the real-time efficiency of the fastest baselines (e.g., CQR and Split CP) and is orders of magnitude faster than generative methods like C-HDR, which require computationally heavy Monte Carlo sampling at test time.

Overall, the implementation is most appropriate when models are retrained infrequently and prediction is repeated. In settings with frequent retraining or limited compute, the additional offline cost may outweigh the gains in interval efficiency.

\begin{table}[htbp!]
\centering
\caption{Comparison of \textbf{Training Time} (in seconds) across synthetic and real-world datasets. Results are averaged over 10 random splits (mean $\pm$ std).}
\label{tab:training_time}
\resizebox{\textwidth}{!}{
\begin{tabular}{@{}lccccccccc@{}}
\toprule
\textbf{Dataset} & \textbf{Split} & \textbf{CPL} & \textbf{Naive CC} & \textbf{CQR} & \textbf{CPCP} & \textbf{RCP} & \textbf{C-HDR} & \textbf{CHR} & \textbf{CoCP} \\
\midrule
\multicolumn{10}{c}{\textit{Synthetic Datasets}} \\
\midrule
Normal & 18.8 \std{11.3} & 100.7 \std{12.3} & 146.8 \std{55.7} & 18.3 \std{5.3} & 12.1 \std{3.0} & 22.5 \std{11.1} & 23.9 \std{9.5} & 26.2 \std{8.0} & 397.1 \std{39.3} \\
LogNormal & 18.4 \std{10.2} & 95.4 \std{10.2} & 133.2 \std{20.5} & 17.3 \std{5.8} & 11.2 \std{3.8} & 22.0 \std{9.2} & 19.5 \std{5.4} & 24.7 \std{7.6} & 374.2 \std{46.3} \\
Exponential & 15.6 \std{7.3} & 93.2 \std{6.7} & 116.8 \std{23.8} & 17.0 \std{7.5} & 9.1 \std{2.9} & 19.0 \std{7.4} & 18.9 \std{6.2} & 23.0 \std{6.9} & 390.6 \std{35.5} \\
\midrule
\multicolumn{10}{c}{\textit{Real-World Datasets}} \\
\midrule
\textit{Bike} & 27.7 \std{3.1} & 54.8 \std{4.2} & - & 27.8 \std{6.9} & 4.7 \std{0.9} & 29.3 \std{3.2} & 9.1 \std{1.1} & 30.3 \std{5.7} & 259.4 \std{22.6} \\
\textit{Bio} & 134.0 \std{25.4} & 167.8 \std{26.0} & - & 105.0 \std{19.9} & 25.1 \std{8.2} & 145.1 \std{27.0} & 44.4 \std{9.5} & 186.9 \std{4.8} & 1358.4 \std{93.6} \\
\textit{Blog} & 30.3 \std{6.6} & 65.8 \std{9.2} & - & 39.6 \std{6.8} & 34.6 \std{15.1} & 37.3 \std{8.0} & 40.0 \std{6.2} & 37.2 \std{6.8} & 1310.8 \std{209.6} \\
\textit{Facebook-1} & 26.3 \std{4.7} & 59.7 \std{5.6} & - & 32.2 \std{4.6} & 16.5 \std{9.9} & 31.1 \std{4.9} & 28.6 \std{1.6} & 33.4 \std{4.4} & 819.3 \std{35.9} \\
\textit{Facebook-2} & 64.4 \std{7.9} & 97.7 \std{7.9} & - & 65.9 \std{8.2} & 24.9 \std{10.6} & 74.2 \std{8.2} & 57.6 \std{5.2} & 83.5 \std{20.5} & 1844.4 \std{132.1} \\
\textit{Homes} & 21.3 \std{3.2} & 55.8 \std{3.8} & - & 21.5 \std{2.5} & 7.9 \std{2.1} & 24.6 \std{2.9} & 17.4 \std{1.8} & 26.2 \std{5.1} & 418.1 \std{28.1} \\
\textit{Supercon.} & 48.2 \std{6.6} & 81.9 \std{6.6} & - & 37.6 \std{9.1} & 9.3 \std{3.1} & 52.0 \std{6.0} & 16.1 \std{1.4} & 70.9 \std{17.4} & 681.3 \std{40.0} \\
\bottomrule
\end{tabular}
}
\end{table}

\begin{table}[htbp!]
\centering
\caption{Comparison of \textbf{Inference Time} (in seconds) across synthetic and real-world datasets. Results are averaged over 10 random splits (mean $\pm$ std). CoCP maintains real-time efficiency comparable to the fastest baselines.}
\label{tab:inference_time}
\resizebox{\textwidth}{!}{
\begin{tabular}{@{}lccccccccc@{}}
\toprule
\textbf{Dataset} & \textbf{Split} & \textbf{CPL} & \textbf{Naive CC} & \textbf{CQR} & \textbf{CPCP} & \textbf{RCP} & \textbf{C-HDR} & \textbf{CHR} & \textbf{CoCP} \\
\midrule
\multicolumn{10}{c}{\textit{Synthetic Datasets}} \\
\midrule
Normal & 0.02 \std{0.01} & 0.01 \std{0.01} & 0.18 \std{0.02} & 0.00 \std{0.00} & 0.05 \std{0.08} & 0.03 \std{0.01} & 4.04 \std{3.28} & 0.06 \std{0.02} & 0.01 \std{0.01} \\
LogNormal & 0.01 \std{0.01} & 0.01 \std{0.01} & 0.40 \std{0.12} & 0.00 \std{0.00} & 0.03 \std{0.02} & 0.03 \std{0.01} & 3.74 \std{3.29} & 0.05 \std{0.01} & 0.01 \std{0.01} \\
Exponential & 0.03 \std{0.06} & 0.00 \std{0.00} & 0.34 \std{0.07} & 0.00 \std{0.00} & 0.03 \std{0.01} & 0.04 \std{0.05} & 2.63 \std{2.05} & 0.07 \std{0.03} & 0.01 \std{0.01} \\
\midrule
\multicolumn{10}{c}{\textit{Real-World Datasets}} \\
\midrule
\textit{Bike} & 0.01 \std{0.00} & 0.00 \std{0.00} & - & 0.00 \std{0.00} & 0.01 \std{0.00} & 0.03 \std{0.05} & 0.72 \std{0.01} & 0.03 \std{0.00} & 0.00 \std{0.00} \\
\textit{Bio} & 0.03 \std{0.00} & 0.01 \std{0.01} & - & 0.01 \std{0.01} & 0.07 \std{0.04} & 0.05 \std{0.01} & 5.20 \std{6.98} & 0.12 \std{0.02} & 0.00 \std{0.00} \\
\textit{Blog} & 0.05 \std{0.02} & 0.03 \std{0.02} & - & 0.01 \std{0.01} & 0.12 \std{0.03} & 0.11 \std{0.05} & 31.53 \std{24.33} & 0.25 \std{0.14} & 0.04 \std{0.06} \\
\textit{Facebook-1} & 0.03 \std{0.00} & 0.01 \std{0.01} & - & 0.00 \std{0.00} & 0.08 \std{0.05} & 0.07 \std{0.07} & 9.81 \std{7.06} & 0.11 \std{0.02} & 0.01 \std{0.01} \\
\textit{Facebook-2} & 0.06 \std{0.06} & 0.01 \std{0.01} & - & 0.01 \std{0.00} & 0.12 \std{0.04} & 0.13 \std{0.08} & 16.76 \std{15.42} & 0.21 \std{0.04} & 0.01 \std{0.01} \\
\textit{Homes} & 0.01 \std{0.00} & 0.01 \std{0.01} & - & 0.00 \std{0.00} & 0.04 \std{0.01} & 0.04 \std{0.01} & 4.19 \std{4.47} & 0.07 \std{0.01} & 0.02 \std{0.02} \\
\textit{Supercon.} & 0.02 \std{0.01} & 0.01 \std{0.01} & - & 0.00 \std{0.00} & 0.08 \std{0.09} & 0.03 \std{0.01} & 7.71 \std{3.40} & 0.06 \std{0.01} & 0.02 \std{0.02} \\
\bottomrule
\end{tabular}
}
\end{table}

\subsection{Accelerating Training: The Fast Variant}\label{app:cocp_fast}

As introduced in Section~\ref{sec:method}, the base algorithm's $K$-fold training introduce an $O(K)$ computational overhead. To mitigate this, we propose a \textbf{Fast variant} that significantly accelerates training without compromising performance via two key improvements:
\begin{enumerate}[leftmargin=*,itemsep=1pt]
    \item \textbf{Fold Parallelism:} Since training within each fold is mutually independent, we leverage multiprocessing to execute all \(K\)-fold simultaneously. This reduces wall-clock time at the cost of a \(K\) fold increase in peak GPU memory; sequential execution remains an option for memory-constrained environments.
    \item \textbf{Strict Early Stopping:} We introduce a minimum improvement threshold (\(\Delta_{\mathrm{min}}=10^{-6}\)). Validation loss must improve by at least this amount to reset the patience counter, preventing redundant computation on marginal updates in flat loss regions.
\end{enumerate}

\paragraph{Performance and Speedup.}
Table~\ref{tab:cocp_fast_comparison} presents a head-to-head comparison between the Base implementation and the Fast variant across all benchmarks. The results demonstrate a substantial reduction in training time. On synthetic datasets, the Fast variant achieves a \(3.1\times\) to \(3.6\times\) speedup. On real-world datasets, the acceleration is even more pronounced, typically ranging from \(3.2\times\) to \(5.1\times\). For instance, on the \textit{Blog} dataset, training time drops from \(1310.8\) seconds to just \(258.6\) seconds, and on the large \textit{Facebook-2} dataset, it decreases from \(1844.4\) seconds to \(462.7\) seconds.

The Fast variant yields similar interval lengths and conditional-reliability diagnostics (ConMAE, MSCE, WSC, and ERT) across the evaluated datasets, while inference remains below $0.05$ seconds for both implementations. These results confirm that CoCP scales efficiently to large-scale tasks via parallelization and robust early stopping, ensuring high practicality for real-world deployments, although it does not eliminate CoCP's higher offline cost relative to CQR.

\begin{table}[htbp!]
\centering
\caption{Comparison between the Base implementation and the Fast variant. Metrics are averaged over 10 random splits (mean \(\pm\) std). The Fast variant achieves significant training speedups (up to \(\approx 5\times\)) while strictly preserving interval efficiency and all conditional reliability metrics.}
\label{tab:cocp_fast_comparison}
\resizebox{\textwidth}{!}{
\begin{tabular}{@{} l l c | c c c c c c @{}}
\toprule
\textbf{Dataset} & \textbf{Variant} & \textbf{Train Time (s) (\(\downarrow\))} & \textbf{Coverage} & \textbf{Length (\(\downarrow\))} & \textbf{ConMAE / MSCE (\(\downarrow\))} & \textbf{WSC (\(\uparrow\))} & \textbf{\(\ell_1\)-ERT (\(\downarrow\))} & \textbf{\(\ell_2\)-ERT (\(\downarrow\))} \\
\midrule
\multicolumn{9}{c}{\textit{Synthetic Datasets (Conditional metric: ConMAE)}} \\
\midrule
\multirow{2}{*}{Normal} 
 & Base & 397.1 \std{39.3} & 0.9007 \std{0.0037} & 1.5977 \std{0.0153} & 0.0066 \std{0.0013} & - & - & - \\
 & Fast & \textbf{109.1} \std{10.4} & 0.9004 \std{0.0036} & 1.5964 \std{0.0144} & 0.0067 \std{0.0014} & - & - & - \\
\midrule
\multirow{2}{*}{LogNormal} 
 & Base & 374.2 \std{46.3} & 0.9002 \std{0.0029} & 0.9599 \std{0.0132} & 0.0066 \std{0.0019} & - & - & - \\
 & Fast & \textbf{112.4} \std{23.0} & 0.9002 \std{0.0026} & 0.9602 \std{0.0099} & 0.0058 \std{0.0015} & - & - & - \\
\midrule
\multirow{2}{*}{Exponential} 
 & Base & 390.6 \std{35.5} & 0.8975 \std{0.0026} & 1.1277 \std{0.0121} & 0.0077 \std{0.0022} & - & - & - \\
 & Fast & \textbf{124.7} \std{8.9} & 0.8967 \std{0.0022} & 1.1290 \std{0.0095} & 0.0077 \std{0.0014} & - & - & - \\
\midrule
\multicolumn{9}{c}{\textit{Real-World Datasets (Conditional metrics: MSCE, WSC, \(\ell_1\)-ERT, \(\ell_2\)-ERT)}} \\
\midrule
\multirow{2}{*}{\textit{Bike}} 
 & Base & 259.4 \std{22.6} & 0.8988 \std{0.0059} & 0.6053 \std{0.0190} & 0.0006 \std{0.0004} & 0.7982 \std{0.0096} & 0.0131 \std{0.0078} & -0.0001 \std{0.0004} \\
 & Fast & \textbf{80.0} \std{5.7} & 0.8991 \std{0.0078} & 0.6233 \std{0.0179} & 0.0008 \std{0.0004} & 0.7889 \std{0.0139} & 0.0165 \std{0.0125} & 0.0001 \std{0.0008} \\
\midrule
\multirow{2}{*}{\textit{Bio}} 
 & Base & 1358.4 \std{93.6} & 0.9013 \std{0.0056} & 1.4231 \std{0.0142} & 0.0002 \std{0.0001} & 0.8478 \std{0.0091} & 0.0059 \std{0.0038} & -0.0000 \std{0.0001} \\
 & Fast & \textbf{401.7} \std{37.0} & 0.9011 \std{0.0042} & 1.4152 \std{0.0162} & 0.0002 \std{0.0001} & 0.8485 \std{0.0084} & 0.0030 \std{0.0054} & -0.0000 \std{0.0001} \\
\midrule
\multirow{2}{*}{\textit{Blog}} 
 & Base & 1310.8 \std{209.6} & 0.8994 \std{0.0046} & 2.7025 \std{0.0897} & 0.0004 \std{0.0001} & 0.8416 \std{0.0069} & 0.0177 \std{0.0036} & -0.0016 \std{0.0007} \\
 & Fast & \textbf{258.6} \std{16.9} & 0.8990 \std{0.0035} & 2.7180 \std{0.1107} & 0.0004 \std{0.0001} & 0.8393 \std{0.0077} & 0.0200 \std{0.0033} & -0.0015 \std{0.0006} \\
\midrule
\multirow{2}{*}{\textit{Facebook-1}} 
 & Base & 819.3 \std{35.9} & 0.8998 \std{0.0047} & 1.6380 \std{0.0232} & 0.0002 \std{0.0001} & 0.8470 \std{0.0074} & 0.0056 \std{0.0051} & -0.0004 \std{0.0003} \\
 & Fast & \textbf{218.4} \std{22.5} & 0.8989 \std{0.0047} & 1.6338 \std{0.0411} & 0.0002 \std{0.0000} & 0.8469 \std{0.0068} & 0.0072 \std{0.0049} & -0.0004 \std{0.0004} \\
\midrule
\multirow{2}{*}{\textit{Facebook-2}} 
 & Base & 1844.4 \std{132.1} & 0.8988 \std{0.0024} & 1.5376 \std{0.0141} & 0.0001 \std{0.0001} & 0.8547 \std{0.0072} & 0.0103 \std{0.0023} & -0.0001 \std{0.0001} \\
 & Fast & \textbf{462.7} \std{45.0} & 0.8981 \std{0.0027} & 1.5308 \std{0.0181} & 0.0001 \std{0.0001} & 0.8547 \std{0.0068} & 0.0109 \std{0.0031} & -0.0001 \std{0.0002} \\
\midrule
\multirow{2}{*}{\textit{Homes}} 
 & Base & 418.1 \std{28.1} & 0.9017 \std{0.0065} & 0.5242 \std{0.0087} & 0.0003 \std{0.0002} & 0.8283 \std{0.0145} & 0.0052 \std{0.0070} & -0.0003 \std{0.0003} \\
 & Fast & \textbf{111.8} \std{9.8} & 0.9003 \std{0.0064} & 0.5222 \std{0.0090} & 0.0003 \std{0.0001} & 0.8272 \std{0.0097} & 0.0039 \std{0.0077} & -0.0004 \std{0.0003} \\
\midrule
\multirow{2}{*}{\textit{Supercon.}} 
 & Base & 681.3 \std{40.0} & 0.9041 \std{0.0044} & 0.8167 \std{0.0156} & 0.0003 \std{0.0002} & 0.8245 \std{0.0069} & 0.0127 \std{0.0091} & -0.0006 \std{0.0006} \\
 & Fast & \textbf{167.6} \std{14.9} & 0.9034 \std{0.0054} & 0.8134 \std{0.0135} & 0.0004 \std{0.0001} & 0.8247 \std{0.0098} & 0.0114 \std{0.0070} & -0.0009 \std{0.0005} \\
\bottomrule
\end{tabular}
}
\end{table}

\subsection{Sensitivity Analysis and Ablation}\label{app:exp-sensitivity}

We evaluate CoCP's robustness and core components using the Fast variant. We analyze three hyperparameters: smoothing temperature \(\beta\), number of folds \(K\), and alternating updates \(T\), varying one while fixing the others at the default \((\beta, K, T) = (0.005, 4, 5)\).

We report Length and ConMAE for synthetic benchmarks. For clarity, real-world evaluation focuses on the four largest datasets (\textit{Bio}, \textit{Blog}, \textit{Facebook-1}, \textit{Facebook-2}), reporting Length, MSCE, WSC, and \(\ell_1\)-ERT (\(\ell_2\)-ERT is omitted due to identical trends). In all figures, color-coded dashed lines denote the CQR baseline, providing a fixed reference for standard center-invariant intervals.

\textbf{Temperature parameter \(\beta\).} We evaluate \(\beta \in \{0.002, 0.005, 0.01, 0.02, 0.05\}\). Figure~\ref{fig:beta_combined_grid} demonstrates CoCP's resilience to the smoothing temperature. Performance is highly stable across \(\beta \in [0.002, 0.02]\), consistently outperforming CQR. However, excessive smoothing (\(\beta = 0.05\)) distorts the geometric boundary signal, causing intervals to widen (visible as a U-shape in Length on real datasets) and conditional reliability to degrade (e.g., inflated ConMAE under Exponential noise). The default \(\beta = 0.005\) strikes an optimal balance, ensuring sufficient gradient flow without compromising geometric precision.

\textbf{Number of folds \(K\).} Figure~\ref{fig:K_combined_grid} evaluates \(K \in \{2, 3, 4, 5\}\), which dictates the \((K-1):1\) data allocation for center versus radius updates. While synthetic benchmarks show mild sensitivity, real-world datasets exhibit clear gains as \(K\) increases to 4: interval lengths drop noticeably (e.g., on \textit{Blog} and \textit{Facebook-1}), and conditional metrics strictly dominate CQR. Performance plateaus at \(K=5\), confirming \(K=4\) as a robust default that balances data utilization with ensemble variance reduction.

\textbf{Alternating iterations \(T\) (Ablation).} Figure~\ref{fig:T_combined_grid} evaluates \(T \in \{0, 1, 2, 3, 4, 5\}\), serving as a crucial ablation of the center-refinement step. At \(T = 0\) (radius learning only), CoCP yields interval lengths comparable to or even worse than CQR. Under symmetric Normal noise, increasing \(T\) has negligible impact since the initial center that estimates the conditional mean is already near-optimal. However, under skewed noise and across all four real-world datasets, enabling alternating updates (\(T > 0\)) triggers a sharp reduction in interval length. Although conditional reliability metrics may exhibit initial fluctuations, they stabilize within 3 to 5 iterations—maintaining or improving upon the \(T = 0\) performance while consistently outperforming the CQR reference. This definitively validates CoCP's core premise: co-optimizing translation and scaling via the push-pull mechanism is fundamentally superior to the fixed-center paradigm.

\begin{figure}[htbp]
    \centering
    \begin{minipage}{\textwidth}
        \centering
        \includegraphics[width=\textwidth]{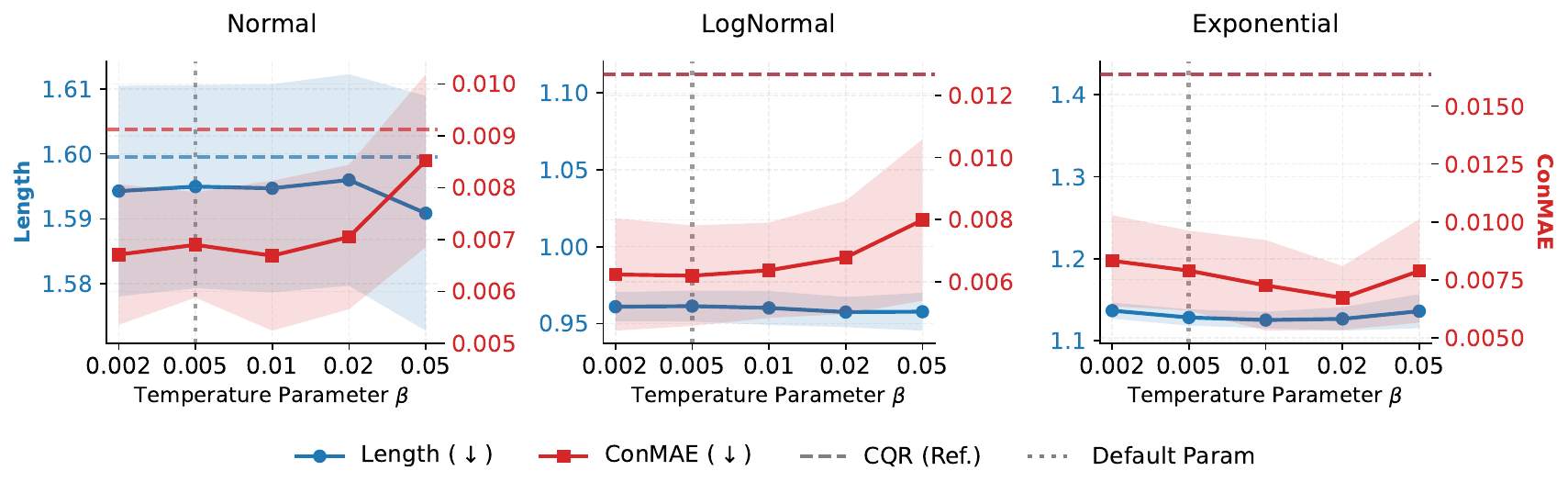}
    \end{minipage}
    \begin{minipage}{\textwidth}
        \centering
        \includegraphics[width=\textwidth]{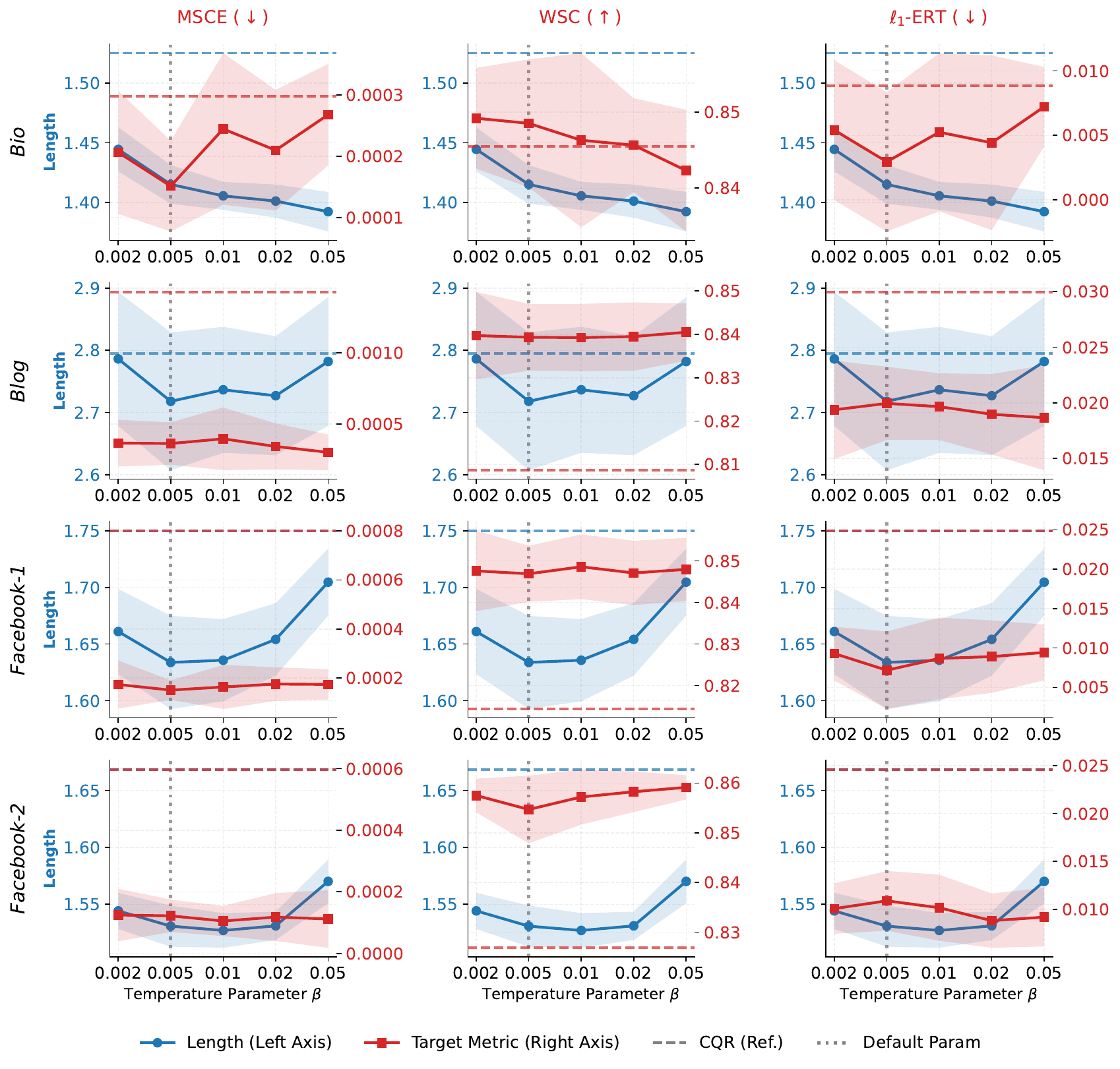}
    \end{minipage}
    \caption{Sensitivity to the smoothing temperature \(\beta \in \{0.002, 0.005, 0.01, 0.02, 0.05\}\) using the Fast variant. \textbf{Top:} Synthetic benchmarks (Length and ConMAE). \textbf{Bottom:} The four largest real datasets (Length, MSCE, WSC, and \(\ell_1\)-ERT). The solid lines show CoCP performance, dashed lines denote the CQR reference (color-coded by axis), and the gray dotted line marks the default \(\beta = 0.005\). CoCP remains stable and outperforms CQR across \(\beta \le 0.02\), with geometric distortion and performance degradation occurring only under excessive smoothing (\(\beta = 0.05\)).}

    \label{fig:beta_combined_grid}
\end{figure}

\begin{figure}[htbp]
    \centering
    \begin{minipage}{\textwidth}
        \centering
        \includegraphics[width=\textwidth]{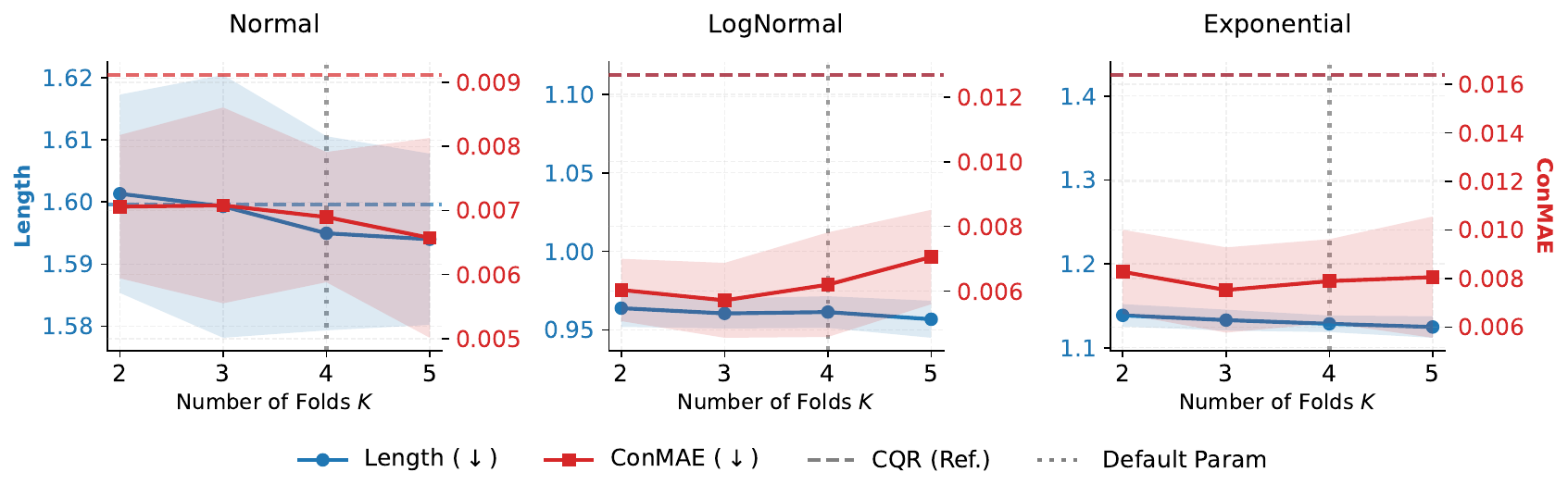}
    \end{minipage}
    \begin{minipage}{\textwidth}
        \centering
        \includegraphics[width=\textwidth]{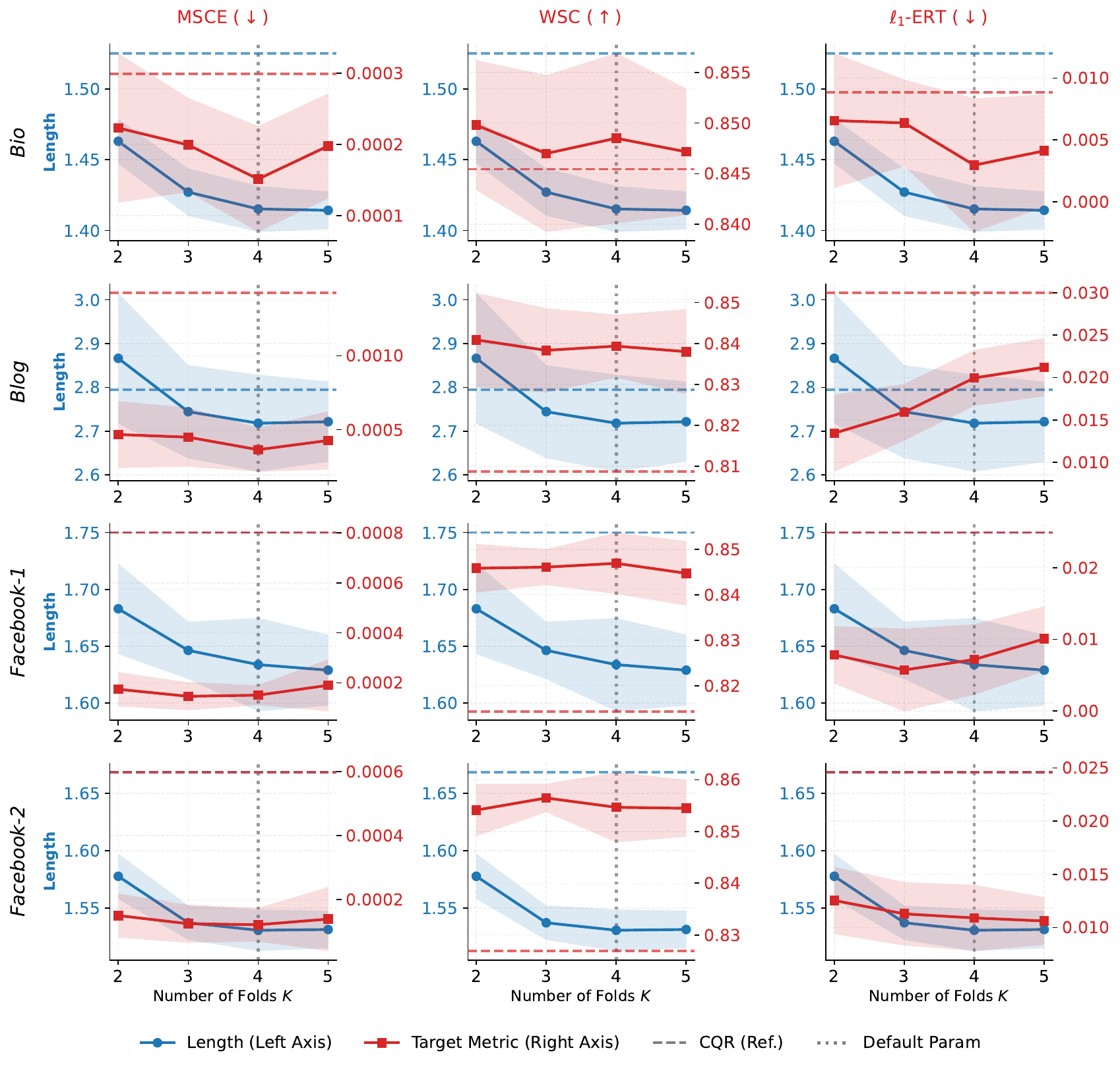}
    \end{minipage}
    \caption{Sensitivity to the number of folds \(K \in \{2, 3, 4, 5\}\) using the Fast variant. \textbf{Top:} Synthetic benchmarks. \textbf{Bottom:} The four largest real datasets. The gray dotted line marks the default \(K = 4\). Increasing \(K\) from 2 to 4 yields clear efficiency gains and stabilizes conditional reliability, particularly on real datasets, with performance saturating at \(K = 4 \sim 5\).}
    \label{fig:K_combined_grid}
\end{figure}

\begin{figure}[htbp]
    \centering
    \begin{minipage}{\textwidth}
        \centering
        \includegraphics[width=\textwidth]{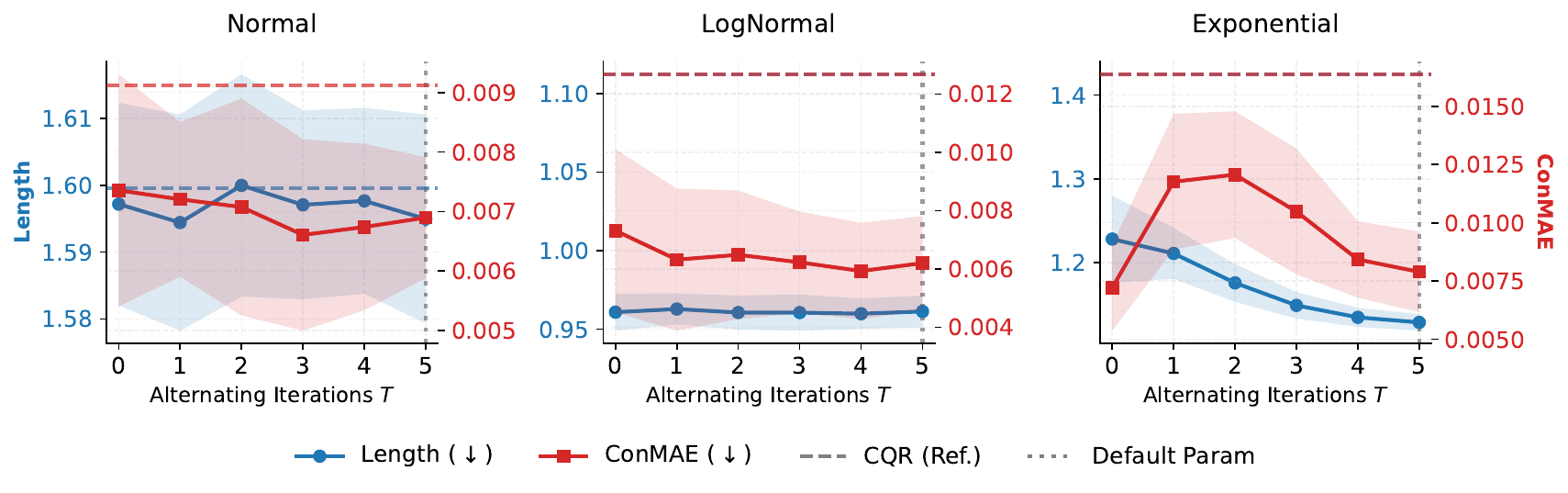}
    \end{minipage}
    \begin{minipage}{\textwidth}
        \centering
        \includegraphics[width=\textwidth]{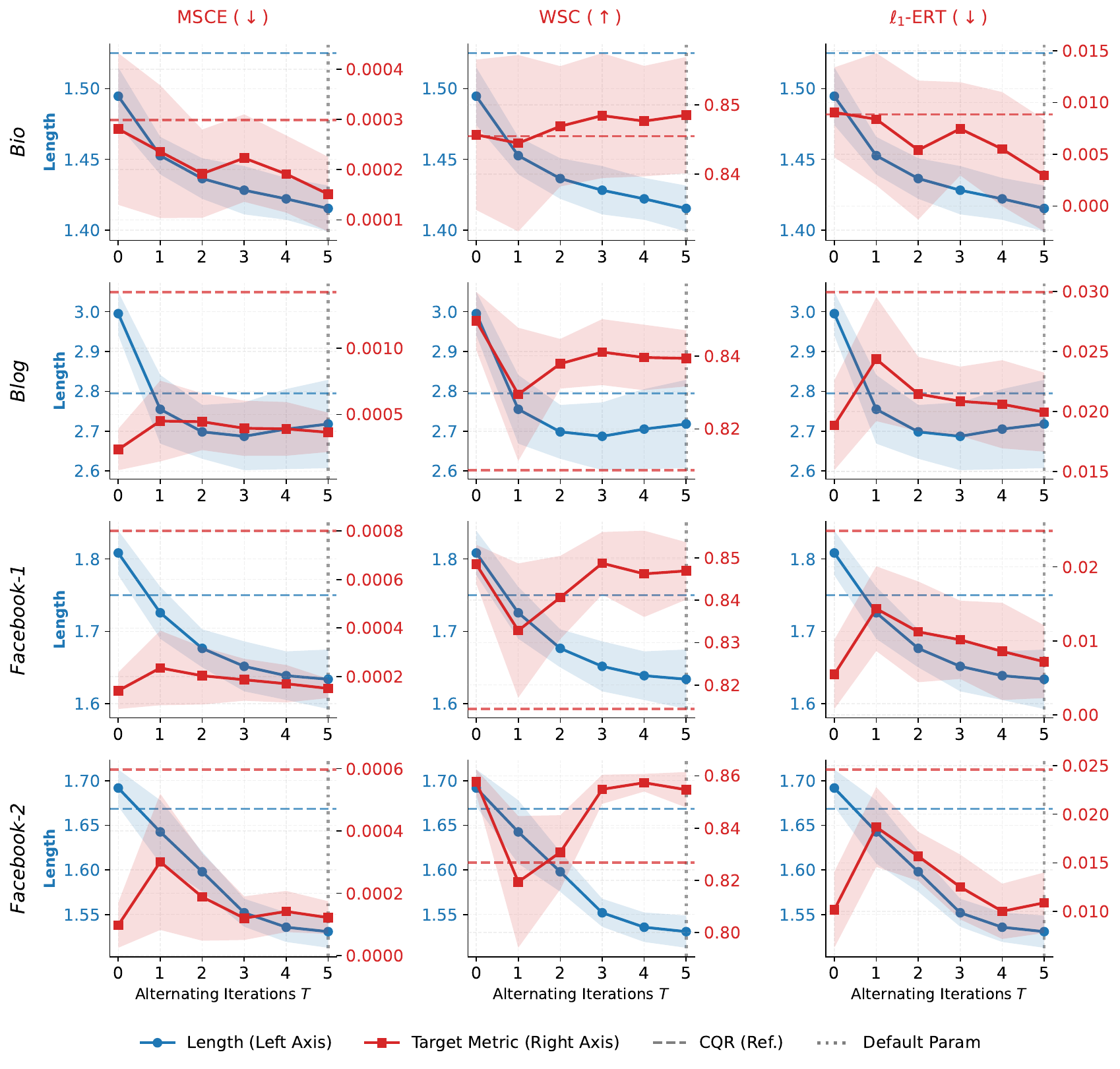}
    \end{minipage}
    \caption{Sensitivity to alternating updates \(T \in \{0, 1, 2, 3, 4, 5\}\) (Fast variant), serving as an ablation of the center-refinement step. \textbf{Top:} Synthetic benchmarks. \textbf{Bottom:} The four largest real datasets. The \(T = 0\) case represents the no-refinement baseline (learning radius only). While \(T=0\) shows no significant advantage over the CQR reference (dashed lines) in interval length, enabling alternating updates (\(T>0\)) yields substantial gains. Performance typically plateaus within 3 to 5 iterations, at which point CoCP consistently outperforms CQR across all metrics.}
    \label{fig:T_combined_grid}
\end{figure}

\newpage
\subsection{Sample-Adaptive Smoothing Strategy}
\label{sec:adaptive_beta}

While a global smoothing parameter \(\beta\) performs robustly on average (Appendix~\ref{app:exp-sensitivity}), it requires tuning based on the scale of the target variable and may be suboptimal under severe heteroscedasticity. A fixed \(\beta\) can under-smooth wide intervals (vanishing gradients) and over-smooth narrow ones (distorting the indicator geometry). To achieve scale invariance, we propose a sample-adaptive strategy that scales \(\beta(x)\) proportionally to the local interval radius, preserving a consistent geometric shape across the covariate space.

\paragraph{Geometric preservation and closed-form approximation.}
Let \(\tilde{\mathbf{1}}_\beta(y; m, h) = \sigma\bigl(\frac{h-(y-m)}{\beta}\bigr) - \sigma\bigl(\frac{-h-(y-m)}{\beta}\bigr)\) denote the smoothed interval indicator. To maintain a ``flat'' interior, we require \(\tilde{\mathbf{1}}_\beta\) to retain a high value of \(1 - \epsilon\) (for a small tolerance \(\epsilon > 0\)) at a fractional distance \(p \in (0, 1)\) of the radius \(h\) from the center \(m\). Evaluating \(\tilde{\mathbf{1}}_\beta(m + ph; m, h) = 1 - \epsilon\) yields:
\begin{equation}
\sigma\left(\frac{(1-p)h}{\beta}\right) - \sigma\left(-\frac{(1+p)h}{\beta}\right) = 1 - \epsilon.
\end{equation}
In practice, \(\beta \ll h\), making the second sigmoid term exponentially vanishing (\(\approx 0\)). Omitting it provides a highly accurate closed-form approximation for the adaptive smoothing parameter:
\begin{equation}
\beta(x) \approx \frac{(1-p)h(x)}{\log\left(\frac{1-\epsilon}{\epsilon}\right)}.
\end{equation}
The truncation error is bounded by \(O(\epsilon^{(1+p)/(1-p)})\) (e.g., \(\sim 10^{-18}\) for \(p=0.8, \epsilon=0.01\)), rendering this closed-form expression virtually exact.

\paragraph{Implementation via QR warm-up.}
Since the optimal radius \(h(x)\) is unknown a priori, we approximate it using a lightweight Quantile Regression (QR) warm-up. We first train a base QR model to estimate initial bounds \(\widehat{q}_{\alpha/2}(x)\) and \(\widehat{q}_{1-\alpha/2}(x)\), yielding a proxy radius \(\widehat{h}(x) = [\widehat{q}_{1-\alpha/2}(x) - \widehat{q}_{\alpha/2}(x)] / 2\). For each sample \(x_i\), we then compute the instance-dependent \(\beta(x_i)\) using \(\widehat{h}(x_i)\) and proceed with the standard CoCP optimization. The complete procedure is summarized in Algorithm~\ref{alg:adaptive_cocp}. Note that this fixed-\((p,\epsilon)\) adaptive rule is a practical scale-normalized heuristic. To strictly satisfy the vanishing schedule \(\beta_n \to 0\) required by Theorem~\ref{thm:cocp_asymp}, one can set \(\beta_{i,n} = c_n \widehat{h}(x_i)\) with \(c_n \to 0\).

\begin{algorithm}[ht]
\caption{Sample-Adaptive Smoothing Strategy for CoCP}
\label{alg:adaptive_cocp}
\begin{algorithmic}[1]
\Require Training data \(\mathcal{D}_{\mathrm{train}}\), calibration data \(\mathcal{D}_{\mathrm{cal}}\), miscoverage \(\alpha\in(0,1)\), folds \(K\), iterations \(T\), interior proportion \(p \in (0,1)\), tolerance \(\epsilon > 0\)
\Ensure Calibrated interval function \(\widehat{C}(\cdot)\)
\State \textbf{Phase 1: QR Warm-up}
\State Train a base quantile regression model on \(\mathcal{D}_{\mathrm{train}}\) to estimate \(\widehat{q}_{\alpha/2}(x)\) and \(\widehat{q}_{1-\alpha/2}(x)\).
\State For each \((X_i, Y_i) \in \mathcal{D}_{\mathrm{train}}\), compute the proxy radius: 
\[
\widehat{h}(X_i) = \frac{1}{2}\left(\widehat{q}_{1-\alpha/2}(X_i) - \widehat{q}_{\alpha/2}(X_i)\right)
\]
\State Compute the instance-dependent smoothing parameter:
\[
\beta_i = \frac{(1-p)\widehat{h}(X_i)}{\log((1-\epsilon)/\epsilon)}
\]
\State \textbf{Phase 2: Adaptive Alternating Optimization}
\State Execute the \(K\)-fold alternating optimization (Lines 1-11 in Algorithm~\ref{alg:cocp}), replacing the global \(\beta\) with the sample-specific \(\beta_i\) in the smooth pinball loss \eqref{eq:rad-loss-smooth} and smooth interval loss \eqref{eq:ctr-loss}.
\State \textbf{Phase 3: Calibration}
\State Compute nonconformity scores and calibrate on \(\mathcal{D}_{\mathrm{cal}}\) (Lines 12-15 in Algorithm~\ref{alg:cocp}).
\State \Return \(\widehat{C}(x)\)
\end{algorithmic}
\end{algorithm}

\paragraph{Empirical evaluation.}
We fix the tolerance \(\epsilon = 0.01\) and evaluate the interior proportion \(p \in \{0.5, 0.6, 0.7, 0.8, 0.9\}\) using the Fast variant of CoCP. Table~\ref{tab:adaptive_beta_combined} summarizes the results across all synthetic and real-world benchmarks. 

The adaptive strategy yields performance that is highly competitive with—and occasionally superior to—the well tuned fixed-\(\beta\) baseline. On complex and skewed distributions (e.g., Exponential, \textit{Blog}, \textit{Facebook}), the adaptive formulation effectively enhances conditional reliability (lower MSCE/\(\ell_1\)-ERT, higher WSC) while maintaining tight intervals. Crucially, this strategy renders CoCP invariant to the absolute scale of the target variable, effectively alleviating the burden of hyperparameter tuning. The performance remains remarkably stable across different choices of \(p\), establishing \(p \in [0.6, 0.8]\) (e.g., \(p=0.7\)) as a robust, plug-and-play default for practical deployments.

\begin{table}[ht!]
\centering
\caption{Performance comparison of the Fixed \(\beta\) strategy versus the Sample-Adaptive Smoothing strategy across different interior proportions \(p\). Best results per row are highlighted in \textbf{bold}. Values in \textcolor{blue}{blue} indicate an improvement over the Fixed \(\beta\) baseline. The adaptive strategy matches or frequently exceeds the baseline, offering a robust, scale-invariant default.}
\label{tab:adaptive_beta_combined}
\resizebox{\textwidth}{!}{
\begin{tabular}{@{} l l c ccccc @{}}
\toprule
\multirow{2}{*}{\textbf{Dataset}} & \multirow{2}{*}{\textbf{Metric}} & \textbf{Fixed \(\beta\)} & \multicolumn{5}{c}{\textbf{Adaptive \(\beta\) (Fast CoCP)}} \\
\cmidrule(l){4-8}
& & \textbf{(Fast CoCP)} & \(p=0.5\) & \(p=0.6\) & \(p=0.7\) & \(p=0.8\) & \(p=0.9\) \\
\midrule
\multicolumn{8}{c}{\textit{Synthetic Datasets}} \\
\midrule
\multirow{2}{*}{Normal}
 & Length (\(\downarrow\)) & 1.5964 \std{0.0144} & 1.5978 \std{0.0159} & 1.5978 \std{0.0162} & \textcolor{blue}{1.5959} \std{0.0167} & \textbf{\textcolor{blue}{1.5938}} \std{0.0161} & \textcolor{blue}{1.5953} \std{0.0170} \\
 & ConMAE (\(\downarrow\)) & 0.0067 \std{0.0014} & \textcolor{blue}{0.0064} \std{0.0013} & \textbf{\textcolor{blue}{0.0063}} \std{0.0013} & \textcolor{blue}{0.0065} \std{0.0011} & 0.0067 \std{0.0014} & 0.0072 \std{0.0014} \\
\midrule
\multirow{2}{*}{LogNormal}
 & Length (\(\downarrow\)) & 0.9602 \std{0.0099} & 0.9603 \std{0.0118} & \textcolor{blue}{0.9599} \std{0.0117} & \textcolor{blue}{0.9595} \std{0.0111} & \textcolor{blue}{0.9599} \std{0.0105} & \textbf{\textcolor{blue}{0.9581}} \std{0.0096} \\
 & ConMAE (\(\downarrow\)) & \textbf{0.0058} \std{0.0015} & 0.0065 \std{0.0020} & 0.0063 \std{0.0015} & 0.0064 \std{0.0016} & 0.0059 \std{0.0013} & 0.0059 \std{0.0015} \\
\midrule
\multirow{2}{*}{Exponential}
 & Length (\(\downarrow\)) & 1.1290 \std{0.0095} & 1.1632 \std{0.0312} & 1.1513 \std{0.0318} & 1.1403 \std{0.0265} & 1.1294 \std{0.0267} & \textbf{\textcolor{blue}{1.1209}} \std{0.0147} \\
 & ConMAE (\(\downarrow\)) & 0.0077 \std{0.0014} & \textcolor{blue}{0.0068} \std{0.0014} & \textcolor{blue}{0.0067} \std{0.0013} & \textbf{\textcolor{blue}{0.0065}} \std{0.0016} & \textcolor{blue}{0.0071} \std{0.0020} & 0.0077 \std{0.0018} \\
\midrule
\multicolumn{8}{c}{\textit{Real-World Datasets}} \\
\midrule
\multirow{5}{*}{\textit{Bike}}
 & Length (\(\downarrow\)) & 0.6233 \std{0.0179} & \textcolor{blue}{0.6136} \std{0.0157} & \textcolor{blue}{0.6077} \std{0.0119} & \textbf{\textcolor{blue}{0.6014}} \std{0.0166} & \textcolor{blue}{0.6082} \std{0.0181} & \textcolor{blue}{0.6157} \std{0.0216} \\
 & MSCE (\(\downarrow\)) & 0.0008 \std{0.0004} & 0.0009 \std{0.0005} & 0.0009 \std{0.0004} & 0.0008 \std{0.0004} & \textbf{\textcolor{blue}{0.0007}} \std{0.0002} & \textcolor{blue}{0.0007} \std{0.0003} \\
 & WSC (\(\uparrow\)) & 0.7889 \std{0.0139} & \textcolor{blue}{0.7893} \std{0.0201} & \textcolor{blue}{0.7897} \std{0.0182} & 0.7796 \std{0.0191} & \textbf{\textcolor{blue}{0.7961}} \std{0.0149} & \textcolor{blue}{0.7940} \std{0.0134} \\
 & \(\ell_1\)-ERT (\(\downarrow\)) & 0.0165 \std{0.0125} & 0.0185 \std{0.0075} & 0.0175 \std{0.0095} & \textcolor{blue}{0.0155} \std{0.0110} & \textbf{\textcolor{blue}{0.0132}} \std{0.0118} & \textcolor{blue}{0.0150} \std{0.0080} \\
 & \(\ell_2\)-ERT (\(\downarrow\)) & 0.0001 \std{0.0008} & 0.0004 \std{0.0007} & 0.0003 \std{0.0009} & 0.0001 \std{0.0010} & \textbf{\textcolor{blue}{-0.0000}} \std{0.0007} & \textcolor{blue}{-0.0000} \std{0.0006} \\
\midrule
\multirow{5}{*}{\textit{Bio}}
 & Length (\(\downarrow\)) & 1.4152 \std{0.0162} & \textbf{\textcolor{blue}{1.3923}} \std{0.0139} & \textcolor{blue}{1.3972} \std{0.0146} & \textcolor{blue}{1.3975} \std{0.0147} & \textcolor{blue}{1.3960} \std{0.0165} & \textcolor{blue}{1.4047} \std{0.0179} \\
 & MSCE (\(\downarrow\)) & \textbf{0.0002} \std{0.0001} & 0.0002 \std{0.0001} & 0.0003 \std{0.0001} & 0.0002 \std{0.0001} & 0.0002 \std{0.0001} & 0.0002 \std{0.0001} \\
 & WSC (\(\uparrow\)) & \textbf{0.8485} \std{0.0084} & 0.8443 \std{0.0103} & 0.8464 \std{0.0092} & 0.8435 \std{0.0105} & 0.8458 \std{0.0086} & 0.8455 \std{0.0088} \\
 & \(\ell_1\)-ERT (\(\downarrow\)) & \textbf{0.0030} \std{0.0054} & 0.0057 \std{0.0067} & 0.0071 \std{0.0054} & 0.0064 \std{0.0046} & 0.0050 \std{0.0060} & 0.0047 \std{0.0048} \\
 & \(\ell_2\)-ERT (\(\downarrow\)) & \textbf{0.0000} \std{0.0001} & 0.0000 \std{0.0001} & 0.0001 \std{0.0002} & 0.0000 \std{0.0001} & \textcolor{blue}{-0.0000} \std{0.0001} & \textcolor{blue}{-0.0000} \std{0.0001} \\
\midrule
\multirow{5}{*}{\textit{Blog}}
 & Length (\(\downarrow\)) & \textbf{2.7180} \std{0.1107} & 2.8474 \std{0.0913} & 2.8348 \std{0.1151} & 2.8057 \std{0.1110} & 2.7814 \std{0.0990} & 2.7317 \std{0.1062} \\
 & MSCE (\(\downarrow\)) & 0.0004 \std{0.0001} & \textcolor{blue}{0.0003} \std{0.0001} & \textbf{\textcolor{blue}{0.0003}} \std{0.0001} & \textcolor{blue}{0.0003} \std{0.0001} & 0.0004 \std{0.0001} & 0.0004 \std{0.0001} \\
 & WSC (\(\uparrow\)) & 0.8393 \std{0.0077} & \textcolor{blue}{0.8432} \std{0.0072} & \textbf{\textcolor{blue}{0.8434}} \std{0.0083} & \textcolor{blue}{0.8432} \std{0.0101} & \textcolor{blue}{0.8404} \std{0.0087} & \textcolor{blue}{0.8402} \std{0.0089} \\
 & \(\ell_1\)-ERT (\(\downarrow\)) & 0.0200 \std{0.0033} & \textbf{\textcolor{blue}{0.0163}} \std{0.0035} & \textcolor{blue}{0.0181} \std{0.0044} & \textcolor{blue}{0.0171} \std{0.0038} & \textcolor{blue}{0.0189} \std{0.0042} & \textcolor{blue}{0.0195} \std{0.0045} \\
 & \(\ell_2\)-ERT (\(\downarrow\)) & -0.0015 \std{0.0006} & -0.0015 \std{0.0007} & -0.0015 \std{0.0007} & \textbf{\textcolor{blue}{-0.0017}} \std{0.0005} & -0.0014 \std{0.0010} & -0.0014 \std{0.0009} \\
\midrule
\multirow{5}{*}{\textit{Facebook-1}}
 & Length (\(\downarrow\)) & \textbf{1.6338} \std{0.0411} & 1.7040 \std{0.0305} & 1.6969 \std{0.0286} & 1.6802 \std{0.0313} & 1.6652 \std{0.0410} & 1.6450 \std{0.0305} \\
 & MSCE (\(\downarrow\)) & 0.0002 \std{0.0000} & \textbf{\textcolor{blue}{0.0001}} \std{0.0001} & \textcolor{blue}{0.0001} \std{0.0001} & \textcolor{blue}{0.0001} \std{0.0001} & 0.0002 \std{0.0001} & 0.0002 \std{0.0001} \\
 & WSC (\(\uparrow\)) & 0.8469 \std{0.0068} & \textcolor{blue}{0.8478} \std{0.0077} & \textbf{\textcolor{blue}{0.8498}} \std{0.0058} & \textcolor{blue}{0.8491} \std{0.0068} & \textcolor{blue}{0.8476} \std{0.0065} & \textcolor{blue}{0.8486} \std{0.0066} \\
 & \(\ell_1\)-ERT (\(\downarrow\)) & \textbf{0.0072} \std{0.0049} & 0.0075 \std{0.0042} & 0.0079 \std{0.0038} & 0.0081 \std{0.0050} & 0.0078 \std{0.0042} & 0.0087 \std{0.0038} \\
 & \(\ell_2\)-ERT (\(\downarrow\)) & -0.0004 \std{0.0004} & -0.0004 \std{0.0003} & -0.0004 \std{0.0004} & -0.0004 \std{0.0004} & \textbf{\textcolor{blue}{-0.0005}} \std{0.0003} & -0.0004 \std{0.0003} \\
\midrule
\multirow{5}{*}{\textit{Facebook-2}}
 & Length (\(\downarrow\)) & \textbf{1.5308} \std{0.0181} & 1.5656 \std{0.0164} & 1.5574 \std{0.0161} & 1.5519 \std{0.0108} & 1.5405 \std{0.0167} & 1.5322 \std{0.0155} \\
 & MSCE (\(\downarrow\)) & 0.0001 \std{0.0001} & 0.0001 \std{0.0001} & \textbf{0.0001} \std{0.0000} & 0.0001 \std{0.0001} & 0.0001 \std{0.0001} & 0.0001 \std{0.0001} \\
 & WSC (\(\uparrow\)) & 0.8547 \std{0.0068} & \textbf{\textcolor{blue}{0.8598}} \std{0.0040} & \textcolor{blue}{0.8584} \std{0.0031} & \textcolor{blue}{0.8590} \std{0.0032} & \textcolor{blue}{0.8592} \std{0.0034} & \textcolor{blue}{0.8589} \std{0.0033} \\
 & \(\ell_1\)-ERT (\(\downarrow\)) & 0.0109 \std{0.0031} & \textcolor{blue}{0.0105} \std{0.0032} & \textbf{\textcolor{blue}{0.0096}} \std{0.0030} & \textcolor{blue}{0.0096} \std{0.0024} & \textcolor{blue}{0.0097} \std{0.0029} & \textcolor{blue}{0.0106} \std{0.0032} \\
 & \(\ell_2\)-ERT (\(\downarrow\)) & -0.0001 \std{0.0002} & -0.0001 \std{0.0001} & \textbf{\textcolor{blue}{-0.0002}} \std{0.0001} & -0.0001 \std{0.0001} & -0.0001 \std{0.0001} & -0.0001 \std{0.0002} \\
\midrule
\multirow{5}{*}{\textit{Homes}}
 & Length (\(\downarrow\)) & \textbf{0.5222} \std{0.0090} & 0.5279 \std{0.0098} & 0.5278 \std{0.0108} & 0.5292 \std{0.0070} & 0.5291 \std{0.0119} & 0.5269 \std{0.0092} \\
 & MSCE (\(\downarrow\)) & \textbf{0.0003} \std{0.0001} & 0.0004 \std{0.0002} & 0.0004 \std{0.0001} & 0.0004 \std{0.0002} & 0.0003 \std{0.0001} & 0.0004 \std{0.0002} \\
 & WSC (\(\uparrow\)) & \textbf{0.8272} \std{0.0097} & 0.8271 \std{0.0136} & 0.8235 \std{0.0150} & 0.8250 \std{0.0150} & 0.8250 \std{0.0097} & 0.8230 \std{0.0137} \\
 & \(\ell_1\)-ERT (\(\downarrow\)) & 0.0039 \std{0.0077} & 0.0064 \std{0.0064} & 0.0059 \std{0.0055} & \textcolor{blue}{0.0038} \std{0.0075} & 0.0048 \std{0.0073} & \textbf{\textcolor{blue}{0.0029}} \std{0.0073} \\
 & \(\ell_2\)-ERT (\(\downarrow\)) & -0.0004 \std{0.0003} & -0.0004 \std{0.0003} & -0.0003 \std{0.0003} & -0.0004 \std{0.0004} & -0.0004 \std{0.0004} & -0.0004 \std{0.0004} \\
\midrule
\multirow{5}{*}{\textit{Supercon.}}
 & Length (\(\downarrow\)) & 0.8134 \std{0.0135} & \textbf{\textcolor{blue}{0.7618}} \std{0.0179} & \textcolor{blue}{0.7696} \std{0.0149} & \textcolor{blue}{0.7698} \std{0.0140} & \textcolor{blue}{0.7779} \std{0.0126} & \textcolor{blue}{0.7932} \std{0.0102} \\
 & MSCE (\(\downarrow\)) & 0.0004 \std{0.0001} & 0.0004 \std{0.0001} & \textcolor{blue}{0.0003} \std{0.0002} & \textbf{\textcolor{blue}{0.0003}} \std{0.0002} & 0.0004 \std{0.0002} & 0.0004 \std{0.0001} \\
 & WSC (\(\uparrow\)) & 0.8247 \std{0.0098} & \textcolor{blue}{0.8252} \std{0.0098} & \textbf{\textcolor{blue}{0.8275}} \std{0.0081} & \textcolor{blue}{0.8251} \std{0.0082} & \textcolor{blue}{0.8262} \std{0.0137} & \textcolor{blue}{0.8258} \std{0.0081} \\
 & \(\ell_1\)-ERT (\(\downarrow\)) & 0.0114 \std{0.0070} & 0.0129 \std{0.0043} & \textbf{\textcolor{blue}{0.0099}} \std{0.0074} & 0.0129 \std{0.0053} & 0.0128 \std{0.0059} & 0.0126 \std{0.0048} \\
 & \(\ell_2\)-ERT (\(\downarrow\)) & \textbf{-0.0009} \std{0.0005} & -0.0006 \std{0.0005} & -0.0006 \std{0.0004} & -0.0005 \std{0.0004} & -0.0007 \std{0.0004} & -0.0007 \std{0.0006} \\
\bottomrule
\end{tabular}
}
\end{table}

\subsection{Center Refinement and Point Prediction}\label{app:center-tradeoff}

Because the learned center $m(x)$ targets the HDI midpoint rather than the conditional mean, it need not serve as a point predictor.
Figure~\ref{fig:center-point-efficiency} separates the two objectives across all ten benchmarks.
The $T\!=\!0$ control shares the MSE initialization, fold allocation, radius fitting, cross-fold averaging, and calibration of full CoCP, but omits center refinement.
It already matches or slightly improves upon the separately trained Split predictor in RMSE (left panel).
Raising $T$ to~5 shifts the center toward the interval-optimal target, shortening the interval on nine of ten datasets (middle panel).
The largest reductions appear under skewness (\textit{Exponential}: $-6.9\%$) and on larger real benchmarks (\textit{Facebook-1}: $-9.9\%$, \textit{Blog}: $-8.1\%$).
Conditional error (right panel) improves or holds steady on most datasets; the few apparent increases occur where MSCE is already near its floor (e.g., \textit{Facebook-2}: absolute MSCE stays below $10^{-4}$).
When point-prediction fidelity matters, the $T\!=\!0$ center or a separate predictor can be retained for point decisions, while CoCP is used solely for uncertainty quantification.

\begin{figure}[htbp]
  \centering
  \includegraphics[width=0.95\textwidth]{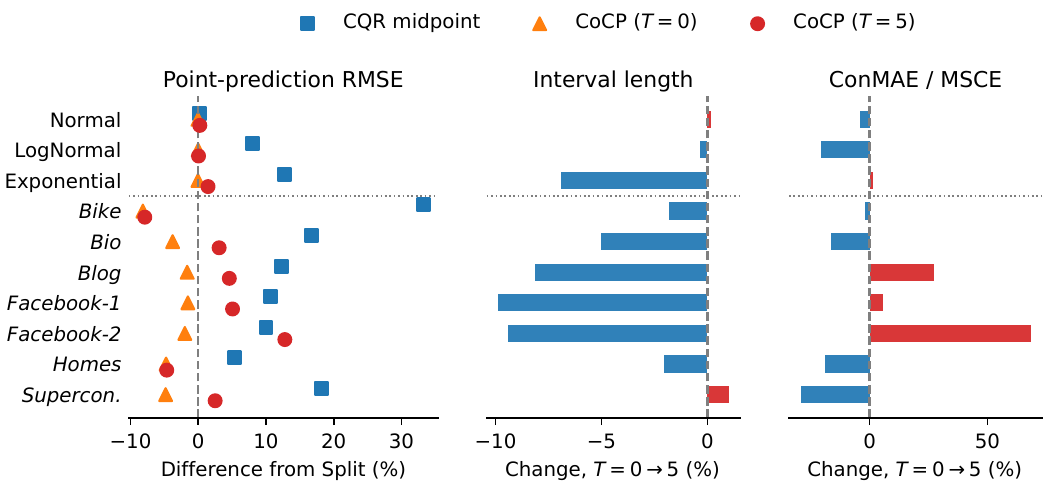}
  \caption{Point-prediction accuracy versus interval efficiency under center refinement ($\alpha=0.1$).
  \textbf{Left:} RMSE of each center relative to the Split predictor (zero line $=$ parity; negative values indicate lower RMSE).
  \textbf{Middle:} percentage change in interval length from $T\!=\!0$ (no refinement) to $T\!=\!5$.
  \textbf{Right:} percentage change in conditional error (ConMAE for synthetic, MSCE for real data).
  The dotted line separates synthetic (above) from real (below) datasets.}
  \label{fig:center-point-efficiency}
\end{figure}

\subsection{Fold Variability and Local Residual Modality}\label{app:fold-unimodality}

CoCP averages $K$ fold-specific predictors.
Figure~\ref{fig:fold-prediction-variability} reports the normalized spread among them at $(\alpha,K,T)=(0.1,4,5)$, averaged over ten splits.
On synthetic data, center disagreement is 1.8--2.4\% of interval length and radius disagreement is 3.0--4.8\% of half-length, small enough that ensembling provides only marginal stability gains.
On real data the corresponding ranges rise to 8--18\% and 16--25\%, consistent with the stronger benefit of $K\!\ge\!4$ observed in Figure~\ref{fig:K_combined_grid}.

\begin{figure}[htbp]
  \centering
  \includegraphics[width=0.65\textwidth]{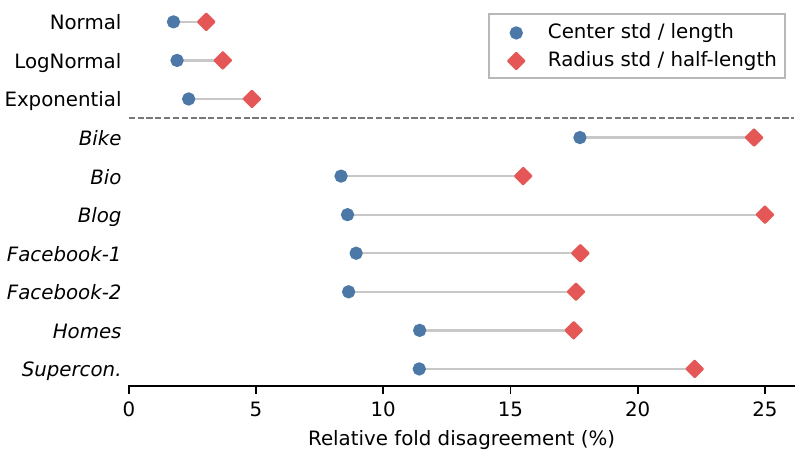}
  \caption{Relative fold disagreement at $(\alpha,K,T)=(0.1,4,5)$, averaged over ten splits.
  Center standard deviation is normalized by mean interval length; radius standard deviation by mean half-length.
  The dashed line separates synthetic (above) from real (below) datasets.}
  \label{fig:fold-prediction-variability}
\end{figure}

To assess whether widespread residual multimodality might explain the larger fold variability on real data, we apply Hartigan's dip test locally.
On each split, signed residuals $Y-\widehat{g}(X)$ from the MSE predictor are formed, covariates are standardized and projected onto the first $q=\min(d,32)$ principal components, and the dip test is applied to the $k=200$ nearest calibration residuals of 100 test anchors.
Table~\ref{tab:local-dip} shows that raw rejection rates never exceed 0.015; all Benjamini--Hochberg and Benjamini--Yekutieli adjusted rates are zero.
The conclusion is stable across $k\in\{100,200,400\}$ and $q\in\{16,32,64\}$.
Failure to reject does not establish unimodality, since each neighborhood pools nearby but non-identical covariates. These findings are consistent with finite-sample estimator variability contributing to the larger fold disagreement on real data, without ruling out departures from conditional unimodality.

\begin{table}[htbp]
\centering
\caption{Hartigan's dip test for local residual unimodality ($q\!=\!\min(d,32)$ PCA components, $k\!=\!200$ nearest neighbors). Mean $\pm$ std over ten splits.}
\label{tab:local-dip}
\begin{tabular}{lcccc}
\toprule
Dataset & Median dip & Raw rejection & BH rejection & BY rejection \\
\midrule
\textit{Bike} & 0.0171 & 0.000$\pm$0.000 & 0.000$\pm$0.000 & 0.000$\pm$0.000 \\
\textit{Bio} & 0.0181 & 0.003$\pm$0.007 & 0.000$\pm$0.000 & 0.000$\pm$0.000 \\
\textit{Blog} & 0.0190 & 0.008$\pm$0.012 & 0.000$\pm$0.000 & 0.000$\pm$0.000 \\
\textit{Facebook-1} & 0.0196 & 0.003$\pm$0.007 & 0.000$\pm$0.000 & 0.000$\pm$0.000 \\
\textit{Facebook-2} & 0.0202 & 0.015$\pm$0.015 & 0.000$\pm$0.000 & 0.000$\pm$0.000 \\
\textit{Homes} & 0.0175 & 0.000$\pm$0.000 & 0.000$\pm$0.000 & 0.000$\pm$0.000 \\
\textit{Supercon.} & 0.0176 & 0.000$\pm$0.000 & 0.000$\pm$0.000 & 0.000$\pm$0.000 \\
\bottomrule
\end{tabular}
\end{table}

\subsection{Comparison with Asymmetric Tail Allocation}\label{app:acqr}

An alternative to translating the center is to optimize a single global split of the two tail probabilities.
We implement this baseline (A-CQR) with intervals
\begin{equation*}
[\widehat q_{\lambda\alpha}(x),\,\widehat q_{1-(1-\lambda)\alpha}(x)],\quad 
\lambda\in\{0.1,\ldots,0.9\}.
\end{equation*}
Candidate models are fitted on the training split, the validation split selects the $\lambda$ minimizing conformalized length, and the selected model is recalibrated on the independent calibration split.
Table~\ref{tab:acqr-main} compares A-CQR with CoCP at $\alpha=0.1$.
A-CQR already improves markedly over equal-tailed CQR on the skewed benchmarks.
CoCP is shorter on nine of ten datasets (\textit{Blog} excepted) and achieves lower conditional error on all ten, consistent with its ability to adapt the center locally as a function of~$x$ rather than selecting a single global~$\lambda$.

\begin{table}[htbp]
\centering
\caption{Comparison with validation-tuned A-CQR at $\alpha=0.1$ (mean over ten splits). The conditional metric is ConMAE on synthetic data and MSCE on real data.}
\label{tab:acqr-main}
\begin{tabular}{lc>{ \columncolor{gray!15} }cc>{ \columncolor{gray!15} }cc>{ \columncolor{gray!15} }c}
\toprule
& \multicolumn{2}{c}{Coverage} & \multicolumn{2}{c}{Length (\(\downarrow\))} & \multicolumn{2}{c}{ConMAE / MSCE (\(\downarrow\))} \\
Dataset & A-CQR & \textbf{CoCP} & A-CQR & \textbf{CoCP} & A-CQR & \textbf{CoCP} \\
\midrule
Normal & 0.9010 & 0.9002 & 1.5967 & \textbf{1.5951} & 0.0098 & \textbf{0.0072} \\
LogNormal & 0.9000 & 0.9006 & 0.9622 & \textbf{0.9606} & 0.0083 & \textbf{0.0065} \\
Exponential & 0.8976 & 0.8968 & 1.1575 & \textbf{1.1278} & 0.0121 & \textbf{0.0079} \\
\midrule
\textit{Bike} & 0.9045 & 0.8999 & 0.7485 & \textbf{0.6230} & 0.0018 & \textbf{0.0008} \\
\textit{Bio} & 0.8997 & 0.9011 & 1.4851 & \textbf{1.4162} & \textbf{0.0002} & \textbf{0.0002} \\
\textit{Blog} & 0.8989 & 0.8986 & \textbf{2.4652} & 2.7090 & 0.0006 & \textbf{0.0004} \\
\textit{Facebook-1} & 0.8984 & 0.9012 & 1.6618 & \textbf{1.6389} & 0.0010 & \textbf{0.0002} \\
\textit{Facebook-2} & 0.8969 & 0.8985 & 1.5704 & \textbf{1.5324} & 0.0005 & \textbf{0.0001} \\
\textit{Homes} & 0.9014 & 0.9015 & 0.5590 & \textbf{0.5230} & 0.0011 & \textbf{0.0003} \\
\textit{Supercon.} & 0.9038 & 0.9040 & 0.8676 & \textbf{0.8071} & 0.0006 & \textbf{0.0003} \\
\bottomrule
\end{tabular}
\end{table}

\subsection{Robustness Across Coverage Levels}\label{app:coverage-levels}

The main experiments (Tables~\ref{tab:synthetic_results} and~\ref{tab:real_world_results}) target $\alpha=0.1$.
To assess sensitivity to the coverage level, we retrain and independently calibrate every method at $\alpha\in\{0.05,0.2\}$ using the same default hyperparameters, without any level-specific tuning.
Tables~\ref{tab:alpha005-ranks} and~\ref{tab:alpha020-ranks} report within-dataset ranks on the ten original benchmarks.
For real data, the conditional-reliability rank averages the MSCE, WSC, $\ell_1$-ERT, and $\ell_2$-ERT ranks (with WSC ranked in descending order).
CoCP achieves the best overall mean rank in both length and conditional reliability at every evaluated level: $2.40\,/\,1.55$ at $\alpha=0.05$ and $2.20\,/\,1.82$ at $\alpha=0.2$.
The advantage is strongest at the default $\alpha=0.1$ ($1.40\,/\,1.00$) but remains favorable at both alternative levels despite no per-level tuning.

\begin{table}[htbp]
\centering
\caption{Per-dataset ranks at $\alpha=0.05$. Each entry is length rank / conditional-reliability rank; smaller is better for every reported rank. The conditional rank is based on ConMAE for synthetic data and is the mean of the MSCE, WSC, $\ell_1$-ERT, and $\ell_2$-ERT ranks for real data. WSC is ranked in descending order before averaging. Bold marks the best rank for each dataset and criterion.}
\label{tab:alpha005-ranks}
\resizebox{\textwidth}{!}{%
\begin{tabular}{@{} l ccc ccccccc c @{}}
\toprule
& \multicolumn{3}{c}{Synthetic} & \multicolumn{7}{c}{Real} & \multicolumn{1}{c}{Overall} \\
\cmidrule(lr){2-4} \cmidrule(lr){5-11} \cmidrule(l){12-12}
Method & Normal & LogNormal & Exponential & \textit{Bike} & \textit{Bio} & \textit{Blog} & \textit{Facebook-1} & \textit{Facebook-2} & \textit{Homes} & \textit{Supercon.} & Mean \\
\midrule
Split & 9.00\,/\,9.00 & 9.00\,/\,9.00 & 9.00\,/\,9.00 & 6.00\,/\,9.00 & 9.00\,/\,9.00 & 8.00\,/\,9.00 & 9.00\,/\,9.00 & 9.00\,/\,9.00 & 9.00\,/\,9.00 & 8.00\,/\,9.00 & 8.50\,/\,9.00 \\
CPL & \textbf{1.00}\,/\,8.00 & 3.00\,/\,8.00 & 4.00\,/\,8.00 & 5.00\,/\,5.00 & 6.00\,/\,5.25 & 9.00\,/\,7.75 & 6.00\,/\,5.00 & 6.00\,/\,5.00 & 8.00\,/\,7.25 & 5.00\,/\,6.75 & 5.30\,/\,6.60 \\
CQR & 5.00\,/\,3.00 & 5.00\,/\,6.00 & 5.00\,/\,7.00 & 8.00\,/\,2.50 & 3.00\,/\,4.00 & 2.00\,/\,4.75 & 3.00\,/\,7.75 & 4.00\,/\,8.00 & 4.00\,/\,5.25 & 3.00\,/\,2.50 & 4.20\,/\,5.08 \\
A-CQR & 3.00\,/\,4.00 & 2.00\,/\,3.00 & \textbf{1.00}\,/\,4.00 & 7.00\,/\,7.75 & 2.00\,/\,2.75 & \textbf{1.00}\,/\,\textbf{2.25} & \textbf{1.00}\,/\,4.75 & \textbf{1.00}\,/\,6.75 & 5.00\,/\,6.00 & 2.00\,/\,5.75 & 2.50\,/\,4.70 \\
CPCP & 4.00\,/\,7.00 & 8.00\,/\,7.00 & 6.00\,/\,6.00 & 3.00\,/\,3.75 & 8.00\,/\,7.00 & 6.00\,/\,3.50 & 8.00\,/\,5.75 & 8.00\,/\,4.75 & 6.00\,/\,5.50 & 4.00\,/\,4.75 & 6.10\,/\,5.50 \\
RCP & 2.00\,/\,6.00 & 6.00\,/\,4.00 & 7.00\,/\,3.00 & 4.00\,/\,5.50 & 7.00\,/\,6.00 & 7.00\,/\,6.50 & 7.00\,/\,3.00 & 7.00\,/\,2.50 & 7.00\,/\,6.00 & 6.00\,/\,6.00 & 6.00\,/\,4.85 \\
C-HDR & 6.00\,/\,2.00 & 7.00\,/\,\textbf{1.00} & 8.00\,/\,\textbf{1.00} & 2.00\,/\,6.00 & 4.00\,/\,7.00 & 5.00\,/\,3.75 & 5.00\,/\,3.50 & 5.00\,/\,3.25 & 2.00\,/\,2.25 & 7.00\,/\,2.75 & 5.10\,/\,3.25 \\
CHR & 8.00\,/\,5.00 & 4.00\,/\,5.00 & 2.00\,/\,5.00 & 9.00\,/\,4.00 & 5.00\,/\,2.75 & 3.00\,/\,4.50 & 4.00\,/\,5.25 & 2.00\,/\,4.75 & 3.00\,/\,2.50 & 9.00\,/\,6.00 & 4.90\,/\,4.47 \\
\rowcolor{gray!15} \textbf{CoCP} & 7.00\,/\,\textbf{1.00} & \textbf{1.00}\,/\,2.00 & 3.00\,/\,2.00 & \textbf{1.00}\,/\,\textbf{1.50} & \textbf{1.00}\,/\,\textbf{1.25} & 4.00\,/\,3.00 & 2.00\,/\,\textbf{1.00} & 3.00\,/\,\textbf{1.00} & \textbf{1.00}\,/\,\textbf{1.25} & \textbf{1.00}\,/\,\textbf{1.50} & \textbf{2.40}\,/\,\textbf{1.55} \\
\bottomrule
\end{tabular}%
}
\end{table}

\begin{table}[htbp]
\centering
\caption{Per-dataset ranks at $\alpha=0.20$. Each entry is length rank / conditional-reliability rank; smaller is better for every reported rank. The conditional rank is based on ConMAE for synthetic data and is the mean of the MSCE, WSC, $\ell_1$-ERT, and $\ell_2$-ERT ranks for real data. WSC is ranked in descending order before averaging. Bold marks the best rank for each dataset and criterion.}
\label{tab:alpha020-ranks}
\resizebox{\textwidth}{!}{%
\begin{tabular}{@{} l ccc ccccccc c @{}}
\toprule
& \multicolumn{3}{c}{Synthetic} & \multicolumn{7}{c}{Real} & \multicolumn{1}{c}{Overall} \\
\cmidrule(lr){2-4} \cmidrule(lr){5-11} \cmidrule(l){12-12}
Method & Normal & LogNormal & Exponential & \textit{Bike} & \textit{Bio} & \textit{Blog} & \textit{Facebook-1} & \textit{Facebook-2} & \textit{Homes} & \textit{Supercon.} & Mean \\
\midrule
Split & 2.00\,/\,9.00 & 9.00\,/\,9.00 & 9.00\,/\,9.00 & 5.00\,/\,9.00 & 9.00\,/\,8.75 & 8.00\,/\,9.00 & 9.00\,/\,9.00 & 9.00\,/\,9.00 & \textbf{1.00}\,/\,9.00 & 4.00\,/\,9.00 & 6.50\,/\,8.97 \\
CPL & \textbf{1.00}\,/\,8.00 & 4.00\,/\,8.00 & 4.00\,/\,8.00 & 9.00\,/\,7.75 & 5.00\,/\,7.00 & 9.00\,/\,4.25 & 8.00\,/\,3.50 & 8.00\,/\,3.00 & 9.00\,/\,8.00 & 7.00\,/\,6.75 & 6.40\,/\,6.42 \\
CQR & 4.00\,/\,3.00 & 8.00\,/\,4.00 & 8.00\,/\,4.00 & 4.00\,/\,6.75 & 6.00\,/\,2.50 & 5.00\,/\,4.25 & 5.00\,/\,6.00 & 5.00\,/\,5.25 & 5.00\,/\,4.75 & 3.00\,/\,3.50 & 5.30\,/\,4.40 \\
A-CQR & 3.00\,/\,5.00 & \textbf{1.00}\,/\,2.00 & 3.00\,/\,2.00 & 3.00\,/\,6.50 & 3.00\,/\,2.50 & 3.00\,/\,\textbf{2.50} & 3.00\,/\,3.50 & 3.00\,/\,4.00 & 4.00\,/\,6.00 & 2.00\,/\,3.25 & 2.80\,/\,3.73 \\
CPCP & 8.00\,/\,7.00 & 6.00\,/\,6.00 & 6.00\,/\,6.00 & 7.00\,/\,3.50 & 7.00\,/\,4.25 & 6.00\,/\,4.25 & 6.00\,/\,3.50 & 7.00\,/\,2.50 & 7.00\,/\,6.00 & 5.00\,/\,6.50 & 6.50\,/\,4.95 \\
RCP & 7.00\,/\,6.00 & 5.00\,/\,5.00 & 7.00\,/\,5.00 & 6.00\,/\,2.00 & 8.00\,/\,5.50 & 7.00\,/\,4.75 & 7.00\,/\,\textbf{1.00} & 6.00\,/\,\textbf{1.25} & 8.00\,/\,4.75 & 8.00\,/\,5.00 & 6.90\,/\,4.03 \\
C-HDR & 5.00\,/\,2.00 & 7.00\,/\,7.00 & 5.00\,/\,7.00 & 8.00\,/\,3.50 & 4.00\,/\,8.25 & 4.00\,/\,8.00 & 4.00\,/\,7.50 & 4.00\,/\,8.00 & 6.00\,/\,2.50 & 9.00\,/\,3.25 & 5.60\,/\,5.70 \\
CHR & 9.00\,/\,4.00 & 3.00\,/\,3.00 & 2.00\,/\,3.00 & 2.00\,/\,5.00 & \textbf{1.00}\,/\,5.25 & \textbf{1.00}\,/\,5.25 & \textbf{1.00}\,/\,7.50 & \textbf{1.00}\,/\,7.00 & 2.00\,/\,3.00 & 6.00\,/\,6.75 & 2.80\,/\,4.97 \\
\rowcolor{gray!15} \textbf{CoCP} & 6.00\,/\,\textbf{1.00} & 2.00\,/\,\textbf{1.00} & \textbf{1.00}\,/\,\textbf{1.00} & \textbf{1.00}\,/\,\textbf{1.00} & 2.00\,/\,\textbf{1.00} & 2.00\,/\,2.75 & 2.00\,/\,3.50 & 2.00\,/\,5.00 & 3.00\,/\,\textbf{1.00} & \textbf{1.00}\,/\,\textbf{1.00} & \textbf{2.20}\,/\,\textbf{1.82} \\
\bottomrule
\end{tabular}%
}
\end{table}

\subsection{Sensitivity to Data Availability}\label{app:data-size-sensitivity}

We vary the amount of training and calibration data on the LogNormal and Exponential benchmarks at $\alpha=0.1$.
In the training-size experiment, every method is refitted for each $n\in\{2000,5000,10000,20000\}$.
In the calibration-size experiment, the fitted models are held fixed while calibration subsets of size $n_{\mathrm{cal}}\in\{100,250,500,1000,2000,4000\}$ are resampled 50 times per training run.
Figure~\ref{fig:data-availability-sensitivity} compares Split, CQR, A-CQR, CHR, and CoCP via within-dataset metric ratios relative to CQR, averaged over the two benchmarks.

Across training sizes, CoCP reduces Length by approximately $13$--$18\%$ and ConMAE by $25$--$46\%$ relative to CQR.
These gains are already present at $n=2{,}000$, although the ConMAE ratio is not monotonic in $n$.
As calibration size increases from $100$ to $4{,}000$, CoCP's mean Length ratio decreases from $0.885$ to $0.830$ and its mean ConMAE ratio from $0.909$ to $0.545$.
At $n_{\mathrm{cal}}=100$, CHR and A-CQR can be slightly shorter on average, but CoCP maintains the lowest mean ConMAE ratio at every calibration size.
The relative gains are thus not confined to the largest data regime.

\begin{figure}[htbp]
  \centering
  \includegraphics[width=0.95\textwidth]{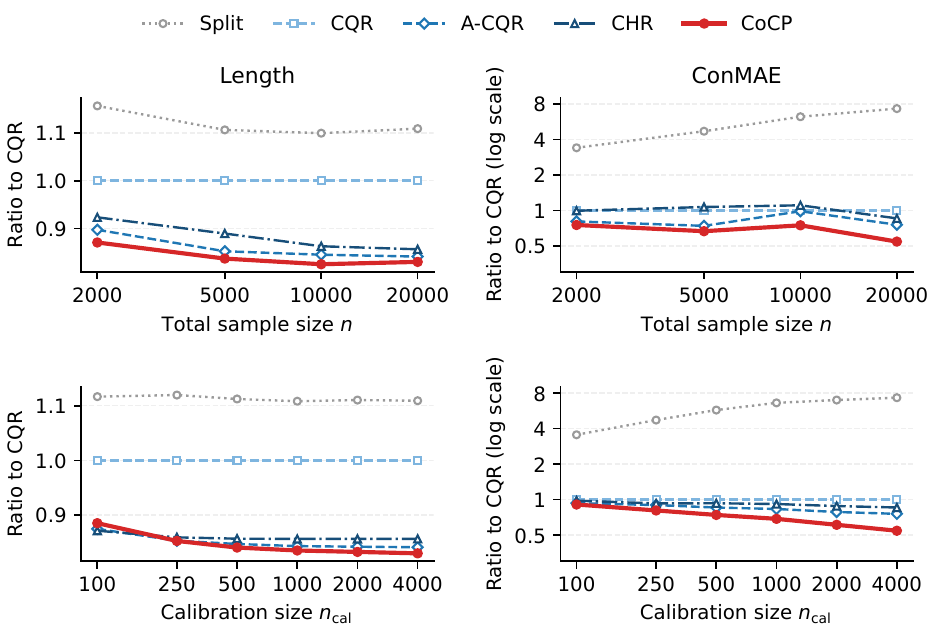}
  \caption{Sensitivity to training size (top row) and calibration size (bottom row) on the LogNormal and Exponential benchmarks at $\alpha=0.1$. Results are ratios to CQR (lower is better; CQR $= 1$), averaged over the two benchmarks. In the top row, every method is refitted at each total sample size~$n$; in the bottom row, the trained models are held fixed and only the calibration set is resampled. Horizontal axes use a logarithmic scale; the ConMAE axis (right column) is also logarithmic.}
  \label{fig:data-availability-sensitivity}
\end{figure}

Table~\ref{tab:small-real} extends the comparison to three additional small real-world datasets: \textit{Concrete} ($n\!=\!1{,}030$) \citep{concrete}, \textit{Community} ($n\!=\!1{,}994$) \citep{community}, and \textit{Star} ($n\!=\!2{,}161$) \citep{star}.
CoCP produces the shortest intervals on \textit{Concrete} and \textit{Community} and the second-shortest on \textit{Star}.
Its MSCE is lowest on \textit{Community} and tied for lowest on \textit{Star}; on \textit{Concrete}, A-CQR has the edge.
These results confirm that CoCP's efficiency gains persist at small sample sizes, while conditional-reliability advantages become more dataset-dependent in this regime.

\begin{table}[htbp]
\centering
\caption{Complete results on three additional small real datasets at $\alpha=0.1$ (mean $\pm$ std over ten splits). Arrows indicate the preferred direction; bold marks the best value for each dataset and metric.}
\label{tab:small-real}
\resizebox{\textwidth}{!}{%
\begin{tabular}{@{}llcccccc@{}}
\toprule
Dataset & Method & Coverage & Length (\(\downarrow\)) & MSCE (\(\downarrow\)) & WSC (\(\uparrow\)) & \(\ell_1\)-ERT (\(\downarrow\)) & \(\ell_2\)-ERT (\(\downarrow\)) \\
\midrule
\multirow{9}{*}{\textit{Community}}
& Split & 0.9035 \std{0.0287} & 2.0124 \std{0.1940} & 0.0133 \std{0.0068} & 0.5671 \std{0.0886} & 0.0521 \std{0.0153} & -0.0096 \std{0.0092} \\
& CPL & 0.9000 \std{0.0121} & 1.7506 \std{0.0905} & 0.0058 \std{0.0040} & 0.6311 \std{0.0528} & 0.0241 \std{0.0148} & -0.0185 \std{0.0077} \\
& CQR & 0.9023 \std{0.0202} & 1.7128 \std{0.0752} & 0.0032 \std{0.0008} & 0.6522 \std{0.0396} & 0.0118 \std{0.0175} & \textbf{-0.0242} \std{0.0065} \\
& A-CQR & 0.9018 \std{0.0213} & 1.6729 \std{0.1043} & 0.0035 \std{0.0010} & 0.6609 \std{0.0493} & 0.0091 \std{0.0198} & -0.0231 \std{0.0063} \\
& CPCP & 0.9045 \std{0.0390} & 1.8380 \std{0.3240} & 0.0054 \std{0.0017} & 0.6577 \std{0.0895} & 0.0439 \std{0.0182} & -0.0123 \std{0.0096} \\
& RCP & 0.9283 \std{0.0193} & 1.9499 \std{0.1778} & 0.0033 \std{0.0012} & \textbf{0.7125} \std{0.0460} & 0.0376 \std{0.0185} & -0.0129 \std{0.0061} \\
& C-HDR & 0.8935 \std{0.0242} & 1.9120 \std{0.5636} & 0.0042 \std{0.0018} & 0.6364 \std{0.0491} & 0.0291 \std{0.0168} & -0.0180 \std{0.0069} \\
& CHR & 0.9170 \std{0.0235} & 1.8082 \std{0.1942} & 0.0025 \std{0.0011} & 0.6950 \std{0.0473} & 0.0139 \std{0.0188} & -0.0201 \std{0.0056} \\
\rowcolor{gray!15} & \textbf{CoCP} & 0.9123 \std{0.0185} & \textbf{1.6202} \std{0.0913} & \textbf{0.0022} \std{0.0006} & 0.6811 \std{0.0261} & \textbf{0.0077} \std{0.0231} & -0.0226 \std{0.0092} \\
\midrule
\multirow{9}{*}{\textit{Concrete}}
& Split & 0.8990 \std{0.0248} & 0.6331 \std{0.0408} & 0.0122 \std{0.0098} & 0.5018 \std{0.1013} & 0.0506 \std{0.0291} & 0.0015 \std{0.0113} \\
& CPL & 0.9252 \std{0.0320} & 0.7168 \std{0.1066} & 0.0077 \std{0.0035} & \textbf{0.5952} \std{0.0817} & 0.0519 \std{0.0209} & 0.0019 \std{0.0046} \\
& CQR & 0.8976 \std{0.0205} & 0.7031 \std{0.1083} & 0.0063 \std{0.0040} & 0.5312 \std{0.0602} & 0.0283 \std{0.0192} & -0.0004 \std{0.0055} \\
& A-CQR & 0.8898 \std{0.0285} & 0.6843 \std{0.0363} & \textbf{0.0062} \std{0.0027} & 0.5160 \std{0.0566} & \textbf{0.0157} \std{0.0275} & -0.0019 \std{0.0065} \\
& CPCP & 0.9049 \std{0.0320} & 0.6370 \std{0.0592} & 0.0105 \std{0.0092} & 0.5375 \std{0.1299} & 0.0355 \std{0.0334} & 0.0025 \std{0.0128} \\
& RCP & 0.9136 \std{0.0354} & 0.6937 \std{0.0929} & 0.0102 \std{0.0057} & 0.5552 \std{0.1082} & 0.0625 \std{0.0173} & 0.0046 \std{0.0039} \\
& C-HDR & 0.8898 \std{0.0270} & 0.5908 \std{0.0709} & 0.0087 \std{0.0074} & 0.5187 \std{0.0734} & 0.0188 \std{0.0186} & \textbf{-0.0040} \std{0.0034} \\
& CHR & 0.9019 \std{0.0174} & 0.5957 \std{0.0661} & 0.0068 \std{0.0028} & 0.5460 \std{0.0544} & 0.0294 \std{0.0278} & -0.0013 \std{0.0029} \\
\rowcolor{gray!15} & \textbf{CoCP} & 0.8966 \std{0.0260} & \textbf{0.5177} \std{0.0512} & 0.0079 \std{0.0046} & 0.5304 \std{0.0728} & 0.0231 \std{0.0434} & -0.0013 \std{0.0059} \\
\midrule
\multirow{9}{*}{\textit{Star}}
& Split & 0.9009 \std{0.0228} & 0.2528 \std{0.0127} & 0.0026 \std{0.0014} & 0.6729 \std{0.0391} & 0.0231 \std{0.0134} & -0.0070 \std{0.0036} \\
& CPL & 0.8977 \std{0.0232} & 0.2733 \std{0.0238} & 0.0038 \std{0.0014} & 0.6525 \std{0.0383} & 0.0361 \std{0.0206} & -0.0046 \std{0.0041} \\
& CQR & 0.9076 \std{0.0235} & 0.3034 \std{0.0202} & 0.0030 \std{0.0009} & 0.6770 \std{0.0400} & \textbf{0.0117} \std{0.0129} & \textbf{-0.0098} \std{0.0037} \\
& A-CQR & 0.8984 \std{0.0205} & 0.2731 \std{0.0131} & 0.0038 \std{0.0009} & 0.6591 \std{0.0351} & 0.0319 \std{0.0220} & -0.0073 \std{0.0046} \\
& CPCP & 0.9012 \std{0.0491} & 0.2884 \std{0.0322} & 0.0058 \std{0.0049} & 0.6720 \std{0.1068} & 0.0469 \std{0.0230} & -0.0033 \std{0.0059} \\
& RCP & 0.9049 \std{0.0256} & 0.2999 \std{0.0258} & 0.0046 \std{0.0016} & 0.6340 \std{0.0647} & 0.0522 \std{0.0177} & -0.0017 \std{0.0062} \\
& C-HDR & 0.9116 \std{0.0217} & \textbf{0.2137} \std{0.0138} & 0.0034 \std{0.0014} & 0.6950 \std{0.0426} & 0.0243 \std{0.0114} & -0.0075 \std{0.0043} \\
& CHR & 0.9188 \std{0.0151} & 0.2809 \std{0.0151} & \textbf{0.0023} \std{0.0008} & 0.6987 \std{0.0368} & 0.0171 \std{0.0191} & -0.0071 \std{0.0040} \\
\rowcolor{gray!15} & \textbf{CoCP} & 0.9178 \std{0.0155} & 0.2515 \std{0.0119} & \textbf{0.0023} \std{0.0007} & \textbf{0.7151} \std{0.0409} & 0.0266 \std{0.0163} & -0.0058 \std{0.0026} \\
\bottomrule
\end{tabular}%
}
\end{table}

\section{Limitations and Future Work}\label{app:limitations}

Although CoCP generally produces shorter intervals and improves empirical conditional-reliability measures on the evaluated benchmarks, several theoretical and practical limitations remain.

\paragraph{Assumption of unimodality.} 
CoCP is designed to output a single, contiguous prediction interval. As proven in our theoretical analysis, this naturally recovers the optimal Highest Density Interval (HDI) when the underlying conditional distribution is unimodal. However, for highly multimodal distributions, the true Highest Density Region (HDR) may consist of disjoint regions. While CoCP will still output a marginally valid interval that captures the required mass, a single contiguous interval is fundamentally suboptimal for multimodal data. Extending the push-pull geometry to multi-interval outputs (e.g., via mixture-density centers) is a promising direction.

\paragraph{Dependency on base estimator consistency.} 
As discussed in Section~\ref{sec:theory}, while finite-sample marginal coverage is always guaranteed regardless of model quality, CoCP's asymptotic optimality strongly depends on the statistical consistency of the base estimators (Assumption~\ref{ass:consistency}). In practice, due to the highly non-convex nature of neural network optimization, strict theoretical consistency cannot be guaranteed. Consequently, under severe data sparsity, suboptimal local minima, or model misspecification, the learned geometry may deviate from the true HDI.

\paragraph{Offline computational overhead.} 
To reduce co-adaptation and stabilize gradient dynamics, CoCP relies on a $K$-fold alternating optimization scheme. While this does not affect real-time inference speed (which remains on par with standard CQR), it increases the offline training cost. Although our Fast variant (Appendix~\ref{app:cocp_fast}) significantly mitigates this overhead via fold-parallelization and strict early stopping, CoCP remains more computationally demanding to train than single-pass methods. Exploring joint optimization strategies that bypass strict alternation without sacrificing stability could further streamline the training pipeline.

\paragraph{Reliance on smoothing approximations.} 
The geometric boundary-balancing mechanism relies on a temperature parameter $\beta$ to smooth the hard interval indicator. Although our sample-adaptive strategy (Appendix~\ref{sec:adaptive_beta}) largely eliminates the need for manual tuning and provides scale invariance, the optimization landscape is still an approximation of the true non-differentiable coverage objective. In extreme cases of heavy-tailed distributions or severe heteroscedasticity, the gradient signal from the smoothed boundaries might still suffer from vanishing or exploding behaviors.

\paragraph{Extension to multivariate and structured outputs.} 
The current formulation of CoCP is tailored to univariate regression, where the boundary consists of exactly two endpoints. Extending this geometric push-pull mechanism to multivariate or structured outputs (e.g., bounding boxes in object detection, semantic segmentation masks) remains nontrivial. A natural first step would be to generalize CoCP to fixed geometric templates that share a similar center-scale structure, such as hyperrectangles, ellipsoids (Mahalanobis balls), or general $L_p$-norm balls. However, achieving optimal arbitrary shapes in higher dimensions requires matching the density along a continuous high-dimensional contour rather than discrete endpoints. Furthermore, while generative or representation-based pipelines \cite{dheur2025aunified} map data to latent spaces, translating a latent center may not preserve volume-optimality in the original space due to geometric distortions. Formulating a tractable, differentiable surrogate for multi-dimensional boundary balancing presents an exciting frontier for conformal regression.


\newpage
\section*{NeurIPS Paper Checklist}

\begin{enumerate}

\item {\bf Claims}
    \item[] Question: Do the main claims made in the abstract and introduction accurately reflect the paper's contributions and scope?
    \item[] Answer: \answerYes{} 
    \item[] Justification: The abstract and introduction clearly state the proposed CoCP method, its theoretical guarantees (marginal and asymptotic conditional coverage), and empirical performance, which are fully supported by the methodology in Section~\ref{sec:method}, theoretical analysis in Section~\ref{sec:theory}, and experiments in Section~\ref{sec:experiments}.
    \item[] Guidelines:
    \begin{itemize}
        \item The answer \answerNA{} means that the abstract and introduction do not include the claims made in the paper.
        \item The abstract and/or introduction should clearly state the claims made, including the contributions made in the paper and important assumptions and limitations. A \answerNo{} or \answerNA{} answer to this question will not be perceived well by the reviewers. 
        \item The claims made should match theoretical and experimental results, and reflect how much the results can be expected to generalize to other settings. 
        \item It is fine to include aspirational goals as motivation as long as it is clear that these goals are not attained by the paper. 
    \end{itemize}

\item {\bf Limitations}
    \item[] Question: Does the paper discuss the limitations of the work performed by the authors?
    \item[] Answer: \answerYes{} 
    \item[] Justification: A dedicated discussion of limitations and future work is provided in Appendix~\ref{app:limitations}, addressing the unimodality assumption, dependency on base estimator consistency, computational overhead, and challenges with multivariate extensions.
    \item[] Guidelines:
    \begin{itemize}
        \item The answer \answerNA{} means that the paper has no limitation while the answer \answerNo{} means that the paper has limitations, but those are not discussed in the paper. 
        \item The authors are encouraged to create a separate ``Limitations'' section in their paper.
        \item The paper should point out any strong assumptions and how robust the results are to violations of these assumptions (e.g., independence assumptions, noiseless settings, model well-specification, asymptotic approximations only holding locally). The authors should reflect on how these assumptions might be violated in practice and what the implications would be.
        \item The authors should reflect on the scope of the claims made, e.g., if the approach was only tested on a few datasets or with a few runs. In general, empirical results often depend on implicit assumptions, which should be articulated.
        \item The authors should reflect on the factors that influence the performance of the approach. For example, a facial recognition algorithm may perform poorly when image resolution is low or images are taken in low lighting. Or a speech-to-text system might not be used reliably to provide closed captions for online lectures because it fails to handle technical jargon.
        \item The authors should discuss the computational efficiency of the proposed algorithms and how they scale with dataset size.
        \item If applicable, the authors should discuss possible limitations of their approach to address problems of privacy and fairness.
        \item While the authors might fear that complete honesty about limitations might be used by reviewers as grounds for rejection, a worse outcome might be that reviewers discover limitations that aren't acknowledged in the paper. The authors should use their best judgment and recognize that individual actions in favor of transparency play an important role in developing norms that preserve the integrity of the community. Reviewers will be specifically instructed to not penalize honesty concerning limitations.
    \end{itemize}

\item {\bf Theory assumptions and proofs}
    \item[] Question: For each theoretical result, does the paper provide the full set of assumptions and a complete (and correct) proof?
    \item[] Answer: \answerYes{} 
    \item[] Justification: All assumptions (e.g., density regularity, estimation consistency, calibration regularity) are explicitly stated in Section~\ref{sec:theory}. Complete proofs for all theorems, propositions, and lemmas are provided in Appendices~\ref{app:theory} and~\ref{app:beta_rates_boundary}.
    \begin{itemize}
        \item The answer \answerNA{} means that the paper does not include theoretical results. 
        \item All the theorems, formulas, and proofs in the paper should be numbered and cross-referenced.
        \item All assumptions should be clearly stated or referenced in the statement of any theorems.
        \item The proofs can either appear in the main paper or the supplemental material, but if they appear in the supplemental material, the authors are encouraged to provide a short proof sketch to provide intuition. 
        \item Inversely, any informal proof provided in the core of the paper should be complemented by formal proofs provided in appendix or supplemental material.
        \item Theorems and Lemmas that the proof relies upon should be properly referenced. 
    \end{itemize}

    \item {\bf Experimental result reproducibility}
    \item[] Question: Does the paper fully disclose all the information needed to reproduce the main experimental results of the paper to the extent that it affects the main claims and/or conclusions of the paper (regardless of whether the code and data are provided or not)?
    \item[] Answer: \answerYes{} 
    \item[] Justification: The paper provides detailed pseudocode (Algorithms~\ref{alg:cocp} and~\ref{alg:adaptive_cocp}) and comprehensive implementation details, including network architectures, hyperparameters, data preprocessing steps, and baseline configurations in Appendix~\ref{app:exp-details}.
    \item[] Guidelines:
    \begin{itemize}
        \item The answer \answerNA{} means that the paper does not include experiments.
        \item If the paper includes experiments, a \answerNo{} answer to this question will not be perceived well by the reviewers: Making the paper reproducible is important, regardless of whether the code and data are provided or not.
        \item If the contribution is a dataset and\slash or model, the authors should describe the steps taken to make their results reproducible or verifiable. 
        \item Depending on the contribution, reproducibility can be accomplished in various ways. For example, if the contribution is a novel architecture, describing the architecture fully might suffice, or if the contribution is a specific model and empirical evaluation, it may be necessary to either make it possible for others to replicate the model with the same dataset, or provide access to the model. In general. releasing code and data is often one good way to accomplish this, but reproducibility can also be provided via detailed instructions for how to replicate the results, access to a hosted model (e.g., in the case of a large language model), releasing of a model checkpoint, or other means that are appropriate to the research performed.
        \item While NeurIPS does not require releasing code, the conference does require all submissions to provide some reasonable avenue for reproducibility, which may depend on the nature of the contribution. For example
        \begin{enumerate}
            \item If the contribution is primarily a new algorithm, the paper should make it clear how to reproduce that algorithm.
            \item If the contribution is primarily a new model architecture, the paper should describe the architecture clearly and fully.
            \item If the contribution is a new model (e.g., a large language model), then there should either be a way to access this model for reproducing the results or a way to reproduce the model (e.g., with an open-source dataset or instructions for how to construct the dataset).
            \item We recognize that reproducibility may be tricky in some cases, in which case authors are welcome to describe the particular way they provide for reproducibility. In the case of closed-source models, it may be that access to the model is limited in some way (e.g., to registered users), but it should be possible for other researchers to have some path to reproducing or verifying the results.
        \end{enumerate}
    \end{itemize}

\item {\bf Open access to data and code}
    \item[] Question: Does the paper provide open access to the data and code, with sufficient instructions to faithfully reproduce the main experimental results, as described in supplemental material?
    \item[] Answer: \answerYes{}
    \item[] Justification: The supplementary material provides the experimental code and instructions for reproducing the main comparisons.
    \item[] Guidelines:
    \begin{itemize}
        \item The answer \answerNA{} means that paper does not include experiments requiring code.
        \item Please see the NeurIPS code and data submission guidelines (\url{https://neurips.cc/public/guides/CodeSubmissionPolicy}) for more details.
        \item While we encourage the release of code and data, we understand that this might not be possible, so \answerNo{} is an acceptable answer. Papers cannot be rejected simply for not including code, unless this is central to the contribution (e.g., for a new open-source benchmark).
        \item The instructions should contain the exact command and environment needed to run to reproduce the results. See the NeurIPS code and data submission guidelines (\url{https://neurips.cc/public/guides/CodeSubmissionPolicy}) for more details.
        \item The authors should provide instructions on data access and preparation, including how to access the raw data, preprocessed data, intermediate data, and generated data, etc.
        \item The authors should provide scripts to reproduce all experimental results for the new proposed method and baselines. If only a subset of experiments are reproducible, they should state which ones are omitted from the script and why.
        \item At submission time, to preserve anonymity, the authors should release anonymized versions (if applicable).
        \item Providing as much information as possible in supplemental material (appended to the paper) is recommended, but including URLs to data and code is permitted.
    \end{itemize}

\item {\bf Experimental setting/details}
    \item[] Question: Does the paper specify all the training and test details (e.g., data splits, hyperparameters, how they were chosen, type of optimizer) necessary to understand the results?
    \item[] Answer: \answerYes{} 
    \item[] Justification: Experimental settings, including data generation, data splits, network architectures, hyperparameters, and baseline configurations, are detailed in Section~\ref{sec:experiments} and Appendix~\ref{app:exp-details}. Additional experiments on center refinement, asymmetric CQR, coverage levels, and data availability are documented in Appendix~\ref{app:additional-res}.
    \item[] Guidelines:
    \begin{itemize}
        \item The answer \answerNA{} means that the paper does not include experiments.
        \item The experimental setting should be presented in the core of the paper to a level of detail that is necessary to appreciate the results and make sense of them.
        \item The full details can be provided either with the code, in appendix, or as supplemental material.
    \end{itemize}

\item {\bf Experiment statistical significance}
    \item[] Question: Does the paper report error bars suitably and correctly defined or other appropriate information about the statistical significance of the experiments?
    \item[] Answer: \answerYes{} 
    \item[] Justification: The main experimental results (Tables~\ref{tab:synthetic_results} and~\ref{tab:real_world_results}) and additional small-dataset results (Table~\ref{tab:small-real}) report means and standard deviations over 10 random data splits. Runtime, implementation variants, and local modality diagnostics also report variability across splits (Tables~\ref{tab:training_time}, \ref{tab:inference_time}, \ref{tab:cocp_fast_comparison}, \ref{tab:adaptive_beta_combined}, and~\ref{tab:local-dip}). Further robustness analyses, including asymmetric CQR and alternative coverage levels, are reported in Appendix~\ref{app:additional-res}.
    \item[] Guidelines:
    \begin{itemize}
        \item The answer \answerNA{} means that the paper does not include experiments.
        \item The authors should answer \answerYes{} if the results are accompanied by error bars, confidence intervals, or statistical significance tests, at least for the experiments that support the main claims of the paper.
        \item The factors of variability that the error bars are capturing should be clearly stated (for example, train/test split, initialization, random drawing of some parameter, or overall run with given experimental conditions).
        \item The method for calculating the error bars should be explained (closed form formula, call to a library function, bootstrap, etc.)
        \item The assumptions made should be given (e.g., Normally distributed errors).
        \item It should be clear whether the error bar is the standard deviation or the standard error of the mean.
        \item It is OK to report 1-sigma error bars, but one should state it. The authors should preferably report a 2-sigma error bar than state that they have a 96\% CI, if the hypothesis of Normality of errors is not verified.
        \item For asymmetric distributions, the authors should be careful not to show in tables or figures symmetric error bars that would yield results that are out of range (e.g., negative error rates).
        \item If error bars are reported in tables or plots, the authors should explain in the text how they were calculated and reference the corresponding figures or tables in the text.
    \end{itemize}

\item {\bf Experiments compute resources}
    \item[] Question: For each experiment, does the paper provide sufficient information on the computer resources (type of compute workers, memory, time of execution) needed to reproduce the experiments?
    \item[] Answer: \answerYes{} 
    \item[] Justification: Appendix~\ref{app:exp-details} details the hardware used (2x NVIDIA A100 80GB PCIe GPUs, CUDA 12.6, peak memory usage), and Appendix~\ref{app:runtime} provides comprehensive runtime comparisons (training and inference times) for all methods.
    \item[] Guidelines:
    \begin{itemize}
        \item The answer \answerNA{} means that the paper does not include experiments.
        \item The paper should indicate the type of compute workers CPU or GPU, internal cluster, or cloud provider, including relevant memory and storage.
        \item The paper should provide the amount of compute required for each of the individual experimental runs as well as estimate the total compute. 
        \item The paper should disclose whether the full research project required more compute than the experiments reported in the paper (e.g., preliminary or failed experiments that didn't make it into the paper). 
    \end{itemize}
    
\item {\bf Code of ethics}
    \item[] Question: Does the research conducted in the paper conform, in every respect, with the NeurIPS Code of Ethics \url{https://neurips.cc/public/EthicsGuidelines}?
    \item[] Answer: \answerYes{} 
    \item[] Justification: The paper presents foundational methodological research in machine learning and strictly adheres to the NeurIPS Code of Ethics.
    \item[] Guidelines:
    \begin{itemize}
        \item The answer \answerNA{} means that the authors have not reviewed the NeurIPS Code of Ethics.
        \item If the authors answer \answerNo, they should explain the special circumstances that require a deviation from the Code of Ethics.
        \item The authors should make sure to preserve anonymity (e.g., if there is a special consideration due to laws or regulations in their jurisdiction).
    \end{itemize}

\item {\bf Broader impacts}
    \item[] Question: Does the paper discuss both potential positive societal impacts and negative societal impacts of the work performed?
    \item[] Answer: \answerNA{} 
    \item[] Justification: This is foundational theoretical and methodological work on conformal prediction. It aims to improve uncertainty quantification reliability, which generally has positive impacts, but it does not have direct foreseeable negative societal impacts or specific applications requiring mitigation strategies.
    \item[] Guidelines:
    \begin{itemize}
        \item The answer \answerNA{} means that there is no societal impact of the work performed.
        \item If the authors answer \answerNA{} or \answerNo, they should explain why their work has no societal impact or why the paper does not address societal impact.
        \item Examples of negative societal impacts include potential malicious or unintended uses (e.g., disinformation, generating fake profiles, surveillance), fairness considerations (e.g., deployment of technologies that could make decisions that unfairly impact specific groups), privacy considerations, and security considerations.
        \item The conference expects that many papers will be foundational research and not tied to particular applications, let alone deployments. However, if there is a direct path to any negative applications, the authors should point it out. For example, it is legitimate to point out that an improvement in the quality of generative models could be used to generate Deepfakes for disinformation. On the other hand, it is not needed to point out that a generic algorithm for optimizing neural networks could enable people to train models that generate Deepfakes faster.
        \item The authors should consider possible harms that could arise when the technology is being used as intended and functioning correctly, harms that could arise when the technology is being used as intended but gives incorrect results, and harms following from (intentional or unintentional) misuse of the technology.
        \item If there are negative societal impacts, the authors could also discuss possible mitigation strategies (e.g., gated release of models, providing defenses in addition to attacks, mechanisms for monitoring misuse, mechanisms to monitor how a system learns from feedback over time, improving the efficiency and accessibility of ML).
    \end{itemize}
    
\item {\bf Safeguards}
    \item[] Question: Does the paper describe safeguards that have been put in place for responsible release of data or models that have a high risk for misuse (e.g., pre-trained language models, image generators, or scraped datasets)?
    \item[] Answer: \answerNA{} 
    \item[] Justification: The paper does not introduce models or datasets with a high risk for misuse (e.g., no LLMs, image generators, or scraped datasets are involved).
    \item[] Guidelines:
    \begin{itemize}
        \item The answer \answerNA{} means that the paper poses no such risks.
        \item Released models that have a high risk for misuse or dual-use should be released with necessary safeguards to allow for controlled use of the model, for example by requiring that users adhere to usage guidelines or restrictions to access the model or implementing safety filters. 
        \item Datasets that have been scraped from the Internet could pose safety risks. The authors should describe how they avoided releasing unsafe images.
        \item We recognize that providing effective safeguards is challenging, and many papers do not require this, but we encourage authors to take this into account and make a best faith effort.
    \end{itemize}

\item {\bf Licenses for existing assets}
    \item[] Question: Are the creators or original owners of assets (e.g., code, data, models), used in the paper, properly credited and are the license and terms of use explicitly mentioned and properly respected?
    \item[] Answer: \answerYes{} 
    \item[] Justification: All real-world datasets used in the experiments are standard, publicly available benchmarks, and their original sources are properly cited in Appendix~\ref{app:exp-real-details} (Table~\ref{tab:dataset_summary}).
    \item[] Guidelines:
    \begin{itemize}
        \item The answer \answerNA{} means that the paper does not use existing assets.
        \item The authors should cite the original paper that produced the code package or dataset.
        \item The authors should state which version of the asset is used and, if possible, include a URL.
        \item The name of the license (e.g., CC-BY 4.0) should be included for each asset.
        \item For scraped data from a particular source (e.g., website), the copyright and terms of service of that source should be provided.
        \item If assets are released, the license, copyright information, and terms of use in the package should be provided. For popular datasets, \url{paperswithcode.com/datasets} has curated licenses for some datasets. Their licensing guide can help determine the license of a dataset.
        \item For existing datasets that are re-packaged, both the original license and the license of the derived asset (if it has changed) should be provided.
        \item If this information is not available online, the authors are encouraged to reach out to the asset's creators.
    \end{itemize}

\item {\bf New assets}
    \item[] Question: Are new assets introduced in the paper well documented and is the documentation provided alongside the assets?
    \item[] Answer: \answerYes{} 
    \item[] Justification: The CoCP implementation is publicly released with documentation, usage instructions, and an open-source license.
    \item[] Guidelines:
    \begin{itemize}
        \item The answer \answerNA{} means that the paper does not release new assets.
        \item Researchers should communicate the details of the dataset\slash code\slash model as part of their submissions via structured templates. This includes details about training, license, limitations, etc. 
        \item The paper should discuss whether and how consent was obtained from people whose asset is used.
        \item At submission time, remember to anonymize your assets (if applicable). You can either create an anonymized URL or include an anonymized zip file.
    \end{itemize}

\item {\bf Crowdsourcing and research with human subjects}
    \item[] Question: For crowdsourcing experiments and research with human subjects, does the paper include the full text of instructions given to participants and screenshots, if applicable, as well as details about compensation (if any)? 
    \item[] Answer: \answerNA{} 
    \item[] Justification: The research does not involve crowdsourcing or human subjects.
    \item[] Guidelines:
    \begin{itemize}
        \item The answer \answerNA{} means that the paper does not involve crowdsourcing nor research with human subjects.
        \item Including this information in the supplemental material is fine, but if the main contribution of the paper involves human subjects, then as much detail as possible should be included in the main paper. 
        \item According to the NeurIPS Code of Ethics, workers involved in data collection, curation, or other labor should be paid at least the minimum wage in the country of the data collector. 
    \end{itemize}

\item {\bf Institutional review board (IRB) approvals or equivalent for research with human subjects}
    \item[] Question: Does the paper describe potential risks incurred by study participants, whether such risks were disclosed to the subjects, and whether Institutional Review Board (IRB) approvals (or an equivalent approval/review based on the requirements of your country or institution) were obtained?
    \item[] Answer: \answerNA{} 
    \item[] Justification: The research does not involve human subjects.
    \item[] Guidelines:
    \begin{itemize}
        \item The answer \answerNA{} means that the paper does not involve crowdsourcing nor research with human subjects.
        \item Depending on the country in which research is conducted, IRB approval (or equivalent) may be required for any human subjects research. If you obtained IRB approval, you should clearly state this in the paper. 
        \item We recognize that the procedures for this may vary significantly between institutions and locations, and we expect authors to adhere to the NeurIPS Code of Ethics and the guidelines for their institution. 
        \item For initial submissions, do not include any information that would break anonymity (if applicable), such as the institution conducting the review.
    \end{itemize}

\item {\bf Declaration of LLM usage}
    \item[] Question: Does the paper describe the usage of LLMs if it is an important, original, or non-standard component of the core methods in this research? Note that if the LLM is used only for writing, editing, or formatting purposes and does \emph{not} impact the core methodology, scientific rigor, or originality of the research, declaration is not required.
    \item[] Answer: \answerNA{} 
    \item[] Justification: LLMs were not used as a component of the core methods in this research.
    \item[] Guidelines:
    \begin{itemize}
        \item The answer \answerNA{} means that the core method development in this research does not involve LLMs as any important, original, or non-standard components.
        \item Please refer to our LLM policy in the NeurIPS handbook for what should or should not be described.
    \end{itemize}

\end{enumerate}

\end{document}